\PassOptionsToPackage{numbers, compress}{natbib}
\PassOptionsToPackage{table}{xcolor}
\documentclass{article}

\usepackage[final,main]{neurips_2026}
\makeatletter
\renewcommand{\@noticestring}{}
\makeatother

\usepackage[utf8]{inputenc}
\usepackage[T1]{fontenc}
\usepackage{hyperref}
\usepackage{url}
\usepackage{booktabs}
\usepackage{amsfonts}
\usepackage{nicefrac}
\usepackage{microtype}
\usepackage{xcolor}
\usepackage{pifont}
\usepackage{enumitem}

\usepackage{amsmath, amssymb, amsfonts, amsthm}
\usepackage{graphicx}
\usepackage{threeparttable}
\usepackage[ruled,vlined,linesnumbered]{algorithm2e}
\SetKwComment{Comment}{/* }{ */}
\usepackage{caption}
\usepackage{subcaption}
\usepackage{multirow}
\usepackage{makecell}
\usepackage{float}

\newtheorem{theorem}{Theorem}[section]

\newtheorem{proposition}[theorem]{Proposition}

\newtheorem{definition}[theorem]{Definition}

\newtheorem{remark}[theorem]{Remark}

\definecolor{DDSRad}{HTML}{D62728}
\definecolor{dtanh}{HTML}{FF7F0E}
\definecolor{BoxPre}{HTML}{1F77B4}
\definecolor{QP}{HTML}{A52A2A}
\definecolor{SRadQP}{HTML}{9467BD}
\definecolor{SRadStrict}{HTML}{2CA02C}

\definecolor{1e-4}{HTML}{D62728}
\definecolor{1e-3}{HTML}{E6550D}
\definecolor{1e-2}{HTML}{FDAE6B}
\definecolor{1e-1}{HTML}{31A354}
\definecolor{1e0}{HTML}{3182BD}

\title{\textbf{Constraint-Enhanced Reinforcement Learning Based on Dynamic Decoupled Spherical Radial Squashing}}

\author{
	Qijun Liao$^{1}$ \quad Zhaoxin Yu$^{2}$ \quad Jue Yang$^{1}$\thanks{Corresponding author: \href{mailto:yangjue@ustb.edu.cn}{yangjue@ustb.edu.cn}} \\[4pt]
	\small $^{1}$School of Mechanical Engineering, University of Science and Technology Beijing \\
	\small $^{2}$Institute of Automation, Chinese Academy of Sciences
}

\begin{document}
	
	\maketitle
	
	\begin{abstract}
		When deploying reinforcement learning policies to physical robots, actuator rate constraints---hard limits on how fast each joint can move per control step---are unavoidable. These limits vary substantially across joints due to differences in motor inertia, power bandwidth, and transmission stiffness, creating pronounced heterogeneity that existing methods fail to handle geometrically: the per-joint feasible region forms a high-dimensional box in action-increment space, yet QP projection and spherical parameterization methods impose isotropic ball-shaped constraints, exponentially under-covering the true feasible set as heterogeneity grows. This paper proposes Dynamic Decoupled Spherical Radial Squashing (DD-SRad), which resolves this mismatch by computing a position-adaptive radius independently for each actuator, achieving tight alignment with the true per-joint feasible region. DD-SRad satisfies per-step hard constraints with probability~1, preserves well-conditioned gradients throughout training, and admits exact policy gradient backpropagation with zero runtime solver overhead. MuJoCo benchmark experiments demonstrate the highest task return at zero constraint violation---matching the unconstrained upper bound---with 30\%--50\% improvement in constraint-space coverage over spherical baselines. High-fidelity IsaacLab simulations with Unitree H1 and G1 humanoid robots confirm end-to-end optimality parameterized directly from official joint specifications, validating a systematic pathway from hardware datasheets to safe deployment. Code is publicly available at \url{https://anonymous.4open.science/r/DD-SRad}.
	\end{abstract}
	
	\section{Introduction}
	\label{sec:intro}
	Reinforcement learning has become a powerful tool for training robotic control policies, enabling legged robots and manipulators to acquire agile behaviors through large-scale simulation. However, deploying these policies on physical hardware requires more than high task rewards: the action sequence produced by the policy must be executable by real actuators at every control step. A fundamental requirement is satisfying actuator rate constraints, namely hard limits on how much each joint's position command can change per control step, imposed by motor inertia, power amplifier bandwidth, and transmission stiffness. Violating these constraints can cause commands to be silently discarded, trigger overcurrent protection, or even lead to permanent joint damage. Critically, these limits differ substantially across joints: for the Unitree H1 humanoid, hip joints allow 2.2$\times$ the rate of ankle joints, while industrial six-axis manipulators can exceed a 5:1 proximal-to-distal ratio~\cite{wensing2017proprioceptive}. This heterogeneity across actuators is not a design defect, but a structural reality that any RL algorithm targeting real-world deployment must address.
	
	Heterogeneous per-step rate constraints impose a structure that existing approaches fail to handle simultaneously across three dimensions. The constraint $|a_t^i - a_{t-1}^i| \leq \delta^i$---where $a_t^i$ denotes the $i$-th joint's position command at step $t$ and $\delta^i > 0$ its per-joint rate limit---defines a feasible set $\mathcal{F}_t = \prod_i [a_{t-1}^i - \delta^i,\, a_{t-1}^i + \delta^i]$ in action-increment space: an $\ell_\infty$ hyperrectangle whose side lengths $2\delta^i$ vary across dimensions and drift with each control step. Yet each major solution family sacrifices a different desideratum: hierarchical MPC/QP methods incur runtime solver overhead and training-deployment inconsistency, precluding end-to-end optimization~\cite{grandia2023perceptive,bledt2018cheetah,kim2019highly,pandala2022robust,williams2017information,nagabandi2018neural}; CMDP-based safe RL methods offer only expected-form guarantees that permit transient per-step violations~\cite{achiam2017constrained,tessler2018reward,zhang2020first,yang2022constrained,zhang2022P3O,altman1999constrained}; and action parameterization methods either impose isotropic $\ell_2$-ball constraints that compress the reachable set under heterogeneous $\delta^i$~\cite{kasaura2023benchmarking}, or assume a static feasible set incompatible with the per-step drift of $\mathcal{F}_t$~\cite{brahmanage2023flowpg,hung2025aram,stolz2025truncated}. Under heterogeneous constraints, the reachable-set volume of $\ell_2$ spherical parameterization is only $(\min_i\delta^i)^d / \prod_i\delta^i$ of the $\ell_\infty$ feasible set---a ratio that degrades exponentially with dimension and heterogeneity. No existing method simultaneously achieves hard per-step constraint satisfaction, exact $\ell_\infty$ coverage of $\mathcal{F}_t$, and end-to-end training without runtime solver overhead. Table~\ref{tab:method_comparison} summarizes this comparison systematically.
	
	\begin{table}[!t]
		\caption{Comparison of Rate-Constraint Handling Methods.}
		\label{tab:method_comparison}
		\centering
		\scriptsize
		\setlength{\tabcolsep}{2.0pt}
		\renewcommand{\arraystretch}{1.08}
		\resizebox{\columnwidth}{!}{%
			\begin{tabular}{lcccccc}
				\toprule
				\textbf{Method}
				& \begin{tabular}{c}\textbf{Hard}\\\textbf{Per-Step}\\\textbf{Guarantee}\end{tabular}
				& \begin{tabular}{c}\textbf{Heterogeneous}\\\textbf{Rate-Limit}\\\textbf{Support}\end{tabular}
				& \begin{tabular}{c}\textbf{Exact}\\\textbf{Box-Shaped}\\\textbf{Geometry}\end{tabular}
				& \begin{tabular}{c}\textbf{Time-Varying}\\\textbf{Feasible Set}\end{tabular}
				& \begin{tabular}{c}\textbf{No Runtime}\\\textbf{Solver}\end{tabular}
				& \begin{tabular}{c}\textbf{Off-Policy}\\\textbf{End-to-End}\end{tabular} \\
				\midrule
				RL+MPC/QP~\cite{grandia2023perceptive,bledt2018cheetah}
				& {\color{green}$\checkmark$} & Partial & {\color{red}\ding{55}} & {\color{green}$\checkmark$} & {\color{red}\ding{55}} & {\color{red}\ding{55}} \\
				
				CMDP (CPO~\cite{achiam2017constrained}, FOCOPS~\cite{zhang2020first})
				& {\color{red}\ding{55}} & {\color{red}\ding{55}} & -- & -- & {\color{green}$\checkmark$} & {\color{green}$\checkmark$} \\
				
				TN-SAC~\cite{stolz2025truncated}
				& {\color{green}$\checkmark$} & {\color{green}$\checkmark$} & {\color{green}$\checkmark$} & {\color{red}\ding{55}} & {\color{green}$\checkmark$} & {\color{green}$\checkmark$} \\
				
				SRad~\cite{kasaura2023benchmarking}
				& {\color{green}$\checkmark$} & {\color{red}\ding{55}} & {\color{red}\ding{55}} & {\color{red}\ding{55}} & {\color{green}$\checkmark$} & {\color{green}$\checkmark$} \\
				
				BoxPre+~\cite{kasaura2023benchmarking}
				& {\color{green}$\checkmark$} & {\color{green}$\checkmark$} & {\color{red}\ding{55}} & {\color{green}$\checkmark$} & {\color{red}\ding{55}} & {\color{green}$\checkmark$} \\
				
				FlowPG~\cite{brahmanage2023flowpg} / ARAM~\cite{hung2025aram}
				& {\color{green}$\checkmark$} & {\color{green}$\checkmark$} & {\color{green}$\checkmark$} & {\color{red}\ding{55}} & {\color{green}$\checkmark$} & {\color{red}\ding{55}} \\
				
				\textbf{DD-SRad (Ours)}
				& {\color{green}$\checkmark$} & {\color{green}$\checkmark$} & {\color{green}$\checkmark$} & {\color{green}$\checkmark$} & {\color{green}$\checkmark$} & {\color{green}$\checkmark$} \\
				\bottomrule
		\end{tabular}}
		\vspace{0.5mm}
		
		\begin{minipage}{\columnwidth}
			\scriptsize
			\textit{Hard Per-Step Guarantee}: every control step satisfies the actuator rate limits.
			\textit{Heterogeneous Rate-Limit Support}: different joints may have different rate limits.
			\textit{Time-Varying Feasible Set}: the valid action region changes with the previous action.
		\end{minipage}
	\end{table}
	
	To fill this gap, this research proposes \textbf{Dynamic Decoupled Spherical Radial Squashing (DD-SRad)}, a smooth analytic action parameterization that simultaneously achieves hard per-step constraint satisfaction, exact $\ell_\infty$ coverage of the feasible set, and end-to-end gradient backpropagation with zero runtime solver overhead. Instead of applying a single global radius, DD-SRad computes a \emph{position-adaptive effective radius} $R_{\text{eff}}^i$ independently for each action dimension, so that the reachable set tightly covers the full $\ell_\infty$ feasible box $\mathcal{F}_t$ rather than an inscribed sphere. This per-dimension decoupling eliminates the dimensional coupling inherent to $\ell_2$ squashing and provably satisfies per-step hard constraints with probability~1. Because the parameterization is fully smooth and analytic, it admits exact policy gradient backpropagation via the chain rule with zero solver overhead, and plugs directly into off-policy backbones such as SAC and TD3 without architectural changes.
	
	The contributions of this paper are as follows:
	\begin{itemize}[left=10pt]
		\item \textbf{Geometric alignment}: DD-SRad achieves exact $\ell_\infty$ coverage of the feasible set via per-dimension adaptive radius, recovering volume systematically lost by $\ell_2$ baselines.
		\item \textbf{Theoretical guarantees}: This research proves per-step hard constraint satisfaction with probability~1 and establishes Jacobian condition number bounds governed by the per-dimension limits.
		\item \textbf{End-to-end compatibility}: The smooth analytic form supports exact backpropagation with no runtime solver.
		\item \textbf{Empirical validation}: DD-SRad achieves the highest return at zero violation on MuJoCo benchmarks and transfers directly to Unitree H1/G1 locomotion in IsaacLab using official rate specifications.
	\end{itemize}
	
	\section{Methodology}
	\label{sec:methodology}
	
	\subsection{Preliminaries}
	
	\paragraph{Spherical Radial Squashing (SRad).}
	Kasaura et al.~\cite{kasaura2023benchmarking} define the mapping:
	\begin{equation}
		a = c_s + R \cdot \frac{u}{\sqrt{1 + \|u\|^2}}
		\label{eq:srad}
	\end{equation}
	where $u \in \mathbb{R}^d$ is the latent action (network output), $c_s$ is the ball center, and $R > 0$ is the global radius. This mapping strictly guarantees $\|a - c_s\| < R$, with $\nabla_u a$ nonzero everywhere. However, the global radius $R$ induces dimensional coupling: if any dimension approaches the boundary, $R$ must be reduced globally, restricting all other dimensions' exploration. In the differential action space, SRad uses $R = \min_i \delta^i$, giving an $\ell_2$ ball that mismatches the $\ell_\infty$ feasible set geometry; no globally valid and tight $R$ exists under heterogeneous $\boldsymbol{\delta}$. This geometric mismatch motivates DD-SRad's design: computing $R_{\text{eff}}^i$ independently per dimension so that the reachable set aligns exactly with $\mathcal{F}_t$ (Theorem~\ref{thm:exploration}), while the smooth per-dimension Jacobian preserves full gradient information under heterogeneous $\boldsymbol{\delta}$ (Proposition~\ref{prop:gradient}).
	
	This work addresses reinforcement learning with rate constraints, where the state space is $\mathcal{S} \subseteq \mathbb{R}^n$, the action space is $\mathcal{A} = \prod_{i=1}^d [a_{\min}^i, a_{\max}^i]$, the reward function is $r: \mathcal{S} \times \mathcal{A} \to \mathbb{R}$, and the transition dynamics are $P(s_{t+1} | s_t, a_t)$. The rate constraints require:
	\begin{equation}
		|a_t^i - a_{t-1}^i| \leq \delta^i, \quad \forall i \in \{1, \ldots, d\},\ \forall t \geq 1
		\label{eq:rate_constraint}
	\end{equation}
	where $\boldsymbol{\delta} = [\delta^1, \ldots, \delta^d] \in \mathbb{R}_{++}^d$, allowing significant heterogeneity ($\delta^i \neq \delta^j$). State augmentation $\tilde{s}_t=(s_t,a_{t-1})$ resolves the induced non-Markovianity (Appendix~\ref{app:proof:augmented_mdp}). Standard conditions assumed: bounded $r$, compact $\mathcal{A}$, Lipschitz $P(\cdot\mid s,a)$.
	
	\subsection{Dynamic Decoupled Spherical Radial Squashing}
	\label{sec:decoupled}
	
	\subsubsection{The DD-SRad Mapping}
	
	\begin{definition}[Per-Dimension Effective Radius]
		\label{def:reff}
		Given per-dimension rate limits $\boldsymbol{\delta} \in \mathbb{R}_{++}^d$, previous action $a_{\text{prev}} \in \mathcal{A}$, and position bounds $[a_{\min}^i, a_{\max}^i]$, define the effective radius for dimension $i$ as:
		\begin{equation}
			R_{\text{eff}}^i(u^i, a_{\text{prev}}^i) = \begin{cases}
				\min\!\left(\delta^i,\ a_{\max}^i - a_{\text{prev}}^i\right), & \text{if } u^i > 0 \\
				\min\!\left(\delta^i,\ a_{\text{prev}}^i - a_{\min}^i\right), & \text{if } u^i < 0 \\
				\delta^i,                                                        & \text{if } u^i = 0
			\end{cases}
			\label{eq:reff}
		\end{equation}
	\end{definition}
	
	\begin{definition}[DD-SRad Mapping]
		Given latent action $u \in \mathbb{R}^d$ and previous action $a_{\text{prev}} \in \mathcal{A}$, the DD-SRad mapping is defined as an independent per-dimension transformation:
		\begin{equation}
			a^i = a_{\text{prev}}^i + R_{\text{eff}}^i(u^i, a_{\text{prev}}^i) \cdot \frac{u^i}{\sqrt{1 + (u^i)^2}}, \quad \forall i \in \{1, \ldots, d\}
			\label{eq:ddsrad}
		\end{equation}
		In vector form:
		\begin{equation}
			\mathbf{a} = \mathbf{a}_{\text{prev}} + \mathbf{R}_{\text{eff}} \odot \mathbf{f}(\mathbf{u})
			\label{eq:ddsrad_vec}
		\end{equation}
		where $\odot$ denotes element-wise product, $\mathbf{f}(\mathbf{u})^i = u^i / \sqrt{1 + (u^i)^2}$ is the per-dimension spherical squashing function, and $\mathbf{R}_{\text{eff}} = [R_{\text{eff}}^1, \ldots, R_{\text{eff}}^d]^\top$.
	\end{definition}
	
	\begin{remark}[Geometric Comparison with SRad and $\ell_\infty$ Projection Methods]
		SRad constructs an $\ell_2$ ball with global radius $R=\min_i\delta^i$; DD-SRad's reachable set is the $\ell_\infty$ hyperrectangle $\prod_i[a_{\text{prev}}^i-R_{\text{eff}}^i,\,a_{\text{prev}}^i+R_{\text{eff}}^i]$ (volume ratio in Theorem~\ref{thm:exploration}). Compared to $\ell_\infty$ clip, DD-SRad has the same reachable set but maintains a smooth Jacobian near the boundary; clip truncates gradients at constraint boundaries, losing directional information during policy updates.
	\end{remark}
	
	As a numerical safety redundancy, optional global $\ell_2$ clipping scales the action increment $\Delta\mathbf{a} = \mathbf{a} - \mathbf{a}_{\text{prev}}$ after per-dimension squashing:
	\begin{equation}
		\widetilde{\Delta \mathbf{a}} = \Delta \mathbf{a} \cdot \min\!\left(1,\ \frac{\|\boldsymbol{\delta}\|_2}{\|\Delta \mathbf{a}\|_2}\right)
		\label{eq:clip_l2}
	\end{equation}
	The trigger rate is below $0.1\%$; per-dimension squashing provides the primary guarantee.
	
	\subsection{Theoretical Guarantees}
	\label{sec:theory}
	
	Proofs are deferred to Appendix~\ref{app:proofs}.
	
	\subsubsection{Hard Constraint Satisfaction}
	
	\begin{theorem}[Heterogeneous Rate Constraint Satisfaction]
		\label{thm:constraint}
		For any $a_{\text{prev}} \in \mathcal{A}$, any latent action $u \in \mathbb{R}^d$, and any per-dimension rate limits $\boldsymbol{\delta} \in \mathbb{R}_{++}^d$, the DD-SRad mapping~(\ref{eq:ddsrad}) guarantees:
		\begin{equation}
			|a^i - a_{\text{prev}}^i| \leq \delta^i \quad \text{and} \quad a^i \in [a_{\min}^i, a_{\max}^i], \quad \forall i,\ \text{with probability } 1
		\end{equation}
		Furthermore, when $a_{\text{prev}}^i \in (a_{\min}^i, a_{\max}^i)$, both constraints hold with strict inequalities.
	\end{theorem}
	\begin{proof}
		See Appendix~\ref{app:proof:constraint}.
	\end{proof}
	
	\subsubsection{Gradient Non-degeneracy}
	
	The spherical squashing function saturates like $\tanh$: as $|u^i|\to\infty$, the Jacobian decreases toward $0^+$ at rate $O(|u^i|^{-3})$; DD-SRad confines $\|u\|$ within the effective gradient region via $\lambda_{\text{base}}\|u\|^2$ (§\ref{sec:regularization}).
	
	\begin{proposition}[Jacobian Structure and Gradient Norm Bound]
		\label{prop:gradient}
		Let $\kappa \triangleq \max_i \delta^i / \min_i \delta^i \geq 1$. For $u^i \neq 0$ and $a_{\text{prev}}^i \in (a_{\min}^i, a_{\max}^i)$, the Jacobian $J(u) = \partial a/\partial u$ of the DD-SRad mapping satisfies:
		\begin{enumerate}
			\item[(i)] $J(u)$ is a diagonal matrix with $i$-th diagonal entry
			\begin{equation}
				\frac{\partial a^i}{\partial u^i} = \frac{R_{\text{eff}}^i(u^i, a_{\text{prev}}^i)}{\left(1 + (u^i)^2\right)^{3/2}} > 0;
				\label{eq:jacobian}
			\end{equation}
			\item[(ii)] The spectral norm $\|J(u)\|_2 \leq \max_i \delta^i$ holds for all $u \in \mathbb{R}^d$; away from position boundaries ($R_{\text{eff}}^i = \delta^i$) at $u = 0$, the condition number reaches $\kappa$, which is the worst-case upper bound on the condition number;
			\item[(iii)] For the TD3 backbone (without entropy term), the gradient norm of the actor loss satisfies
			\begin{equation}
				\|\nabla_u \mathcal{L}_{\text{actor}}\|_2 \leq \max_i \delta^i \cdot \|\nabla_a Q(\tilde{s}, a)\|_2 + 2\lambda_{\text{base}} \|u\|_2;
				\label{eq:grad_bound}
			\end{equation}
			for the SAC backbone, the entropy regularization term $\alpha\|\nabla_u \log\pi(u|\tilde{s})\|_2$ is additionally included.
		\end{enumerate}
	\end{proposition}
	\begin{proof}
		See Appendix~\ref{app:proof:gradient}.
	\end{proof}
	
	\begin{remark}
		At $u^i = 0$, the mapping is continuous but non-differentiable, with left/right derivatives $R^{-,i}=\min(\delta^i,a_{\text{prev}}^i-a_{\min}^i)$ and $R^{+,i}=\min(\delta^i,a_{\max}^i-a_{\text{prev}}^i)$; automatic differentiation returns one as a subgradient (both strictly positive). This measure-zero event causes no empirical convergence issues under continuous policy distributions (§\ref{sec:sensitivity}).
	\end{remark}
	
	The mapping is $\|\boldsymbol{\delta}\|_2$-Lipschitz continuous in $u$ (Appendix Proposition~\ref{prop:lipschitz}), ensuring small policy perturbations cause bounded action changes---a standard regularity condition for gradient-based optimization.
	
	\subsubsection{Tight $\ell_\infty$-Coverage of the Reachable Set}
	
	\begin{theorem}[Tight Coverage of the Reachable Set with the $\ell_\infty$ Feasible Set]
		\label{thm:exploration}
		Fix $a_{\text{prev}} \in \mathcal{A}$, and let the joint rate-position feasible set be $\mathcal{F} = \{a \in \mathbb{R}^d : |a^i - a_{\text{prev}}^i| \leq \delta^i\} \cap \mathcal{A}$. The closure of the reachable set of the DD-SRad mapping satisfies
		\begin{equation}
			\overline{\mathcal{R}_{\text{DD}}} = \prod_{i=1}^d \left[a_{\text{prev}}^i - R^{-,i},\, a_{\text{prev}}^i + R^{+,i}\right] = \mathcal{F},
			\label{eq:reach_tight}
		\end{equation}
		with Lebesgue measure $V_{\text{DD}} = \prod_{i=1}^d (R^{-,i}+R^{+,i})$. Relative to $\ell_2$ spherical parameterization (SRad) with global radius $R = \min_i \delta^i$, the volume ratio away from position boundaries ($R_{\text{eff}}^i = \delta^i$) is
		\begin{equation}
			\frac{V_{\text{DD}}}{V_{\text{SRad}}} = \frac{2^d \Gamma(d/2+1)}{\pi^{d/2}} \cdot \frac{\prod_i \delta^i}{(\min_i \delta^i)^d},
			\label{eq:volume_ratio}
		\end{equation}
		where $V_{\text{SRad}} = \frac{\pi^{d/2}}{\Gamma(d/2+1)}(\min_i \delta^i)^d$ is the $d$-dimensional $\ell_2$ ball volume and $\Gamma(\cdot)$ is the Gamma function.
	\end{theorem}
	\begin{proof}
		See Appendix~\ref{app:proof:exploration}.
	\end{proof}
	
	\subsection{Policy Optimization with DD-SRad Parameterization}
	\label{sec:regularization}
	
	Upon embedding the DD-SRad mapping into the off-policy actor network, the actor loss for SAC/TD3 backbones takes the unified form
	\begin{equation}
		\mathcal{L}_{\text{actor}} = \mathbb{E}_{(s,a_{\text{prev}}) \sim \mathcal{D}}\!\left[\alpha \log \pi_\theta(a|\tilde{s}) - Q_\phi(\tilde{s}, a)\right] + \lambda_{\text{base}} \, \mathbb{E}[\|u\|^2],
		\label{eq:actor_loss}
	\end{equation}
	where $\tilde{s} = (s, a_{\text{prev}})$ is the augmented state, $\mathcal{D}$ is the replay buffer, $\alpha \geq 0$ is the SAC entropy temperature, $Q_\phi$ is the Q-network with parameters $\phi$, and $\lambda_{\text{base}} > 0$ is the latent action regularization coefficient; for the TD3 backbone, the entropy term $\alpha = 0$ and Eq.~(\ref{eq:actor_loss}) reduces to the deterministic policy gradient form. For the SAC backbone, since the action $a = \phi_{\tilde{s}}(u)$ is obtained by mapping the latent action $u \sim \mathcal{N}(\mu_\theta(s), \sigma_\theta^2(s))$ through DD-SRad, the policy density is given via the change-of-variables formula under standard reparameterization:
	\begin{equation}
		\log \pi_\theta(a|\tilde{s}) = \log \mathcal{N}(u; \mu_\theta(s), \sigma_\theta^2(s)) - \sum_{i=1}^d \log \frac{\partial a^i}{\partial u^i},
		\label{eq:sac_logprob}
	\end{equation}
	where the per-dimension Jacobian $\partial a^i/\partial u^i = R_{\text{eff}}^i/(1+(u^i)^2)^{3/2}$ is given by Eq.~(\ref{eq:jacobian}), so the log-determinant term admits the closed-form expression
	\begin{equation}
		\sum_{i=1}^d \log \frac{\partial a^i}{\partial u^i} = \sum_{i=1}^d \left[\log R_{\text{eff}}^i(u^i, a_{\text{prev}}^i) - \frac{3}{2}\log\!\left(1 + (u^i)^2\right)\right].
		\label{eq:logdet}
	\end{equation}
	This expression is smooth in $\theta$ (at $u\neq0$) with no additional overhead. The regularization $\lambda_{\text{base}}\|u\|^2$ confines $\|u\|$ within the effective gradient regime---when $\|u\|\gg1$, Jacobian entries $\partial a^i/\partial u^i\to0^+$, degrading gradient transmission---and mitigates cross-dimensional gradient imbalance under heterogeneous constraints (Eq.~\ref{eq:grad_bound}). Robustness of $\lambda_{\text{base}}$ is verified in §\ref{sec:sensitivity}.
	
	The complete procedure is in Algorithm~\ref{alg:ddsrad} (Appendix); DD-SRad, as an independent parameterization layer, is backbone-agnostic: constraint satisfaction is guaranteed by Theorem~\ref{thm:constraint} and gradient backpropagation by Proposition~\ref{prop:gradient}.
	
	\subsection{Gradient Propagation Through Off-Policy Backbones}
	\label{sec:offpolicy}
	
	The DD-SRad parameterization is naturally compatible with off-policy algorithms based on an action-conditioned $Q$ function (SAC~\cite{haarnoja2018soft}, TD3~\cite{fujimoto2018td3}). Taking SAC as an example, the actor gradient is
	\begin{equation}
		\nabla_\theta \mathcal{L}_{\text{actor}} = \mathbb{E}_{\tilde{s}}\!\left[\nabla_a Q(\tilde{s}, a)\big|_{a=\phi_{\tilde{s}}(u_\theta)} \cdot \frac{\partial \phi_{\tilde{s}}}{\partial u} \cdot \frac{\partial u}{\partial \theta}\right],
		\label{eq:offpolicy_grad}
	\end{equation}
	where $\phi_{\tilde{s}}: \mathbb{R}^d \to \mathcal{A}$ denotes the DD-SRad mapping (Eq.~\ref{eq:ddsrad}) conditioned on $a_{\text{prev}}$ in $\tilde{s}$, and $\partial \phi_{\tilde{s}} / \partial u$ is diagonally positive definite (Proposition~\ref{prop:gradient}(i)), propagating all directional information in $\nabla_a Q$ losslessly to $\theta$; TD3 is identical without the entropy term. DD-SRad is thus a plug-and-play parameterization layer for both off-policy families.
	
	For on-policy algorithms (e.g., PPO~\cite{schulman2017proximal}), constraint satisfaction is still guaranteed by the mapping structure and $a_{t-1}$ can be incorporated into the augmented observation; however, the scalar advantage $A_t$ has substantially higher estimation variance under high-$\kappa$ constraints than the per-dimension $\nabla_a Q$ signal, limiting fine-grained per-dimension exploration (see §\ref{sec:limitation} for extension directions).
	
	\section{Experiments}
	\label{sec:experiments}
	
	Three experiment groups verify the quantitative predictions of §\ref{sec:theory}: \S\ref{sec:benchmark} confirms hard constraint satisfaction (Proposition~\ref{prop:gradient}, Theorem~\ref{thm:constraint}) and the performance advantage of smooth Jacobian structure over gradient-truncating baselines; \S\ref{sec:hetero_test} verifies the $\ell_\infty$-coverage prediction of Theorem~\ref{thm:exploration} via per-dimension utilization; \S\ref{sec:sensitivity} validates the $\lambda_{\text{base}}$ sensitivity bound (Proposition~\ref{prop:gradient}(iii)). All runs use 5 seeds; metrics are Return$\uparrow$, Utilization ($\mathbb{E}[\|\Delta a\|_1]/\|\boldsymbol{\delta}\|_1$), and CV$\downarrow$; CVR is reported in the notes when nonzero. Environment: MuJoCo v2.3.7 / Gymnasium v0.29; hyperparameters in Appendix~\ref{app:hyperparams}.
	
	DD-SRad is evaluated on SAC~\cite{haarnoja2018soft} and TD3~\cite{fujimoto2018td3}. Baselines: \textbf{D-Tanh} uses $\tanh$ instead of spherical squashing with the same $R_{\text{eff}}^i$ (isolates Jacobian decay effect); \textbf{BoxPre+} uses critic-aware $\ell_\infty$ clip projection~\cite{kasaura2023benchmarking}; \textbf{Post(QP)} uses constraint-unaware clip post-processing; \textbf{SRad-Strict} sets $R=\min_i\delta^i$ (quantifies $\ell_2$ mismatch); \textbf{SRad-QP} appends QP after SRad.
	
	\subsection{Benchmark Performance}
	\label{sec:benchmark}
	
	Experiments cover four environments---Ant-v5 (8D), Humanoid-v5 (17D), HalfCheetah-v5 (6D), Hopper-v5 (3D)---each tested under wide homogeneous constraints and tight heterogeneous constraints (see Appendix Table~\ref{tab:constraint_params}), with 1.5M--2M training steps.
	
	\renewcommand{\arraystretch}{1}
	\begin{table*}[t]
		\scriptsize
		\centering
		\captionsetup{font=scriptsize, labelfont=scriptsize}
		\caption{Quantitative benchmark summary under tight heterogeneous constraints (mean$\pm$std over 5 seeds; best in \textbf{bold}; $\dagger$ see footnote)}
		\label{tab:benchmark_main}
		
		\textbf{(a) SAC Backbone}\\[0.5mm]
		\begin{tabular}{l ccc ccc}
			\toprule
			& \multicolumn{3}{c}{\textbf{Ant-v5}
				\small($\boldsymbol{\delta}=[0.2^4,\!0.5^4]$, $\kappa=2.5$)}
			& \multicolumn{3}{c}{\textbf{Humanoid-v5}
				\small($\boldsymbol{\delta}=[0.8^6,\!0.5^6,\!0.2^5]$, $\kappa=4.0$)} \\
			\cmidrule(lr){2-4}\cmidrule(lr){5-7}
			\textbf{Method} & Return $\uparrow$ & CV $\downarrow$ & Util
			& Return $\uparrow$ & CV $\downarrow$ & Util \\
			\midrule
			BoxPre+
			& $1998\pm125$           & $\mathbf{6.3}$ & $0.453\pm0.026$
			& $5204\pm176$           & $\mathbf{3.4}$ & $0.159\pm0.033$ \\
			Post(QP)$^\dagger$
			& $1517\pm137$           & $9.0$ & $0.651\pm0.015$
			& $5191\pm346$           & $6.7$ & $0.249\pm0.045$ \\
			SRad-Strict
			& $1448\pm996$           & $68.8$ & $0.141\pm0.029$
			& $2742\pm1350$          & $49.3$ & $0.024\pm0.007$ \\
			SRad-QP
			& $1703\pm232$           & $13.6$ & $0.583\pm0.035$
			& $4558\pm500$           & $11.0$ & $0.053\pm0.024$ \\
			D-Tanh
			& $2719\pm174$           & $6.4$ & $0.570\pm0.037$
			& $4868\pm267$           & $5.5$ & $0.662\pm0.069$ \\
			\rowcolor{blue!10}
			DD-SRad (Ours)
			& $\mathbf{3112\pm788}$  & $25.3$ & $0.492\pm0.051$
			& $\mathbf{5620\pm215}$  & $3.8$ & $0.489\pm0.035$ \\
			
			\midrule
			& \multicolumn{3}{c}{\textbf{HalfCheetah-v5}
				\small($\boldsymbol{\delta}=[0.2^3,\!0.5^3]$, $\kappa=2.5$)}
			& \multicolumn{3}{c}{\textbf{Hopper-v5}
				\small($\boldsymbol{\delta}=[0.2,\!0.5,\!0.5]$, $\kappa=2.5$)} \\
			\cmidrule(lr){2-4}\cmidrule(lr){5-7}
			\textbf{Method} & Return $\uparrow$ & CV $\downarrow$ & Util
			& Return $\uparrow$ & CV $\downarrow$ & Util \\
			\midrule
			BoxPre+
			& $3796\pm59$            & $\mathbf{1.6}$ & $0.600\pm0.016$
			& $2319\pm208$           & $9.0$ & $0.267\pm0.015$ \\
			Post(QP)$^\dagger$
			& $1798\pm246$           & $13.7$ & $0.767\pm0.022$
			& $1278\pm165$           & $12.9$ & $0.353\pm0.014$ \\
			SRad-Strict
			& $1432\pm147$           & $10.3$ & $0.213\pm0.014$
			& $2195\pm735$           & $33.5$ & $0.155\pm0.020$ \\
			SRad-QP
			& $2482\pm433$           & $17.4$ & $0.557\pm0.018$
			& $1624\pm332$           & $20.4$ & $0.370\pm0.009$ \\
			D-Tanh
			& $4205\pm231$           & $5.5$ & $0.662\pm0.069$
			& $2334\pm276$           & $11.8$ & $0.480\pm0.011$ \\
			\rowcolor{blue!10}
			DD-SRad (Ours)
			& $\mathbf{4290\pm682}$  & $15.9$ & $0.626\pm0.095$
			& $\mathbf{2610\pm183}$  & $\mathbf{7.0}$ & $0.461\pm0.007$ \\
			\bottomrule
		\end{tabular}%
		
		\vspace{1mm}
		\scriptsize
		\textbf{(b) TD3 Backbone}\\[0.5mm]
		\begin{tabular}{l ccc ccc}
			\toprule
			& \multicolumn{3}{c}{\textbf{Ant-v5}
				\small($\boldsymbol{\delta}=[0.2^4,\!0.5^4]$, $\kappa=2.5$)}
			& \multicolumn{3}{c}{\textbf{Humanoid-v5}
				\small($\boldsymbol{\delta}=[0.8^6,\!0.5^6,\!0.2^5]$, $\kappa=4.0$)} \\
			\cmidrule(lr){2-4}\cmidrule(lr){5-7}
			\textbf{Method} & Return $\uparrow$ & CV $\downarrow$ & Util
			& Return $\uparrow$ & CV $\downarrow$ & Util \\
			\midrule
			BoxPre+
			& $2917\pm289$           & $9.9$ & $0.518\pm0.018$
			& $5291\pm119$           & $\mathbf{2.3}$ & $0.179\pm0.030$ \\
			Post(QP)$^\dagger$
			& $1998\pm282$           & $14.1$ & $0.689\pm0.029$
			& $4420\pm117$           & $2.6$ & $0.465\pm0.036$ \\
			SRad-Strict
			& $1246\pm376$           & $30.2$ & $0.176\pm0.015$
			& $2270\pm1126$          & $49.6$ & $0.053\pm0.013$ \\
			SRad-QP
			& $2147\pm195$           & $9.1$ & $0.554\pm0.011$
			& $3810\pm332$           & $8.7$ & $0.015\pm0.003$ \\
			D-Tanh
			& $3484\pm781$           & $22.4$ & $0.565\pm0.017$
			& $5313\pm341$           & $6.4$ & $0.184\pm0.039$ \\
			\rowcolor{blue!10}
			DD-SRad (Ours)
			& $\mathbf{4260\pm176}$  & $\mathbf{4.1}$ & $0.548\pm0.021$
			& $\mathbf{5498\pm203}$  & $3.7$ & $0.203\pm0.024$ \\
			
			\midrule
			& \multicolumn{3}{c}{\textbf{HalfCheetah-v5}
				\small($\boldsymbol{\delta}=[0.2^3,\!0.5^3]$, $\kappa=2.5$)}
			& \multicolumn{3}{c}{\textbf{Hopper-v5}
				\small($\boldsymbol{\delta}=[0.2,\!0.5,\!0.5]$, $\kappa=2.5$)} \\
			\cmidrule(lr){2-4}\cmidrule(lr){5-7}
			\textbf{Method} & Return $\uparrow$ & CV $\downarrow$ & Util
			& Return $\uparrow$ & CV $\downarrow$ & Util \\
			\midrule
			BoxPre+
			& $3641\pm681$           & $18.7$ & $0.607\pm0.081$
			& $3256\pm196$           & $\mathbf{6.0}$ & $0.303\pm0.014$ \\
			Post(QP)$^\dagger$
			& $1484\pm1042$          & $70.2$ & $0.695\pm0.060$
			& $849\pm251$            & $29.6$ & $0.487\pm0.023$ \\
			SRad-Strict
			& $1518\pm185$           & $12.2$ & $0.211\pm0.011$
			& $2394\pm322$           & $13.4$ & $0.166\pm0.015$ \\
			SRad-QP
			& $3211\pm128$           & $\mathbf{4.0}$ & $0.582\pm0.027$
			& $1020\pm139$           & $13.6$ & $0.333\pm0.019$ \\
			D-Tanh
			& $4000\pm318$           & $7.9$ & $0.672\pm0.035$
			& $2712\pm228$           & $8.4$ & $0.496\pm0.003$ \\
			\rowcolor{blue!10}
			DD-SRad (Ours)
			& $\mathbf{4329\pm529}$  & $12.2$ & $0.668\pm0.028$
			& $\mathbf{3312\pm384}$  & $11.6$ & $0.462\pm0.023$ \\
			\bottomrule
		\end{tabular}%
		
		\vspace{1.5mm}
		\scriptsize
		\textbf{Notes:}
		CVR\,${=}$\,0 for all parameterization-based methods across all configurations, confirmed by Theorem~\ref{thm:constraint}.
		$^\dagger$\,SAC-Post(QP)/TD3-Post(QP): Critic does not perceive projection;
		its Util is inflated by boundary-dwelling behavior (clip),
		not indicative of superior geometric coverage---
		Util comparison is meaningful only among Critic-aware methods.
	\end{table*}
	\renewcommand{\arraystretch}{1}
	
	Table~\ref{tab:benchmark_main} provides the quantitative summary; wide-homogeneous results are in Appendix Table~\ref{tab:benchmark_wide}. All constrained methods achieve CVR$\,=0$ (Theorem~\ref{thm:constraint}); SAC-Post(QP) shows non-zero CVR under tight configurations.
	
	Under tight heterogeneous constraints, DD-SRad achieves the highest Return across all 8 environment--backbone configurations, matching or exceeding the unconstrained upper bound. D-Tanh shares the same $R_{\text{eff}}^i$ mechanism yet trails by 5\%--14\% under SAC despite higher Util: the $\tanh$ Jacobian decays exponentially in $|u^i|$ versus DD-SRad's polynomial decay (Proposition~\ref{prop:gradient}), compounding with $\kappa$---on Humanoid ($d=17$, $\kappa=4.0$), D-Tanh falls below BoxPre+, while DD-SRad leads.
	
	SRad-Strict exhibits structural collapse (Ant-SAC CV$=68.8\%$; Humanoid-SAC CV$=49.3\%$), confirming the $\ell_2$ isotropic geometry mismatch predicted by Theorem~\ref{thm:exploration}: SRad-Strict ankle utilization ($\approx0.081$) falls far below $\min_j\delta^j/\delta^\text{ankle}=0.40$. BoxPre+ achieves the same $\ell_\infty$ rectangular reachable set as DD-SRad (Theorem~\ref{thm:exploration}) yet trails DD-SRad substantially on Ant-SAC (1998 vs.\ 3112), confirming Proposition~\ref{prop:gradient}: $\ell_\infty$ clip discontinuously truncates $\nabla_a Q$ at constraint boundaries, while DD-SRad's diagonally positive definite Jacobian propagates full directional gradient information to $\theta$. Learning curves and radar charts in Appendix~\ref{app:bench_curves},~\ref{app:radar_benchmark}.
	
	\subsection{Analysis of Action Optimization Under Heterogeneous Constraints}
	\label{sec:hetero_test}
	
	This section verifies Theorem~\ref{thm:exploration} via training diagnostics; Utilization distinguishes policy-driven conservatism from parameterization-forced feasible-set truncation. Figure~\ref{fig:hetero_return_all} confirms no divergence across all 8 configurations under heterogeneous constraints; HalfCheetah degrades most ($\approx40\%$ under SAC) as its locomotion task directly requires large joint outputs.
	
	\begin{figure}[htbp]
		\centering
		\begin{subfigure}{0.24\linewidth}
			\includegraphics[width=\linewidth]{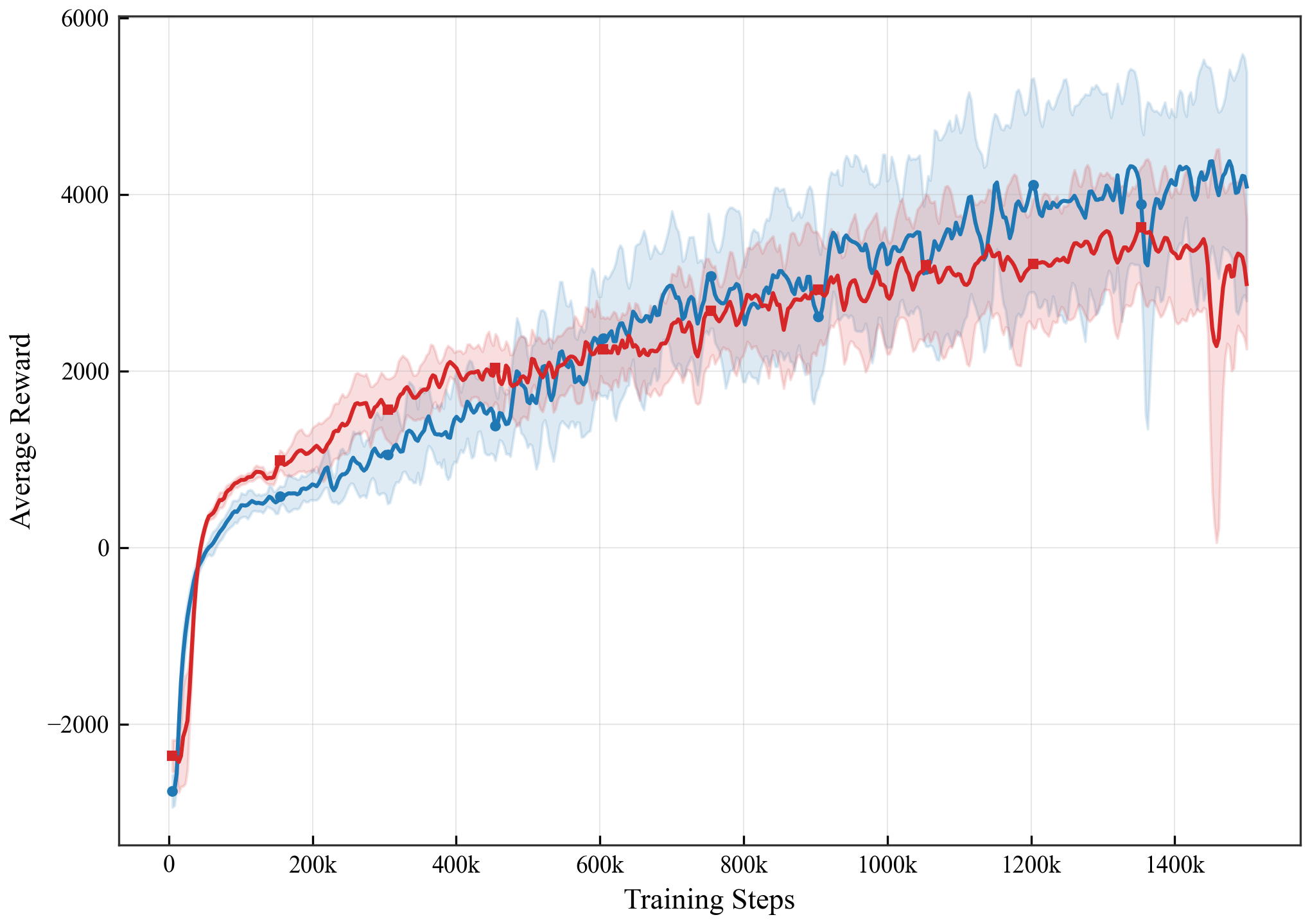}
			\caption{Ant-SAC}
		\end{subfigure}
		\hfill
		\begin{subfigure}{0.24\linewidth}
			\includegraphics[width=\linewidth]{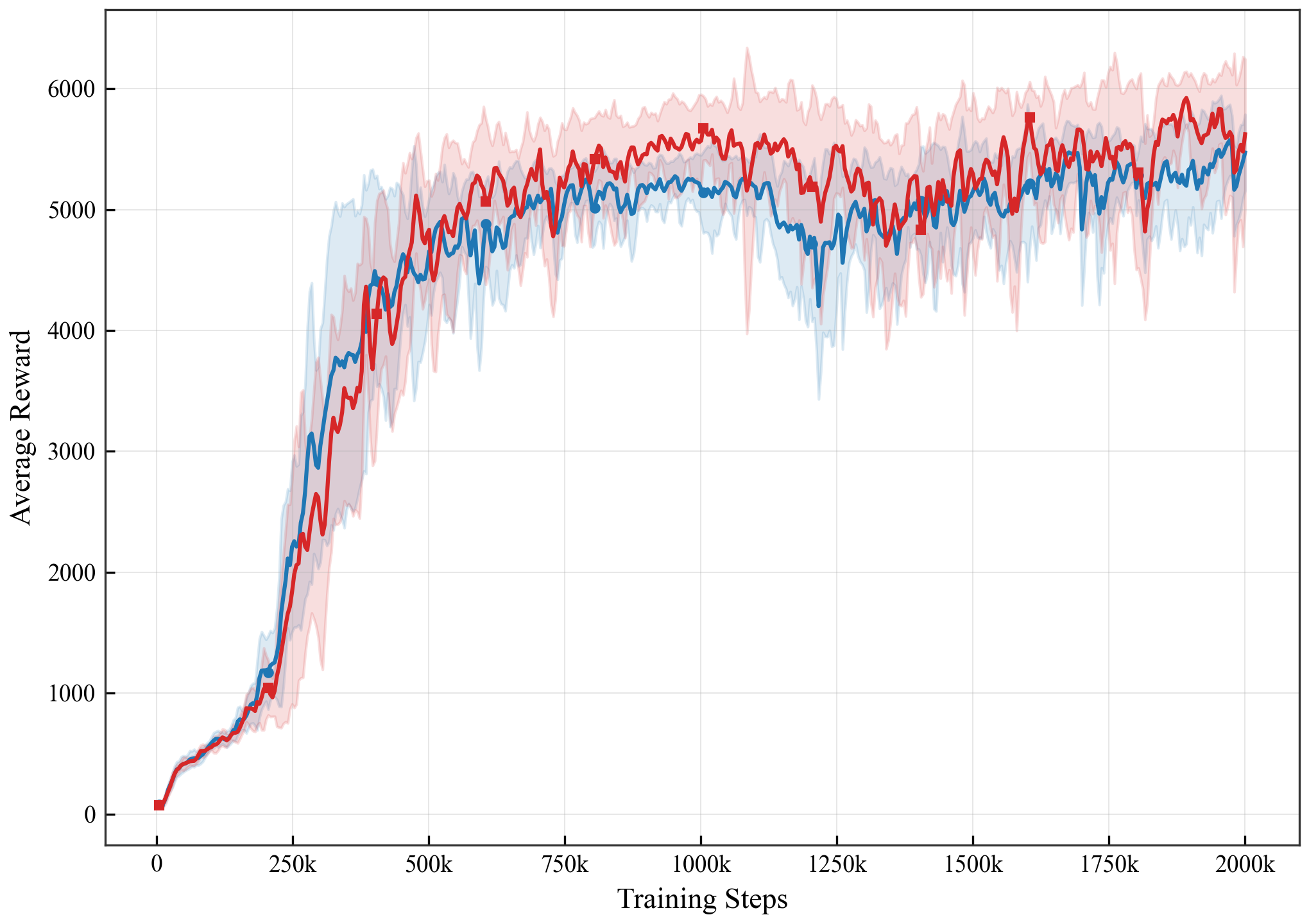}
			\caption{Humanoid-SAC}
		\end{subfigure}
		\vspace{0.5em}
		\begin{subfigure}{0.24\linewidth}
			\includegraphics[width=\linewidth]{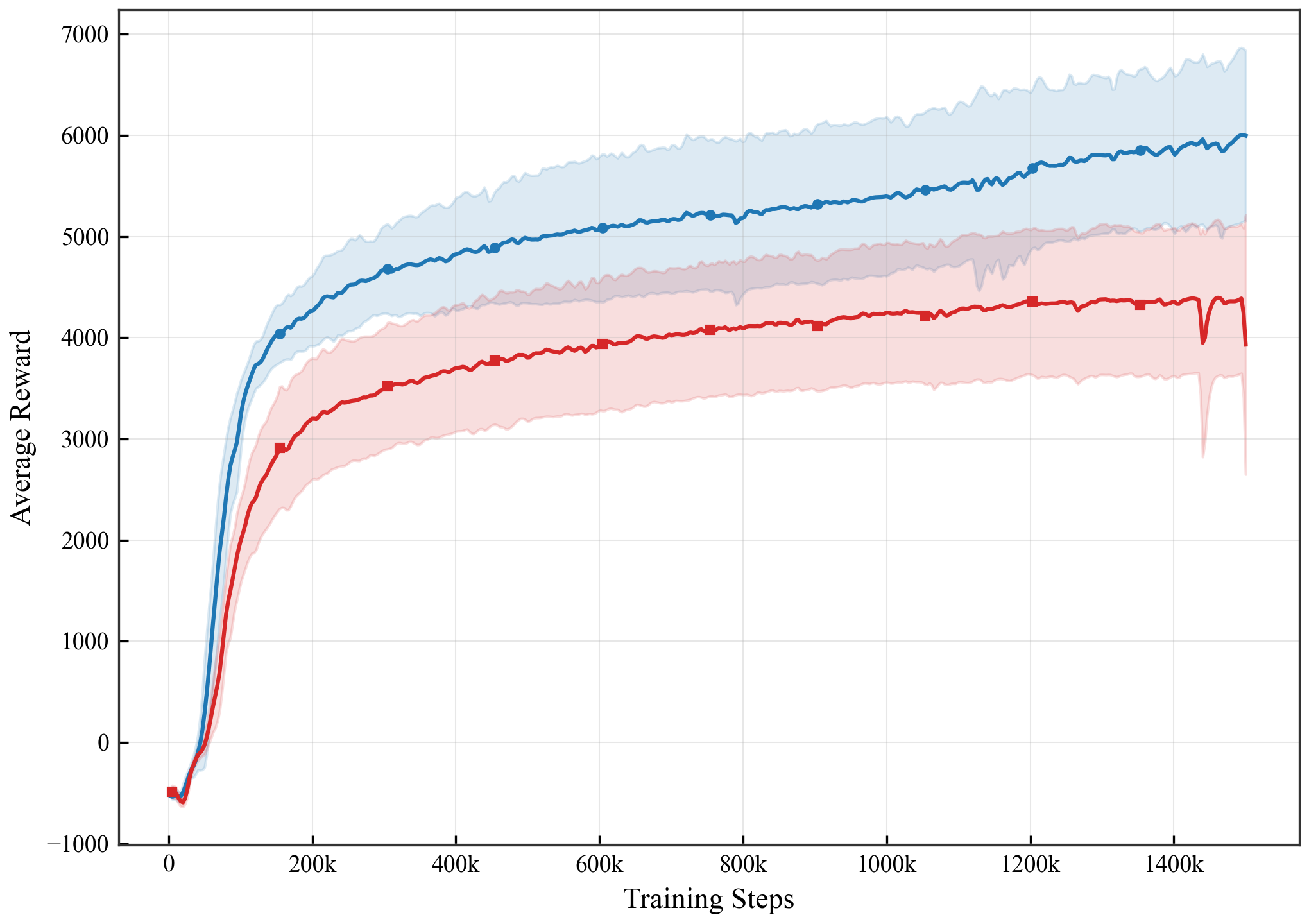}
			\caption{HalfCheetah-SAC}
		\end{subfigure}
		\hfill
		\begin{subfigure}{0.24\linewidth}
			\includegraphics[width=\linewidth]{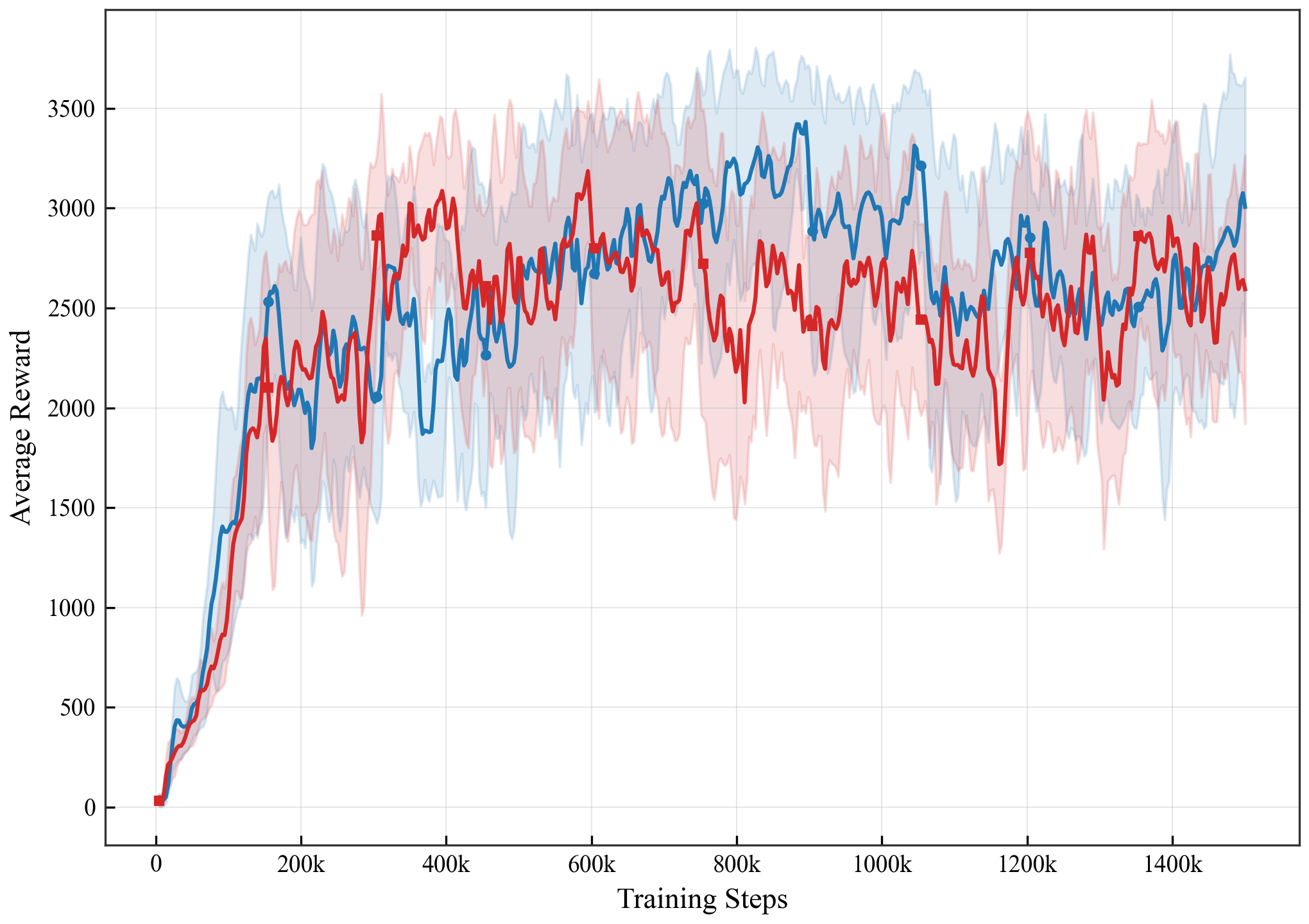}
			\caption{Hopper-SAC}
		\end{subfigure}
		
		\begin{subfigure}{0.24\linewidth}
			\includegraphics[width=\linewidth]{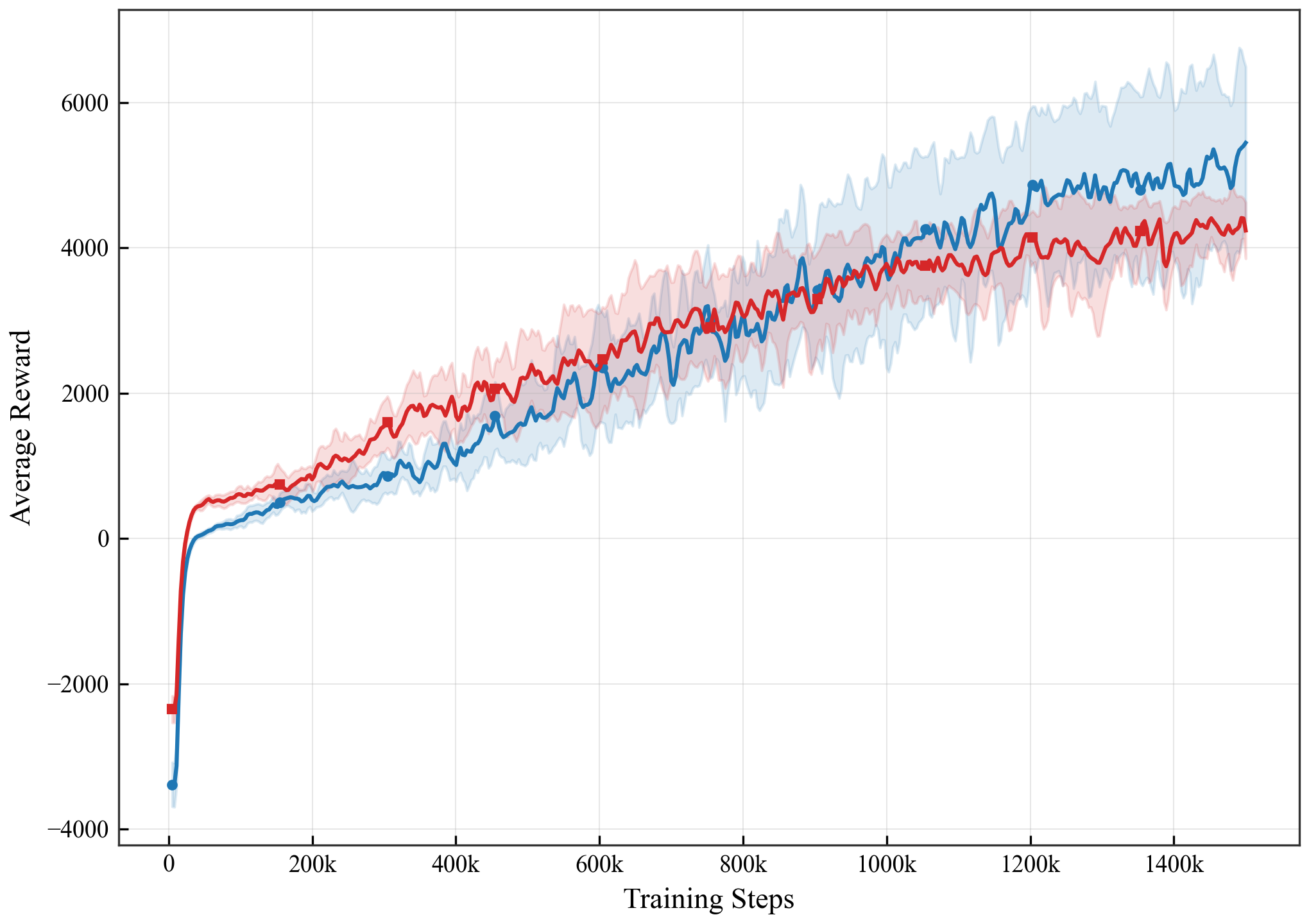}
			\caption{Ant-TD3}
		\end{subfigure}
		\hfill
		\begin{subfigure}{0.24\linewidth}
			\includegraphics[width=\linewidth]{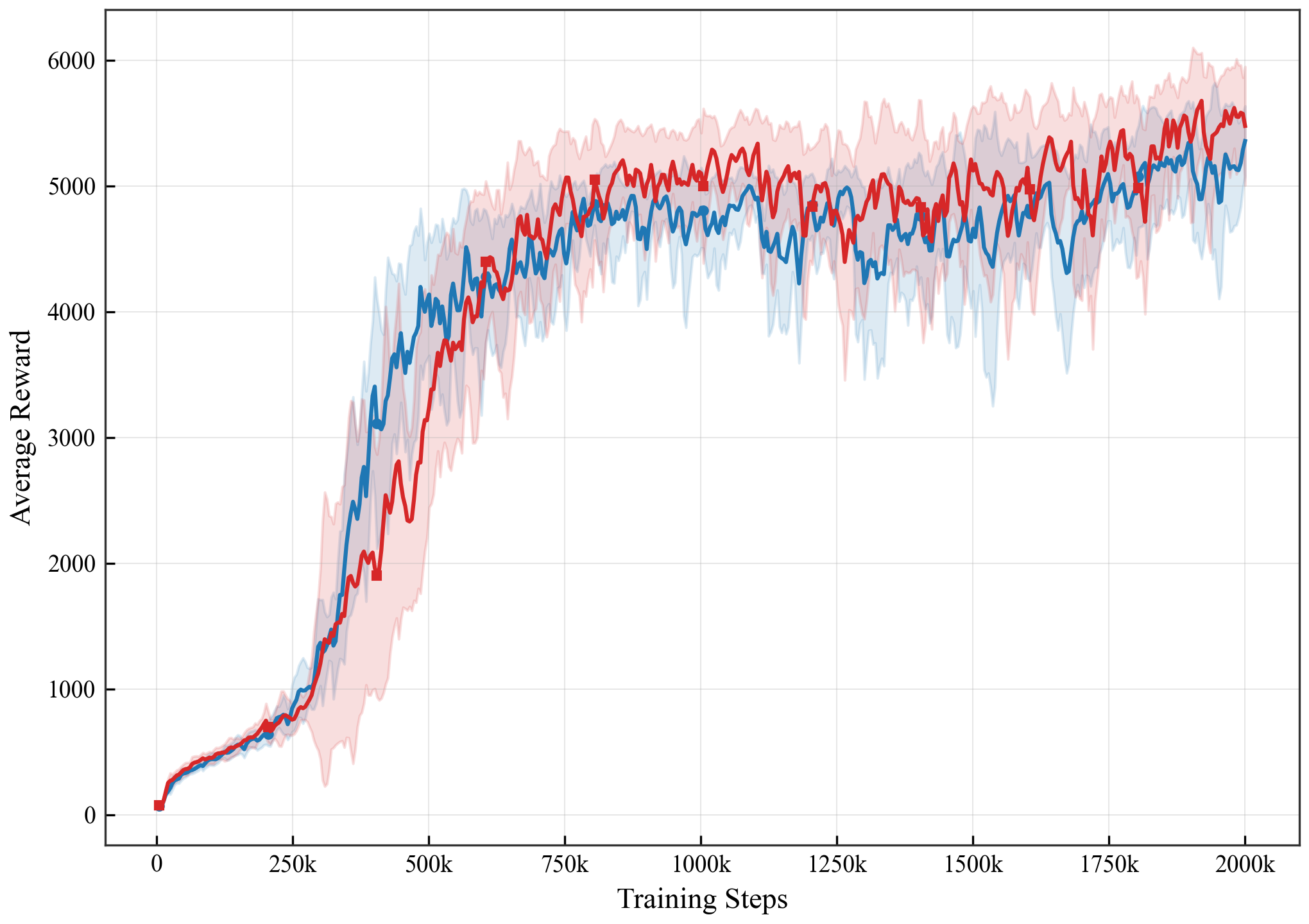}
			\caption{Humanoid-TD3}
		\end{subfigure}
		\vspace{0.5em}
		\begin{subfigure}{0.24\linewidth}
			\includegraphics[width=\linewidth]{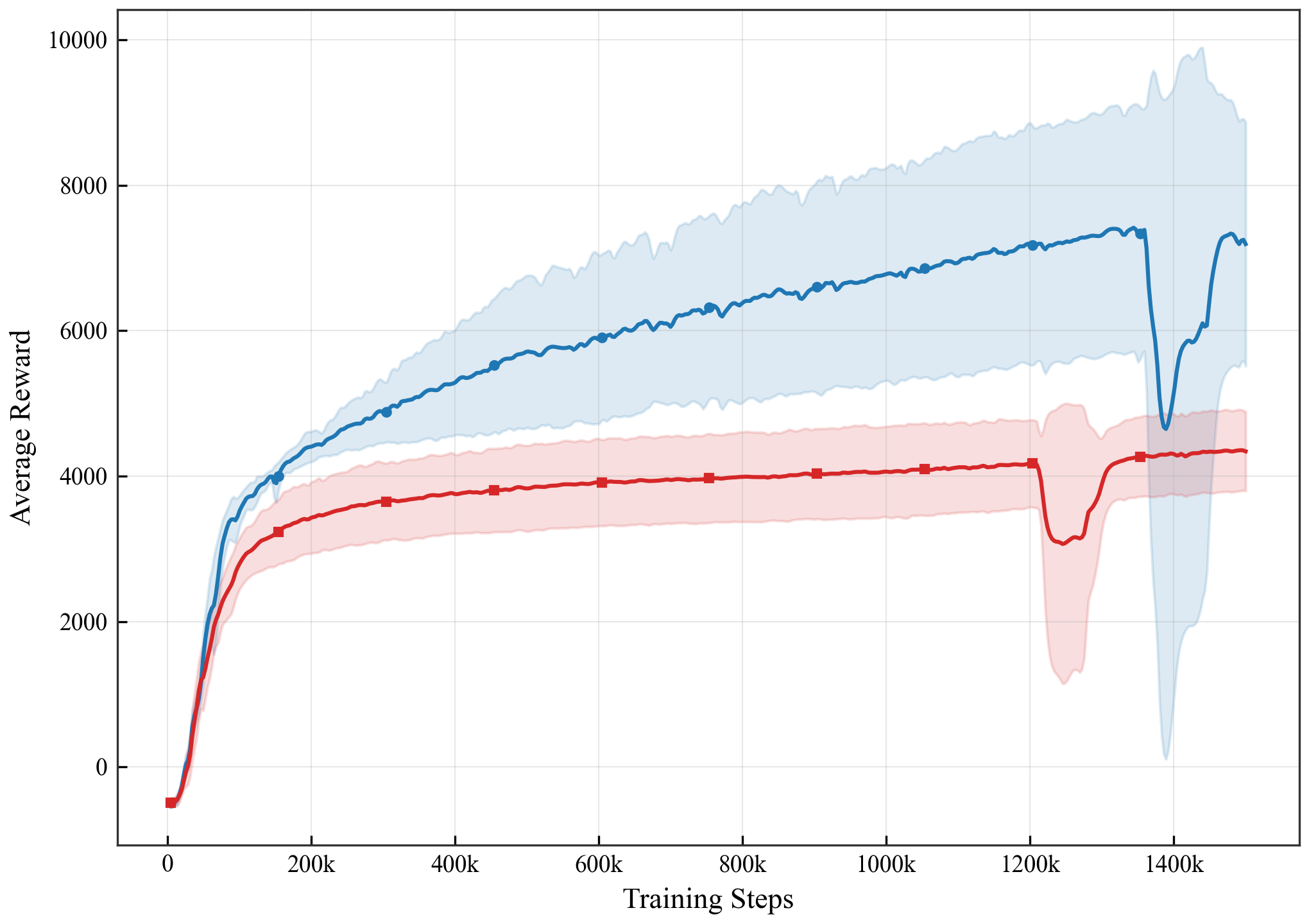}
			\caption{HalfCheetah-TD3}
		\end{subfigure}
		\hfill
		\begin{subfigure}{0.24\linewidth}
			\includegraphics[width=\linewidth]{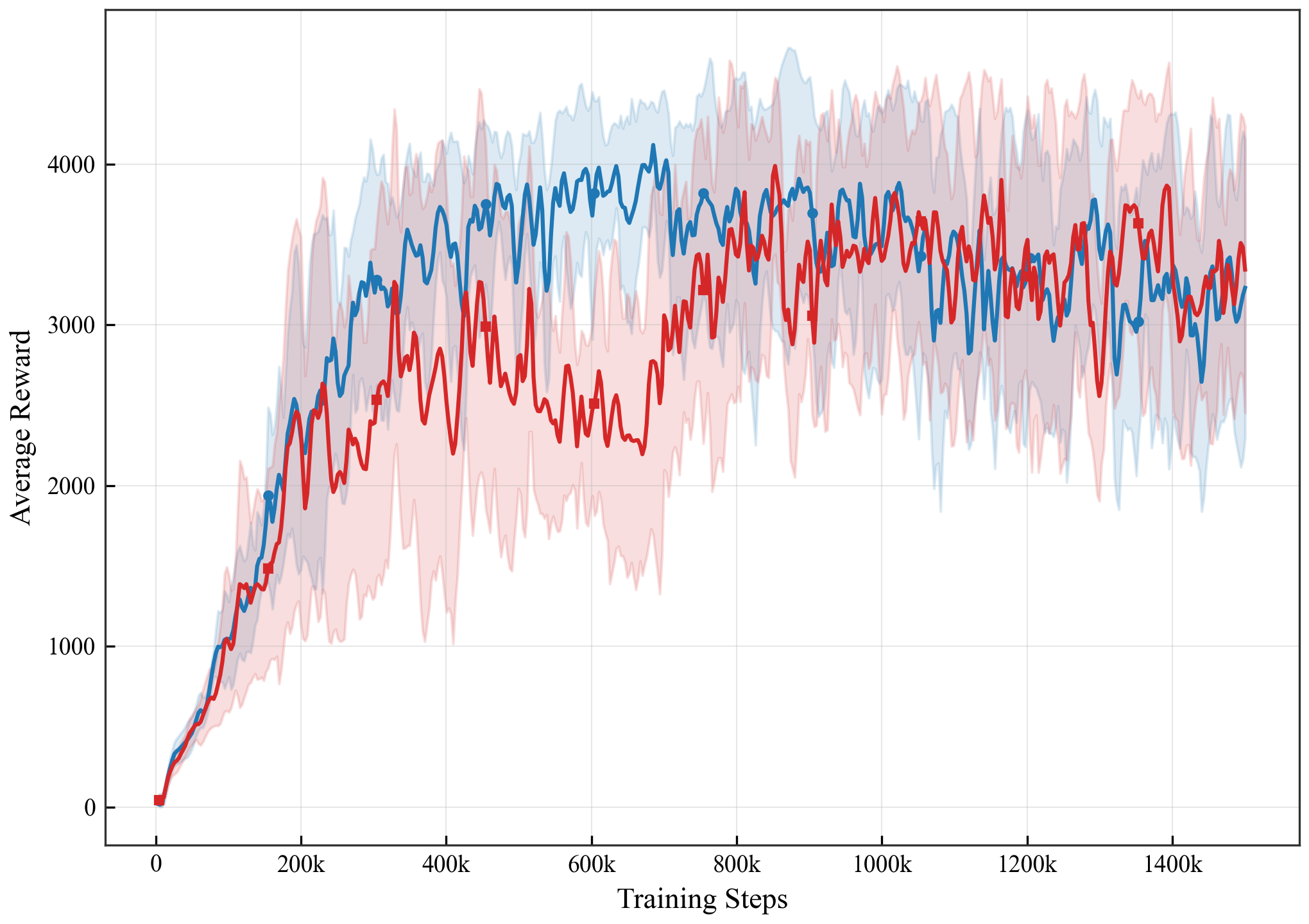}
			\caption{Hopper-TD3}
		\end{subfigure}
		
		\vspace{0.0cm}
		\centering
		\begin{tabular}{@{}l@{\hspace{1.5em}}l@{\hspace{1.5em}}l@{\hspace{1.5em}}l@{\hspace{1.5em}}l@{\hspace{1.5em}}l@{}}
			\colorbox{BoxPre}{\rule{0pt}{1pt}\rule{8pt}{0pt}} \raisebox{-2.0pt}{\scriptsize Wide Homogeneous} &
			\colorbox{DDSRad}{\rule{0pt}{1pt}\rule{8pt}{0pt}} \raisebox{-2.0pt}{\scriptsize Tight Heterogeneous}
		\end{tabular}
		
		\caption{Mean return curves of DD-SRad under wide homogeneous and tight heterogeneous constraints with SAC/TD3 backbones.}
		\label{fig:hetero_return_all}
	\end{figure}
	
	Per-dimension utilization exhibits a layered structure under tight constraints: Ant hip ($\delta=0.2$) and ankle ($\delta=0.5$) groups each independently approach their own $\delta^i$ boundary, rather than being truncated to $\min_i\delta^i$---direct evidence of $\ell_\infty$ parameterization. Radar charts across all environments in Appendix~\ref{app:hetero_radar}.
	
	\subsection{Sensitivity Analysis of $\lambda_{\text{base}}$}
	\label{sec:sensitivity}
	
	$\lambda_{\text{base}}$ is the sole algorithmic hyperparameter; its role is to confine $\|u\|$ within the effective gradient region (Proposition~\ref{prop:gradient}(iii)). This section tests $\lambda_{\text{base}} \in \{0.0001,\,0.001,\,0.01,\,0.1,\,1.0\}$ on Ant-v5 under wide homogeneous ($\boldsymbol{\delta}=[1.0]^8$) and tight heterogeneous ($\boldsymbol{\delta}=[0.2^4,0.5^4]$) constraints.
	
	Figure~\ref{fig:sens} shows that under wide homogeneous constraints, $\lambda_{\text{base}}\in\{0.001,\,0.01,\,0.1,\,1.0\}$ all converge comparably; $\lambda_{\text{base}}=0.0001$ shows $\approx20\%$ return degradation with irregular cross-dimensional utilization, indicating regularization failure.
	
	\begin{figure}[htbp]
		\centering
		\begin{subfigure}{0.28\linewidth}
			\includegraphics[width=\linewidth]{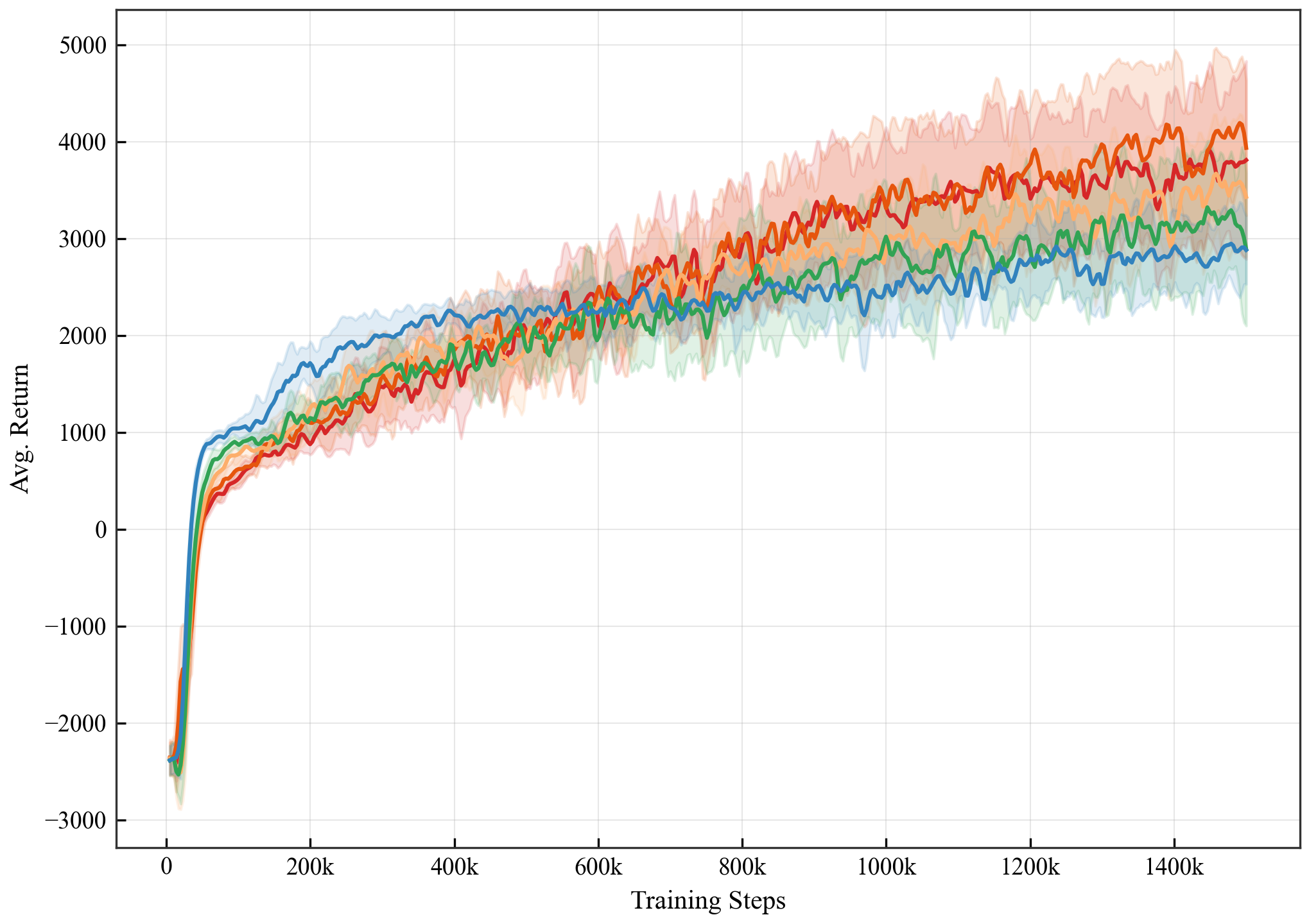}
			\caption{Return (Tight)}
		\end{subfigure}
		\vspace{0.5em}
		\begin{subfigure}{0.28\linewidth}
			\includegraphics[width=\linewidth]{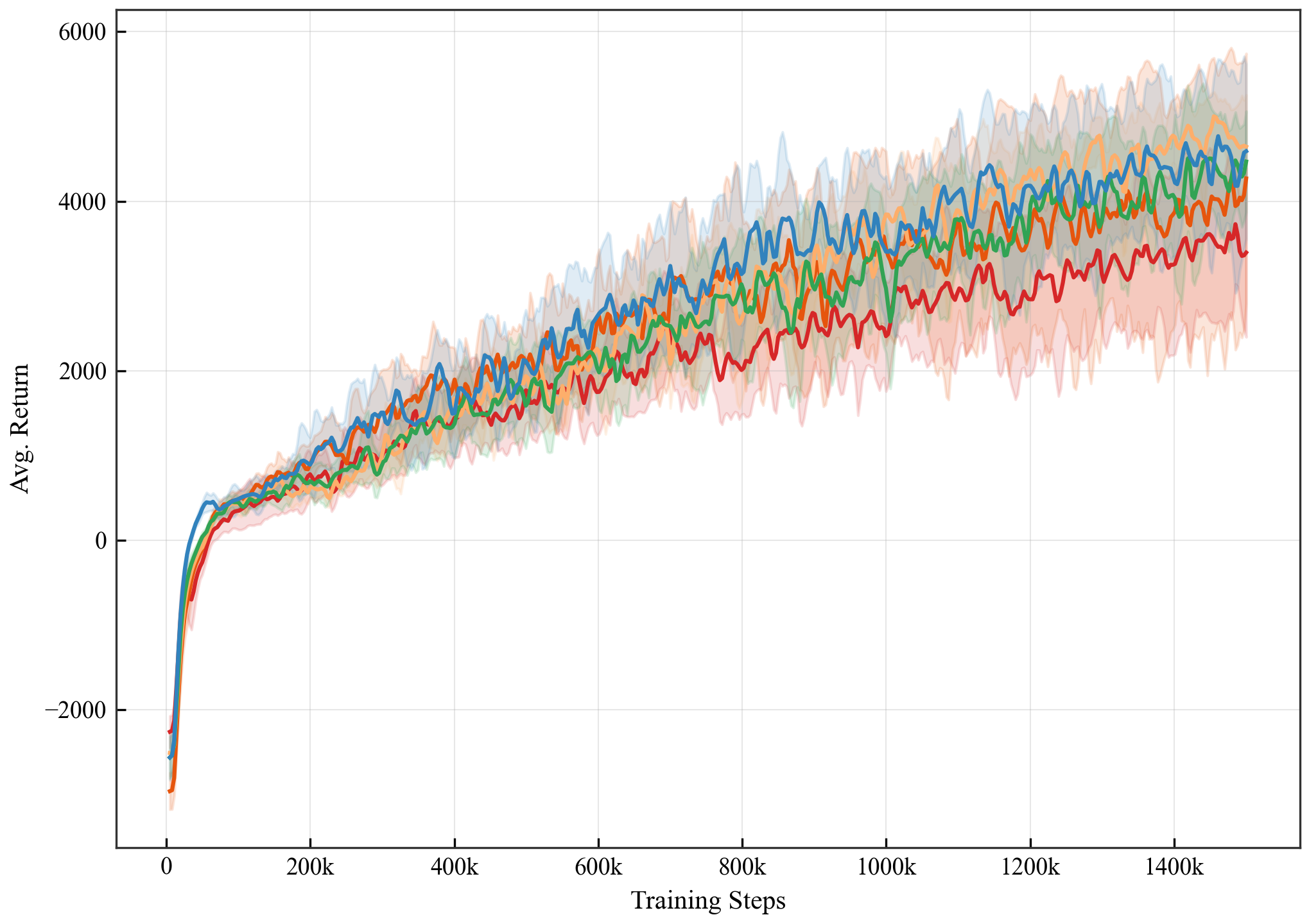}
			\caption{Return (Wide)}
		\end{subfigure}
		\vspace{0.5em}
		\begin{subfigure}{0.21\linewidth}
			\includegraphics[width=\linewidth]{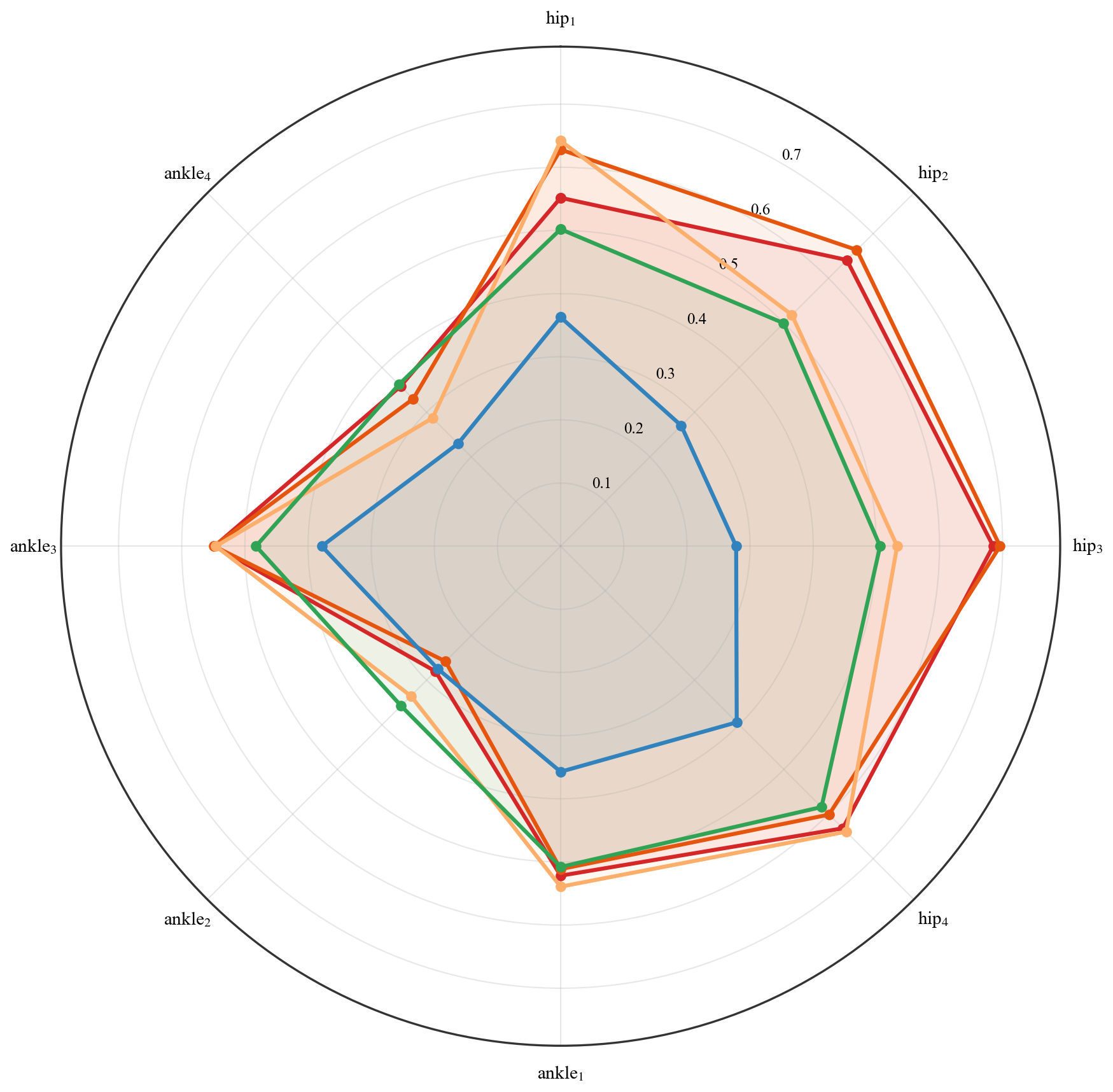}
			\caption{Utilization (Tight)}
		\end{subfigure}
		\vspace{0.5em}
		\begin{subfigure}{0.21\linewidth}
			\includegraphics[width=\linewidth]{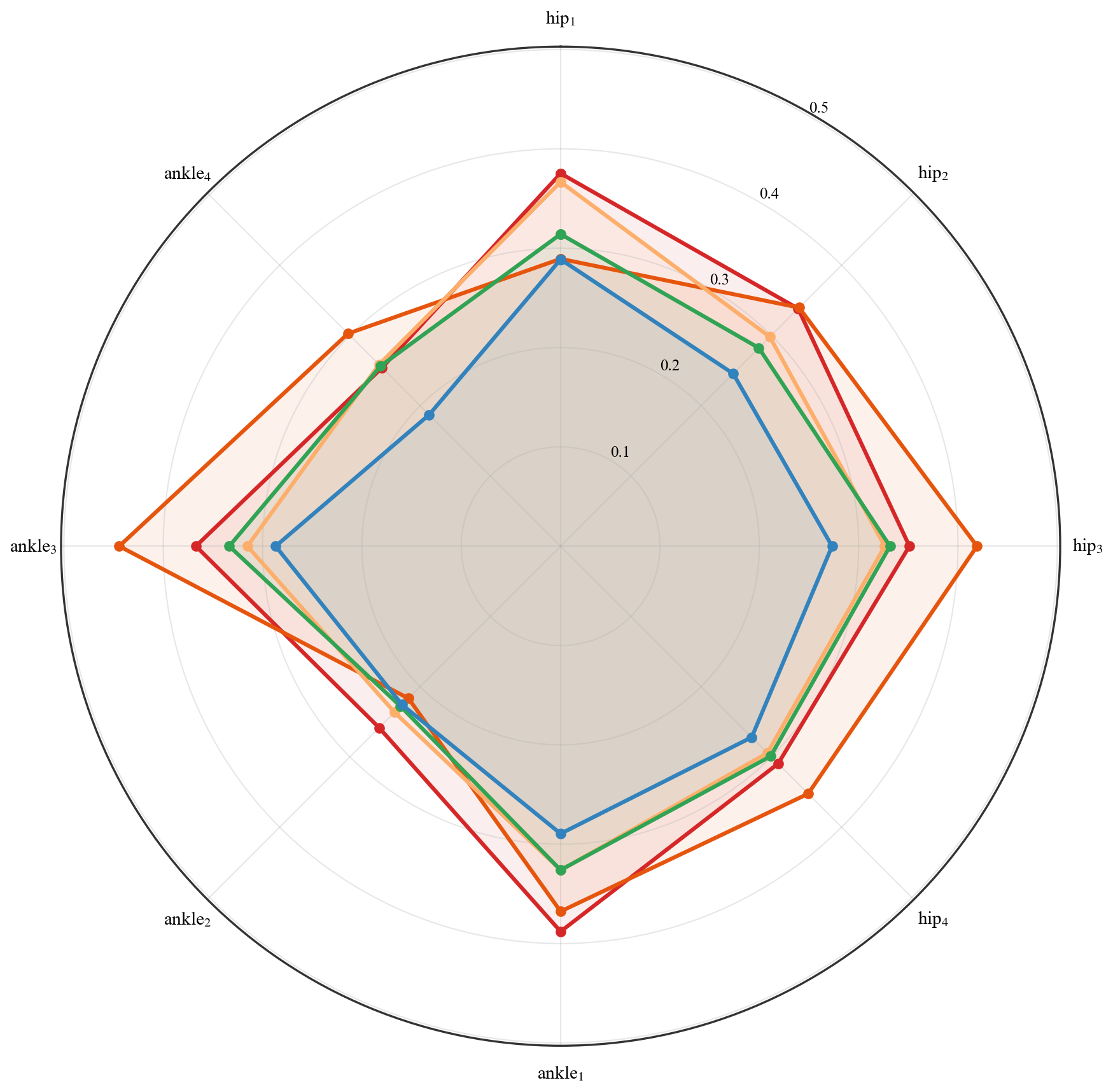}
			\caption{Utilization (Wide)}
		\end{subfigure}
		
		\vspace{-0.3cm}
		\centering
		\begin{tabular}{@{}l@{\hspace{1.5em}}l@{\hspace{1.5em}}l@{\hspace{1.5em}}l@{\hspace{1.5em}}l@{\hspace{1.5em}}l@{}}
			\colorbox{1e-4}{\rule{0pt}{1pt}\rule{8pt}{0pt}} \raisebox{-2.0pt}{\scriptsize 1e-4} &
			\colorbox{1e-3}{\rule{0pt}{1pt}\rule{8pt}{0pt}} \raisebox{-2.0pt}{\scriptsize 1e-3} &
			\colorbox{1e-2}{\rule{0pt}{1pt}\rule{8pt}{0pt}} \raisebox{-2.0pt}{\scriptsize 1e-2} &
			\colorbox{1e-1}{\rule{0pt}{1pt}\rule{8pt}{0pt}} \raisebox{-2.0pt}{\scriptsize 1e-1} &
			\colorbox{1e0}{\rule{0pt}{1pt}\rule{8pt}{0pt}} \raisebox{-2.0pt}{\scriptsize 1e0}
		\end{tabular}
		
		\caption{Mean return and per-dimension utilization of DD-SRad (SAC) on Ant-v5 under varying $\lambda_{\text{base}}$.}
		\label{fig:sens}
	\end{figure}
	
	Under tight heterogeneous constraints (Figure~\ref{fig:sens}(a)(c)), the influence pattern changes qualitatively and becomes non-monotone: $\lambda_{\text{base}}=0.01$ achieves the highest return, with both smaller and larger values degrading performance. Two competing failure modes drive this structure: at $\lambda_{\text{base}}=0.0001$, insufficient regularization causes runaway per-dimension imbalance---the radar chart shows ankle dimensions overexpanding above $0.65$ while hip dimensions collapse irregularly, yielding the worst return despite maximal exploration freedom; at $\lambda_{\text{base}}=1.0$, excessive regularization uniformly suppresses utilization to $\approx0.2$, eliminating the hip--ankle differential and destroying differentiated coverage. Combining both regimes: wide constraints favor $\lambda_{\text{base}}=0.001$ while tight constraints favor
	$\lambda_{\text{base}}=0.01$; $\lambda_{\text{base}}=0.005$ lies at the geometric midpoint and performs consistently well across both, requiring no per-environment tuning.
	
	\section{Simulations}
	\label{sec:simulations}
	
	IsaacLab high-fidelity simulation with per-joint rate limits from official specifications verifies cross-platform $\ell_\infty$ reachable-set consistency (Gap metric) under contact dynamics, sensor noise, and joint temporal coupling unavailable in MuJoCo.
	
	\subsection{Experimental Setup}
	\label{sec:sim_setup}
	
	Experiments select Unitree H1 (10 joints, rough terrain, $\kappa\approx2.2$) and G1 (13 joints with torso, flat terrain, $\kappa=4.0$), providing controlled comparison of $\kappa$ effects under different terrain conditions. Rate limits are read from official specifications and normalized; tightened constraints verify hard guarantee validity; details in Appendix Table~\ref{tab:sim_constraint_params}.
	
	Experiments uniformly use the FastSAC~\cite{seo2025fastsac} backbone---an SAC variant for large-scale parallel simulation ($N_{\text{env}}=4096$) with stable humanoid locomotion training; its policy gradient structure is mathematically identical to standard SAC, and DD-SRad compatibility is guaranteed by §\ref{sec:offpolicy}. Baselines: SAC-DD-SRad (Ours), SAC-D-Tanh, SAC-BoxPre+, SAC-Post(QP), SRad-QP, and SRad-Strict. Additional metrics: VTE$\downarrow$, Fall$\downarrow$, Steps$\uparrow$, Util-P/Util-D/Gap (Appendix Figure~\ref{fig:sim_distribution}).
	
	\subsection{Simulation Results}
	\label{sec:sim_results}
	
	\begin{figure}[htbp]
		\centering
		\begin{subfigure}{0.26\linewidth}
			\includegraphics[width=\linewidth]{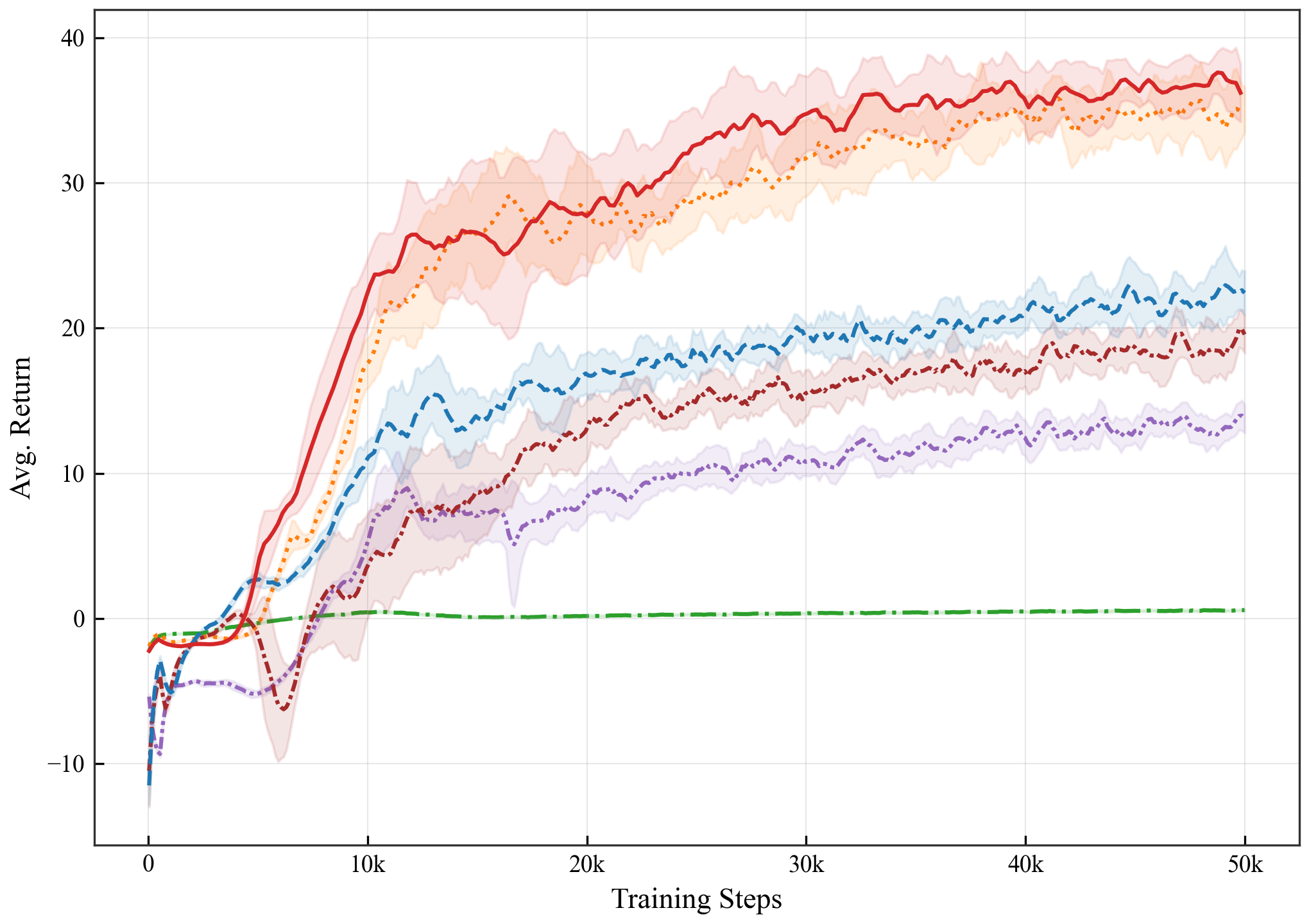}
			\caption{H1 Mean Return}
			\label{fig:sim_tar_h1}
		\end{subfigure}
		\hfill
		\begin{subfigure}{0.26\linewidth}
			\includegraphics[width=\linewidth]{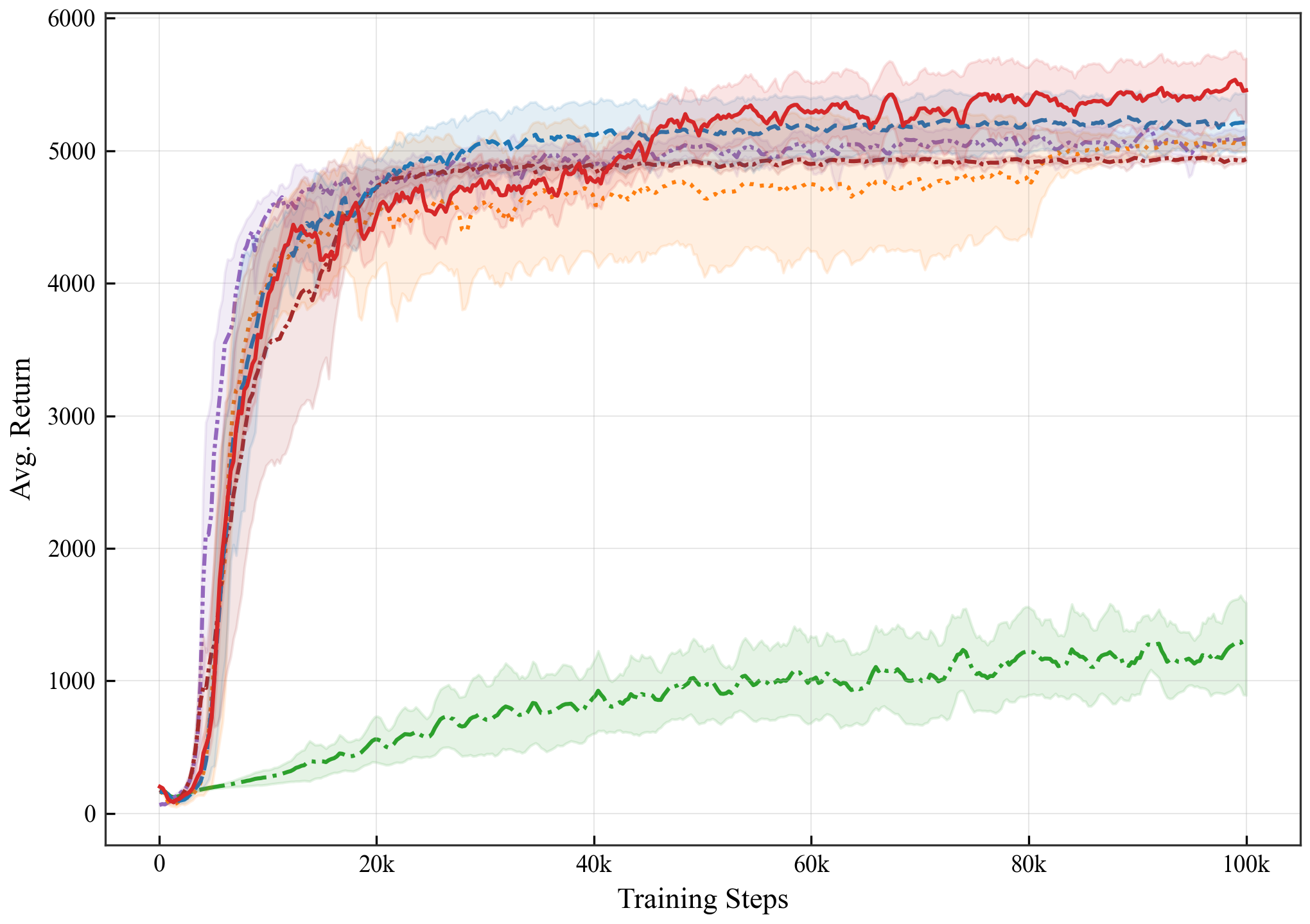}
			\caption{G1 Mean Return}
			\label{fig:sim_tar_g1}
		\end{subfigure}
		\hfill
		\begin{subfigure}{0.19\linewidth}
			\includegraphics[width=\linewidth]{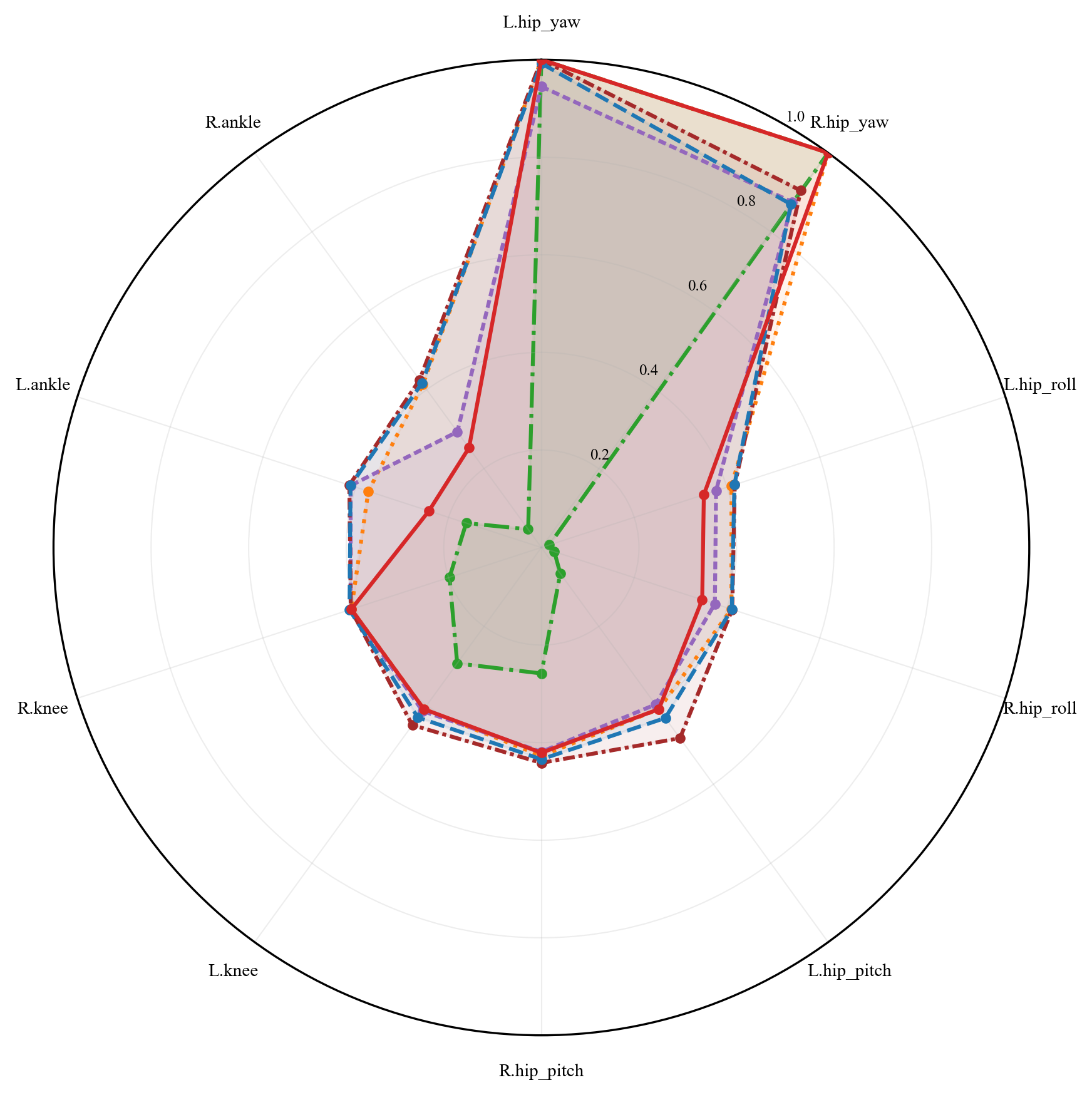}
			\caption{H1 Radar Chart}
			\label{fig:sim_radar_h1}
		\end{subfigure}
		\hfill
		\begin{subfigure}{0.20\linewidth}
			\includegraphics[width=\linewidth]{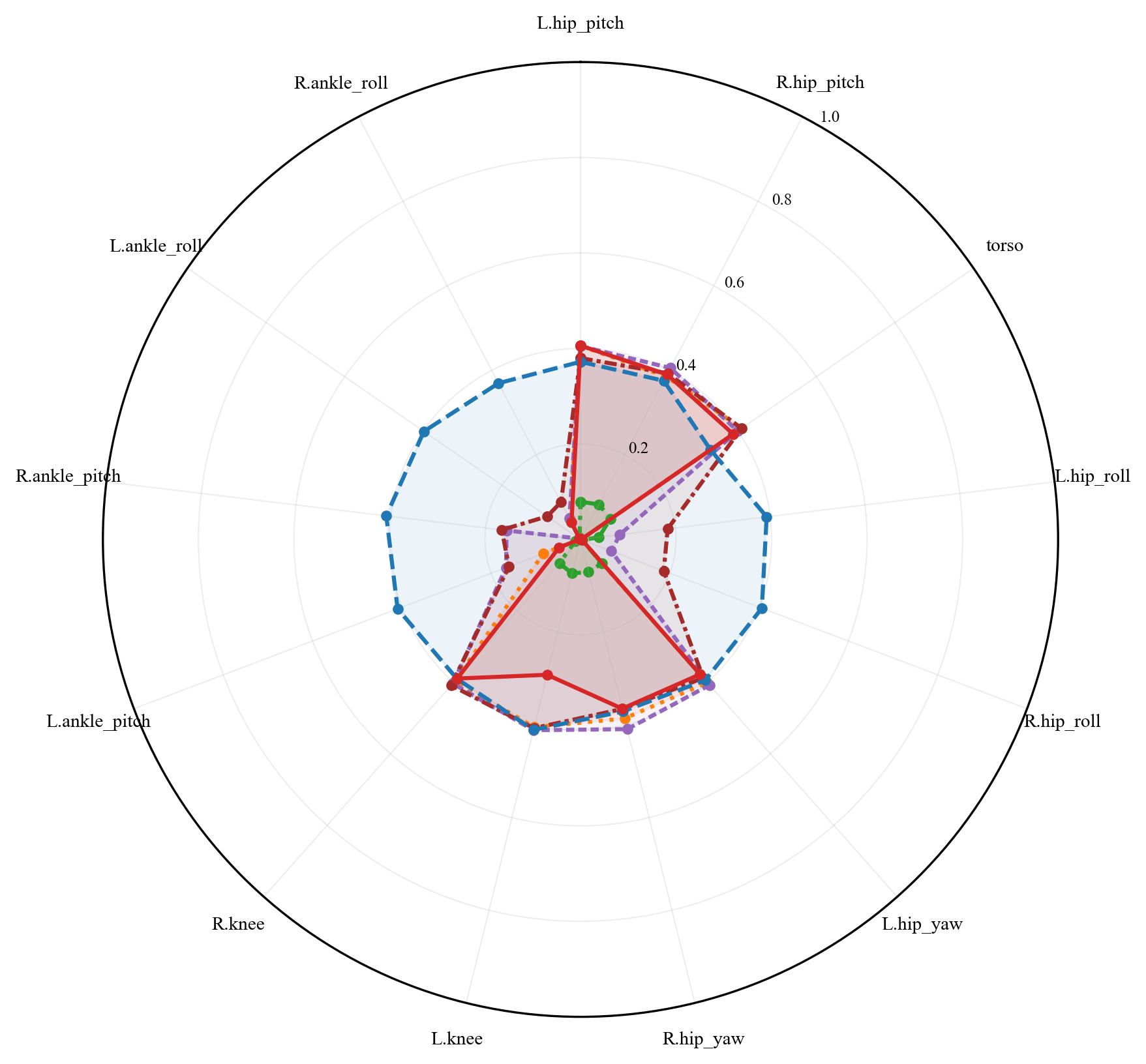}
			\caption{G1 Radar Chart}
			\label{fig:sim_radar_g1}
		\end{subfigure}
		
		\vspace{0cm}
		\centering
		\begin{tabular}{@{}l@{\hspace{1.5em}}l@{\hspace{1.5em}}l@{\hspace{1.5em}}l@{\hspace{1.5em}}l@{\hspace{1.5em}}l@{}}
			\colorbox{BoxPre}{\rule{0pt}{1pt}\rule{8pt}{0pt}} \raisebox{-2.0pt}{\scriptsize BoxPre+} &
			\colorbox{QP}{\rule{0pt}{1pt}\rule{8pt}{0pt}} \raisebox{-2.0pt}{\scriptsize QP} &
			\colorbox{SRadQP}{\rule{0pt}{1pt}\rule{8pt}{0pt}} \raisebox{-2.0pt}{\scriptsize SRad-QP} &
			\colorbox{SRadStrict}{\rule{0pt}{1pt}\rule{8pt}{0pt}} \raisebox{-2.0pt}{\scriptsize SRad-Strict} &
			\colorbox{dtanh}{\rule{0pt}{1pt}\rule{8pt}{0pt}} \raisebox{-2.0pt}{\scriptsize D-Tanh} &
			\colorbox{DDSRad}{\rule{0pt}{1pt}\rule{8pt}{0pt}} \raisebox{-2.0pt}{\scriptsize DD-SRad}
		\end{tabular}
		
		\caption{Learning curves and radar charts for H1 and G1 (mean $\pm$ std over 5 seeds; official constraints).}
		\label{fig:sim_return_radar}
	\end{figure}
	
	On H1 (Table~\ref{tab:sim_full}(a)), DD-SRad leads at $37.14{\pm}2.84$ (Fall$=48.7\%$, Steps$=748$); D-Tanh ($34.56$), BoxPre+ ($23.11$, Fall$=70.2\%$), SAC-Post(QP) ($19.53$), SRad-QP ($13.64$), and SRad-Strict ($0.83$, Fall$=100\%$) follow, with gaps amplified by rough terrain. On G1, DD-SRad achieves $5473{\pm}238$ (Fall$=0.3\%$, Steps$=998$); BoxPre+ ($5213$) ranks second in Return, slightly exceeding D-Tanh ($5145$), yet shows 62\% higher VTE ($0.224$ vs $0.138$), revealing gait instability from clip's uniform activation; SRad-Strict ($1312$, Fall$=91\%$) confirms high-$\kappa$ collapse.
	
	Table~\ref{tab:sim_full}(b) shows per-joint-group utilization. DD-SRad achieves task-adaptive allocation: balanced proximal--distal on H1 (Gap$=0.10$); hip-dominant on G1 (Util-P$=0.40$, Util-D$\approx0.02$, Gap$=0.38$). BoxPre+'s near-equal utilization (Gap$\approx0.01$) directly causes 62\% higher G1 VTE: consistent with Proposition~\ref{prop:gradient}, gradient truncation at clip boundaries prevents the critic signal from driving dimension-selective allocation. SRad-Strict Util-P$=0.08$ falls far below $\min_i\delta^i/\delta^\mathrm{P}{\approx}0.45$ (Theorem~\ref{thm:exploration}).
	
	\renewcommand{\arraystretch}{1.08}
	\begin{table*}[!t]
		\scriptsize
		\centering
		\captionsetup{font=scriptsize, labelfont=scriptsize}
		\caption{Comprehensive performance in IsaacLab high-fidelity simulation (FastSAC backbone~\cite{seo2025fastsac}; H1 Rough \& G1 Flat; mean$\pm$std over 5 seeds; \textbf{bold} is best among end-to-end methods; all constrained methods satisfy CVR\,${=}$\,0 (constraint layer enforced directly); $\kappa=\max_i\delta^i/\min_i\delta^i$: H1 $\kappa{\approx}2.2$, G1 $\kappa{=}4.0$)}
		\label{tab:sim_full}
		
		\textbf{(a) Task-Level Performance}
		(Return$\uparrow$: cumulative episode reward;
		VTE$\downarrow$: velocity tracking error (m/s);
		Fall$\downarrow$: base contact termination rate (\%);
		Steps$\uparrow$: mean episode length)\\[1mm]
		
		\begin{tabular}{l cccc cccc}
			\toprule
			& \multicolumn{4}{c}{\textbf{Unitree H1}(Humanoid,Rough,$\kappa{\approx}2.2$)}
			& \multicolumn{4}{c}{\textbf{Unitree G1}(Humanoid,Flat,$\kappa{=}4.0$)} \\
			\cmidrule(lr){2-5}\cmidrule(lr){6-9}
			\textbf{Method}
			& Return$\uparrow$ & VTE$\downarrow$ & Fall$\downarrow$ & Steps$\uparrow$
			& Return$\uparrow$ & VTE$\downarrow$ & Fall$\downarrow$ & Steps$\uparrow$ \\
			\midrule
			BoxPre+
			& $23.11{\pm}1.87$          & $\mathbf{0.385{\pm}0.026}$ & $70.2$ & $526$
			& $5213{\pm}241$            & $0.224{\pm}0.028$ & $5.7$           & $970$ \\
			Post(QP)
			& $19.53{\pm}2.43$          & $0.465{\pm}0.064$ & $88.0$          & $339$
			& $4954{\pm}42$             & $0.139{\pm}0.006$ & $\mathbf{0.0}$  & $\mathbf{1000}$ \\
			SRad-QP
			& $13.64{\pm}1.33$          & $0.390{\pm}0.044$ & $76.5$          & $431$
			& $5166{\pm}98$             & $0.187{\pm}0.011$ & $0.3$           & $998$ \\
			SRad-Strict
			& $0.83{\pm}0.11$           & $0.624{\pm}0.008$ & $100.0$         & $139$
			& $1312{\pm}344$            & $0.360{\pm}0.036$ & $91.0$          & $403$ \\
			D-Tanh
			& $34.56{\pm}3.04$          & $0.462{\pm}0.041$ & $58.7$ & $679$
			& $5145{\pm}91$             & $0.207{\pm}0.016$ & $2.2$           & $989$ \\
			\rowcolor{blue!10}
			DD-SRad (Ours)
			& $\mathbf{37.14{\pm}2.84}$ & $0.459{\pm}0.023$ & $\mathbf{48.7}$ & $\mathbf{748}$
			& $\mathbf{5473{\pm}238}$   & $\mathbf{0.138{\pm}0.009}$ & $0.3$  & $998$ \\
			\bottomrule
		\end{tabular}
		
		\vspace{1mm}
		\scriptsize
		\textbf{(b) Constraint Utilization}
		(Util-P: proximal joint group (H1: $\delta^\mathrm{P}{=}0.40$, G1: $\delta^\mathrm{P}{=}0.40$);
		Util-D: distal joint group (H1: $\delta^\mathrm{D}{=}0.18$, G1: $\delta^\mathrm{D}{=}0.18$);
		Gap\,${=}|$Util-P$-$Util-D$|$)\\[1mm]
		
		{\setlength{\tabcolsep}{9.5pt}
			\begin{tabular}{l ccc ccc}
				\toprule
				& \multicolumn{3}{c}{\textbf{Unitree H1}(Humanoid,Rough,$\kappa{\approx}2.2$)}
				& \multicolumn{3}{c}{\textbf{Unitree G1}(Humanoid,Flat,$\kappa{=}4.0$)} \\
				\cmidrule(lr){2-4}\cmidrule(lr){5-7}
				\textbf{Method} & Util-P & Util-D & Average Gap
				& Util-P & Util-D & Average Gap \\
				\midrule
				BoxPre+
				& $0.42\pm0.00$ & $0.44\pm0.01$ & $0.02$
				& $0.37\pm0.02$ & $0.38\pm0.02$ & $0.01$ \\
				Post(QP)
				& $0.45\pm0.01$ & $0.48\pm0.01$ & $0.03$
				& $0.39\pm0.01$ & $0.09\pm0.01$ & $0.30$ \\
				SRad-QP
				& $0.37\pm0.01$ & $0.42\pm0.01$ & $0.06$
				& $0.41\pm0.01$ & $0.03\pm0.04$ & $0.38$ \\
				SRad-Strict
				& $0.08\pm0.00$ & $0.50\pm0.01$ & $0.42$
				& $0.08\pm0.01$ & $0.01\pm0.00$ & $0.08$ \\
				D-Tanh
				& $0.41\pm0.00$ & $0.62\pm0.05$ & $0.21$
				& $0.40\pm0.01$ & $0.02\pm0.05$ & $0.37$ \\
				\rowcolor{blue!10}
				DD-SRad (Ours)
				& $0.38\pm0.01$ & $0.48\pm0.03$ & $0.10$
				& $0.40\pm0.01$ & $0.02\pm0.04$ & $0.38$ \\
				\bottomrule
		\end{tabular}}
		
		\vspace{1.5mm}
		\scriptsize
		\textbf{Notes:}
		All values are measured on 5 evaluation seeds with 1000 test episodes per seed after convergence; learning curves (Fig.~\ref{fig:sim_tar_h1}, Fig.~\ref{fig:sim_tar_g1}) and radar charts (Fig.~\ref{fig:sim_radar_h1}, Fig.~\ref{fig:sim_radar_g1}) are plotted from the same evaluation data.
		H1 proximal group: hip joints ($\delta{=}0.40$); distal group: ankle joints ($\delta{=}0.18$).
		G1 proximal group: hip\_pitch joints ($\delta{=}0.40$); distal group: ankle joints ($\delta{=}0.18$--$0.25$).
		H1 fall rates reflect terrain difficulty.
	\end{table*}
	\renewcommand{\arraystretch}{1}
	
	\section{Related Work}
	\label{sec:related}
	
	\textbf{Hierarchical and model-based methods.} MPC/QP approaches achieve hard per-step satisfaction but incur solver overhead and training-deployment inconsistency~\cite{williams2017information,nagabandi2018neural,grandia2023perceptive,kim2019highly,pandala2022robust,bledt2018cheetah}; hierarchical RL decouples constraint handling at the cost of objective inconsistency and degraded sim-to-real transfer~\cite{johannink2019residual,lee2020learning,li2020safe,rajeswaran2018learning,gros2020data,vecerik2017leveraging,chebotar2019closing}.
	
	\textbf{Constrained Markov decision processes.} CMDP methods bound expected cumulative costs~\cite{altman1999constrained,achiam2017constrained,tessler2018reward,zhang2020first,yang2022constrained,zhang2022P3O,ray2019benchmarking,wachi2024survey,gu2024review} but permit transient per-step violations, falling short of $|\Delta a_t^i|\leq\delta^i,\forall t$~\cite{bertsekas1996stochastic}.
	
	\textbf{Action parameterization and safety filters.} $\tanh$/squashed-Gaussian outputs enforce static intervals but not inter-step differential constraints~\cite{lillicrap2016continuous,haarnoja2018soft,schulman2017proximal,fujimoto2018td3}; SRad~\cite{kasaura2023benchmarking} imposes a global $\ell_2$ radius that compresses the reachable set under heterogeneous $\delta^i$; CBF-QP/OptLayer filters achieve hard satisfaction but incur QP latency and gradient discontinuities~\cite{ames2017control,pham2018optlayer,dalal2018safe,singletary2021safety}.
	
	\textbf{Generative and flow-based methods.} FlowPG~\cite{brahmanage2023flowpg}, CV-Flows~\cite{cvflows2025}, ARAM~\cite{hung2025aram}, and TN-SAC~\cite{stolz2025truncated} achieve strong performance on static feasible sets but cannot accommodate the per-step $\mathcal{F}_t$ drift induced by $a_{t-1}$.
	
	A detailed per-method analysis is provided in Appendix~\ref{app:related_extended}.
	
	\section{Conclusion}
	\label{sec:conclusion}
	
	This paper identifies the geometric mismatch between the $\ell_\infty$ rate-constraint feasible set and $\ell_2$ spherical parameterization, quantifying the volume ratio as $\prod_i\delta^i/(\min_i\delta^i)^d$. DD-SRad resolves this via position-adaptive $R_{\text{eff}}^i$ computed independently per dimension, with three theoretical guarantees: probability-1 constraint satisfaction (Theorem~\ref{thm:constraint}), diagonally positive definite Jacobian with spectral structure determined by $\{\delta^i\}$ (Proposition~\ref{prop:gradient}), and $\ell_\infty$-tight coverage with volume ratio $\prod_i\delta^i/(\min_i\delta^i)^d$ relative to $\ell_2$ (Theorem~\ref{thm:exploration}). MuJoCo experiments achieve the highest return at zero violation across all 8 environment--backbone configurations, with 30\%--50\% coverage improvement over $\ell_2$ baselines and $\lambda_{\text{base}}$ robust over two orders of magnitude. IsaacLab H1/G1 simulations confirm end-to-end optimality parameterized directly from official joint rate specifications, validating the systematic pathway from hardware specifications to safe deployment without reward engineering.
	

	\appendix
	
	\section{Extended Related Work}
	\label{app:related_extended}
	
	The most mature engineering solution for handling actuator constraints is the hierarchical control architecture, which decouples the high-level RL policy from the low-level constraint handling. Williams et al.~\cite{williams2017information} proposed information-theoretic MPPI, solving constrained optimization within the control period via large-scale parallel sampling and demonstrating strong online planning capability on racing trajectory tracking tasks; however, the computational cost of MPPI grows steeply with state dimension, making it difficult to maintain the required control frequency on high-DOF humanoid robots. Nagabandi et al.~\cite{nagabandi2018neural} learned system dynamics with a neural network and ran MPC on top, achieving online adaptation to changes in environment parameters; however, the accuracy of the learned dynamics model degrades away from the training distribution, causing constraint satisfaction to deteriorate---particularly evident in contact-rich locomotion tasks. Grandia et al.~\cite{grandia2023perceptive} embedded perception into nonlinear MPC to achieve dynamic locomotion on rough terrain, demonstrating MPC's constraint-handling capability in complex environments; however, the per-step solve time of approximately 15\,ms, together with perception processing overhead, makes generalization to higher-frequency control scenarios difficult. Kim et al.~\cite{kim2019highly} proposed a combined whole-body impulse control and MPC framework achieving highly dynamic gaits on a quadruped; the method's reliance on precise dynamics models limits its robustness to parametric uncertainty and external disturbances. Pandala et al.~\cite{pandala2022robust} introduced robust predictive control to reduce the gap between simplified and full-order models; however, its conservative design further compresses the reachable action space under tight rate constraints, affecting motion performance.
	
	Johannink et al.~\cite{johannink2019residual} proposed residual reinforcement learning, stacking an RL policy on top of a classical controller to enable safe exploration with fewer samples in robotic manipulation; however, the performance ceiling of this framework is limited by the quality of the base controller, and the high-level--low-level objective inconsistency is especially pronounced near constraint boundaries. Lee et al.~\cite{lee2020learning} completed end-to-end locomotion learning on challenging terrain with a real quadruped, demonstrating the potential of pure RL under complex dynamics; however, that work does not explicitly handle joint rate limits, relying on reward engineering for soft constraints on abnormal actions without providing hard guarantees. Li et al.~\cite{li2020safe} designed an RL-based safe hierarchical planning framework for complex driving scenarios; the high-level policy cannot perceive the low-level constraint boundaries at inference time, leading to biased policy gradient estimates near constraint regions, and empirical studies show that actual execution performance degrades by approximately 18\% compared to end-to-end methods~\cite{rajeswaran2018learning}. Gros and Zanon~\cite{gros2020data} proposed a data-driven economic nonlinear MPC framework that theoretically unifies RL and MPC; however, their framework requires differentiable dynamics models and incurs high computational complexity, limiting real-time deployment feasibility.
	
	The above hierarchical approaches share three structural difficulties: objective inconsistency between the high-level RL and low-level constraint handler---the high-level optimizes long-term cumulative return while the low-level minimizes tracking error, and these two objectives conflict fundamentally when constraints are active; computational overhead of the low-level solver (MPC, QP) in high-dimensional systems significantly compresses control bandwidth, with the convex MPC solve on MIT Cheetah 3 reducing the control frequency from an ideal 500\,Hz to approximately 40\,Hz in practice~\cite{bledt2018cheetah}; and training-deployment inconsistency making sim-to-real transfer difficult~\cite{vecerik2017leveraging,chebotar2019closing}---the policy observes idealized execution during training, while low-level corrections change the action effect at deployment, with performance degradation of 25--40\%.
	
	From a constrained RL perspective, Altman~\cite{altman1999constrained} established the theoretical foundation of constrained Markov decision processes (CMDP), formalizing constraint satisfaction as expected cumulative constraints $\mathbb{E}_\pi[\sum_t c_t] \leq d$; the core limitation of this framework is that it permits transient violations as long as the long-term average is satisfied, whereas rate constraints require $|\Delta a_t^i| \leq \delta^i, \forall t$ to hold strictly at every time step---two constraint types with essentially different mathematical structures. Achiam et al.~\cite{achiam2017constrained} proposed Constrained Policy Optimization (CPO), maintaining constraint feasibility during policy updates via second-order trust-region methods and establishing an important reference point for constrained RL on the Safety Gym benchmark; however, CPO handles expected-form constraints, exhibits high gradient estimate variance in sparse violation scenarios, and requires solving a linear system at each step with considerable computational cost. Tessler et al.~\cite{tessler2018reward} proposed RCPO, converting constraints into penalty terms via dynamic Lagrangian multiplier adjustment; this method cannot eliminate single-step violations at test time even after convergence, lacking hard guarantees. Zhang et al.~\cite{zhang2020first} proposed FOCOPS, approximating CPO's constrained policy update with first-order optimization to substantially reduce computational cost; however, its theoretical guarantees are equally limited to expected constraints and are ineffective for per-step rate limits. Yang et al.~\cite{yang2022constrained} proposed CUP, decoupling policy improvement from constraint satisfaction via constrained update projection; however, this method's conservatism near constraint boundaries reduces exploration efficiency, leading to insufficient constraint utilization under high-$\kappa$ heterogeneous constraints. Zhang et al.~\cite{zhang2022P3O} proposed P3O, balancing task performance and constraint satisfaction with adaptive penalty coefficients; its online tuning mechanism is sensitive to hyperparameters, with questionable stability under rapidly varying rate constraints. Ray et al.~\cite{ray2019benchmarking} systematically evaluated the performance of multiple safe RL algorithms, revealing an inherent trade-off between constraint satisfaction rate and task performance in existing CMDP methods, but providing no specialized solutions for the temporally dependent rate-constraint scenario. Survey works by Wachi et al.~\cite{wachi2024survey} and Gu et al.~\cite{gu2024review} further confirm that existing safe RL theoretical frameworks cannot mathematically cover the per-step hard guarantee requirements of rate constraints.
	
	At the policy parameterization level, Schulman et al.~\cite{schulman2017proximal} proposed PPO, which has found wide application in on-policy settings; however, PPO itself does not handle action-domain constraints, allowing actions to exceed physical limits during training, and requiring additional constraint handling at deployment. Fujimoto et al.~\cite{fujimoto2018td3} proposed TD3, improving critic estimate accuracy via target policy smoothing; similarly, the standard TD3 framework has no native support for rate constraints. Lillicrap et al.~\cite{lillicrap2016continuous} introduced a $\tanh$ output layer in DDPG to restrict actions to $[-1,1]$; this method hard-codes the action range as a uniform symmetric interval and cannot handle dynamic shifting feasible sets centered at $a_{t-1}$. Haarnoja et al.~\cite{haarnoja2018soft} proposed SAC, adopting a squashed Gaussian policy to explicitly model action boundaries while maintaining entropy maximization; however, squashed Gaussians are equally applicable only to static position-domain constraints and cannot express inter-step action differential relationships. Kasaura et al.~\cite{kasaura2023benchmarking} systematically studied the effect of action constraints on actor-critic algorithms and proposed Spherical Radial Squashing (SRad), constraining policy output within a ball of center $c_s$ and radius $R$ via $a = c_s + R \cdot u/\sqrt{1+\|u\|^2}$, outperforming naive clip and QP post-processing on robotic control constraint benchmarks; SRad performs well under homogeneous constraints, but its global radius design faces a structural trade-off under heterogeneous constraints---if $R = \min_i\delta^i$, the reachable set for high-budget dimensions is compressed to a fraction $\min_i\delta^i/\delta^i$ of the original radius; if $R > \min_i\delta^i$, low-budget dimensions are violated; no globally valid and tight radius selection exists. Singletary et al.~\cite{singletary2021safety} proposed a safety-critical kinematics control framework using control barrier functions (CBF) to guarantee safety constraints at the kinematic level; the CBF approach requires high model accuracy and faces computational bottlenecks with real-time CBF-QP solving in high-dimensional joint spaces. Pham et al.~\cite{pham2018optlayer} proposed OptLayer, inserting a differentiable optimization layer between the policy and environment to project infeasible actions to the nearest feasible point, achieving hard constraint satisfaction; however, QP solve time under high-dimensional non-convex constraints grows with dimension, and the Jacobian computation of the differentiable projection layer introduces additional latency in real-time control. Dalal et al.~\cite{dalal2018safe} proposed Safe Layer, learning a safe projection via linear approximation of constraint functions to achieve low-cost action correction; linear approximation accuracy degrades near constraint boundaries, causing protection failures under high-frequency rate constraint activation. Ames et al.~\cite{ames2017control} proposed a CBF-based QP safety filter, guaranteeing forward invariance in control input space via QP solving with rigorous continuous-time safety guarantees; however, adapting CBF-QP to discrete-time RL differential action constraints requires conservative margin design, excessively restricting the action range during normal motion.
	
	Recent generative model-based methods attempt to fundamentally change the constraint handling paradigm. Brahmanage et al.~\cite{brahmanage2023flowpg} proposed FlowPG, sampling directly from the feasible set via normalizing flows, achieving excellent performance on static action constraint benchmarks; however, FlowPG requires pre-building a flow model of the global feasible set, while the feasible set $\mathcal{F}_t$ of rate constraints drifts with each step's $a_{t-1}$, making it impossible to establish a static distribution, rendering FlowPG inapplicable to rate-constraint scenarios. Conditional normalizing flows can theoretically model time-varying feasible sets conditioned on $a_{t-1}$, but executing a full $L$-layer flow forward pass at each inference step introduces unacceptable real-time latency at 50\,Hz control frequency; DD-SRad's analytic parameterization incurs zero inference overhead. Brahmanage et al.~\cite{cvflows2025} further proposed CV-Flows, augmenting flow model training with constraint violation signals to improve sampling quality near boundaries; however, its assumption of a static feasible set equally prevents it from handling temporally dependent rate constraints. Hung et al.~\cite{hung2025aram} proposed ARAM, efficiently realizing constraint-feasible action sampling via acceptance-rejection sampling and augmented MDPs, with sample efficiency significantly superior to MCMC methods under static constraints; however, the acceptance-rejection framework equally requires relatively stable feasible sets during training and cannot directly accommodate rate constraints where the feasible set shifts dynamically with $a_{t-1}$ at each step. Stolz et al.~\cite{stolz2025truncated} proposed defining a truncated Gaussian distribution policy directly over constrained intervals, resolving the abnormal probability density of squashed Gaussians at constraint boundaries via closed-form/approximate normalization constants and entropy corrections, achieving significant performance improvements on static and state-dependent action constraint benchmarks; the original work does not cover the temporally dependent feasible set $\mathcal{F}_t$ induced by $a_{t-1}$, and adapting it to rate constraints requires state augmentation and recomputing truncation endpoints at each step. This approach and DD-SRad represent two parallel design routes within action-constrained RL: the former focuses on correct density modeling of the policy distribution, the latter on analytic parameterization of $\ell_\infty$ geometry; the two share the same reachable set under an $\ell_\infty$ feasible set, differing in that one handles the density function for stochastic sampling while the other uses a deterministic mapping such that entropy and log-probability computations in $u$-space remain consistent with standard SAC.

	Synthesizing the above analysis, existing methods share a common geometric mismatch problem when handling rate constraints. The rate constraint $|a_t^i - a_{t-1}^i| \leq \delta^i$ defines an $\ell_\infty$ hyperrectangle feasible set $\mathcal{F}_t = \{a : |a^i - a_{t-1}^i| \leq \delta^i,\ \forall i\}$ in the differential action space, while $\ell_2$ spherical parameterization (SRad) and the static feasible set assumption of generative model methods are both incompatible with this $\ell_\infty$ geometric structure. Under heterogeneous constraints ($\kappa \gg 1$), the effective exploration radius of $\ell_2$ spherical parameterization is limited to $R \leq \min_i\delta^i$, compressing the reachable set for high-budget dimensions to a fraction $\min_i\delta^i/\delta^i$ of the original radius---corresponding to approximately 60\% radius loss at $\kappa = 2.5$. The reachable set of $\ell_\infty$ projection methods (such as per-dimension clip) is consistent with $\mathcal{F}_t$, but clip truncates the gradient at the boundary of the feasible set, causing the policy gradient to lose directional information as the constraint approaches saturation.
	
	\section{Proofs of Theoretical Results}
	\label{app:proofs}
	
	This appendix collects complete proofs of all theoretical results stated in \S\ref{sec:theory}.
	Notation follows \S\ref{sec:methodology} throughout.
	Two elementary facts are used repeatedly throughout the proofs.
	
	\begin{itemize}
		\item \textbf{Squashing bound.}
		For $f(t) \triangleq t/\sqrt{1+t^2}$, one has $|f(t)| < 1$ for all $t \in \mathbb{R}$
		and $|f(t)| \leq |t|$ (since $\sqrt{1+t^2} \geq 1$).
		\item \textbf{Squashing derivative.}
		$f'(t) = (1+t^2)^{-3/2}$, which satisfies $0 < f'(t) \leq 1$ for all $t \in \mathbb{R}$.
	\end{itemize}

	\subsection{Proof of Augmented MDP Validity}
	\label{app:proof:augmented_mdp}
	
	\begin{proof}
		The three properties are verified in turn.
		
		\medskip
		\noindent\textbf{(i) Markov property.}
		By construction, $\tilde{s}_{t+1} = (s_{t+1}, a_t)$, where $s_{t+1}$ is drawn from the original transition kernel $P(\cdot \mid s_t, a_t)$ and the action component is set deterministically to $a_t$.
		Hence the augmented transition kernel satisfies
		\[
		\tilde{P}\!\left(\tilde{s}_{t+1} \mid \tilde{s}_t, a_t\right)
		= P\!\left(s_{t+1} \mid s_t, a_t\right).
		\]
		The right-hand side depends only on $(s_t, a_t)$, and $s_t$ is a component of $\tilde{s}_t$, so the
		transition probability is a function of $(\tilde{s}_t, a_t)$ alone.
		The Markov property therefore holds on the augmented state space $\tilde{\mathcal{S}} = \mathcal{S}\times\mathcal{A}$.
		
		\medskip
		\noindent\textbf{(ii) Policy validity.}
		The DD-SRad mapping \eqref{eq:ddsrad} takes $a_{\mathrm{prev}} = a_{t-1}$ as input, which is the second component of the augmented state $\tilde{s}_t = (s_t, a_{t-1})$.
		By Theorem~\ref{thm:constraint}, the output $a_t = \phi(u_t, a_{t-1})$ satisfies
		$|a_t^i - a_{t-1}^i| \leq \delta^i$ and $a_t^i \in [a_{\min}^i, a_{\max}^i]$ for all $i$ with probability~1.
		Therefore the stochastic policy $\pi : \tilde{\mathcal{S}} \to \Delta(\mathcal{A})$ is well-defined and maps every augmented state to a distribution over $\mathcal{A}$ satisfying constraint~\eqref{eq:rate_constraint}.
		
		\medskip
		\noindent\textbf{(iii) Existence of the optimal value function.}
		The augmented MDP $\tilde{\mathcal{M}}$ inherits all standard regularity conditions from the original MDP: the reward $r$ is bounded, the augmented action space $\mathcal{A}$ is compact (unchanged), and the augmented transition kernel $\tilde{P}$ is Lipschitz continuous in $(\tilde{s}, a)$ because $P$ is Lipschitz in $(s,a)$ and the action component of $\tilde{s}_{t+1}$ is set deterministically.
		Under these conditions, the augmented Bellman operator
		\[
		(\mathcal{T} Q)(\tilde{s}, a)
		\triangleq r(s,a) + \gamma\,\mathbb{E}_{\tilde{s}' \sim \tilde{P}(\cdot \mid \tilde{s},a)}\!\left[\max_{a'} Q(\tilde{s}',a')\right]
		\]
		is a $\gamma$-contraction in $\|\cdot\|_\infty$ on the space of bounded functions~\cite{bertsekas1996stochastic}.
		By Banach's fixed-point theorem, $\mathcal{T}$ admits a unique fixed point $Q^*$, which is the optimal action-value function of $\tilde{\mathcal{M}}$.
	\end{proof}

	\subsection{Proof of Theorem~\ref{thm:constraint} (Heterogeneous Rate Constraint Satisfaction)}
	\label{app:proof:constraint}

	\begin{proof}
		Fix any dimension $i \in \{1,\ldots,d\}$ and any $u^i \in \mathbb{R}$.
		Let $\Delta a^i \triangleq a^i - a_{\mathrm{prev}}^i = R_{\mathrm{eff}}^i(u^i, a_{\mathrm{prev}}^i)\cdot f(u^i)$.
		
		Three exhaustive cases are considered according to the sign of $u^i$.
		
		\medskip
		\noindent\textbf{Case 1: $u^i > 0$.}
		Then $R_{\mathrm{eff}}^i = \min\!\left(\delta^i,\, a_{\max}^i - a_{\mathrm{prev}}^i\right) \geq 0$ and $f(u^i) \in (0,1)$.
		\begin{enumerate}
			\item \textit{Rate constraint.}
			$|\Delta a^i| = R_{\mathrm{eff}}^i \cdot f(u^i) < R_{\mathrm{eff}}^i \leq \delta^i$,
			where the strict inequality uses $f(u^i) < 1$.
			(When $a_{\mathrm{prev}}^i = a_{\max}^i$ one has $R_{\mathrm{eff}}^i = 0$, giving
			$\Delta a^i = 0$, which also satisfies $|\Delta a^i| \leq \delta^i$.)
			\item \textit{Position constraint.}
			Since $\Delta a^i = R_{\mathrm{eff}}^i f(u^i) \geq 0$,
			it follows that $a^i = a_{\mathrm{prev}}^i + \Delta a^i \geq a_{\mathrm{prev}}^i \geq a_{\min}^i$.
			Moreover, $\Delta a^i < R_{\mathrm{eff}}^i \leq a_{\max}^i - a_{\mathrm{prev}}^i$,
			so $a^i < a_{\max}^i$.
			Hence $a^i \in [a_{\min}^i, a_{\max}^i]$.
		\end{enumerate}
		If additionally $a_{\mathrm{prev}}^i \in (a_{\min}^i, a_{\max}^i)$, then $R_{\mathrm{eff}}^i > 0$,
		so $|\Delta a^i| < \delta^i$ and $a^i \in (a_{\min}^i, a_{\max}^i)$ hold strictly.
		
		\medskip
		\noindent\textbf{Case 2: $u^i < 0$.}
		Then $R_{\mathrm{eff}}^i = \min\!\left(\delta^i,\, a_{\mathrm{prev}}^i - a_{\min}^i\right) \geq 0$ and $f(u^i) \in (-1,0)$.
		A symmetric argument to Case~1 (with the roles of $a_{\min}^i$ and $a_{\max}^i$ exchanged) yields
		$|\Delta a^i| \leq \delta^i$ and $a^i \in [a_{\min}^i, a_{\max}^i]$,
		with strict inclusions when $a_{\mathrm{prev}}^i \in (a_{\min}^i, a_{\max}^i)$.
		
		\medskip
		\noindent\textbf{Case 3: $u^i = 0$.}
		Then $\Delta a^i = R_{\mathrm{eff}}^i \cdot f(0) = 0$, so $|\Delta a^i| = 0 \leq \delta^i$ and $a^i = a_{\mathrm{prev}}^i \in \mathcal{A}^i$.
		
		\medskip
		Applying the above to every dimension independently and observing that the latent action $u$ is drawn from a continuous distribution (hence $u^i = 0$ occurs with probability~0), the constraints $|a^i - a_{\mathrm{prev}}^i| \leq \delta^i$ and $a^i \in [a_{\min}^i, a_{\max}^i]$ hold for all $i$ with probability~1.
	\end{proof}

	\subsection{Proof of Proposition~\ref{prop:gradient} (Jacobian Structure and Gradient Norm Bound)}
	\label{app:proof:gradient}
	
	\begin{proof}
		\noindent\textbf{(i) Diagonal structure and positivity.}
		Fix dimension $i$ and suppose $u^i \neq 0$.
		On the open half-lines $\{u^i > 0\}$ and $\{u^i < 0\}$, the effective radius $R_{\mathrm{eff}}^i(u^i, a_{\mathrm{prev}}^i)$ is piecewise constant (it equals $\min(\delta^i, a_{\max}^i - a_{\mathrm{prev}}^i)$ for $u^i > 0$ and $\min(\delta^i, a_{\mathrm{prev}}^i - a_{\min}^i)$ for $u^i < 0$).
		Differentiating $a^i = a_{\mathrm{prev}}^i + R_{\mathrm{eff}}^i \cdot f(u^i)$ with respect to $u^i$ on each open half-line via the chain rule gives
		\[
		\frac{\partial a^i}{\partial u^i}
		= R_{\mathrm{eff}}^i \cdot f'(u^i)
		= \frac{R_{\mathrm{eff}}^i}{\left(1+(u^i)^2\right)^{3/2}}.
		\]
		Since $a_j$ depends only on $u^j$ for $j \neq i$, all off-diagonal entries of the Jacobian $J(u) = \partial a/\partial u$ are zero.
		Hence $J(u)$ is diagonal.
		When $a_{\mathrm{prev}}^i \in (a_{\min}^i, a_{\max}^i)$, $R_{\mathrm{eff}}^i > 0$, so $\partial a^i/\partial u^i > 0$.
		
		\medskip
		\noindent\textbf{(ii) Spectral norm and condition number.}
		Since $J(u)$ is diagonal, its spectral norm equals the largest absolute diagonal entry:
		\[
		\|J(u)\|_2
		= \max_i \frac{R_{\mathrm{eff}}^i}{(1+(u^i)^2)^{3/2}}
		\leq \max_i R_{\mathrm{eff}}^i
		\leq \max_i \delta^i,
		\]
		where the bound $(1+(u^i)^2)^{3/2} \geq 1$ and $R_{\mathrm{eff}}^i \leq \delta^i$ are applied.
		
		For the condition number, note that the $i$-th diagonal entry $J_{ii}(u) = R_{\mathrm{eff}}^i/(1+(u^i)^2)^{3/2}$ is maximized at $u^i = 0$ with value $R_{\mathrm{eff}}^i|_{u^i=0} = \delta^i$.
		At $u = \mathbf{0}$ with $R_{\mathrm{eff}}^i = \delta^i$ for all $i$ (i.e., away from position boundaries),
		\[
		\kappa(J(\mathbf{0}))
		= \frac{\max_i \delta^i}{\min_i \delta^i} = \kappa.
		\]
		For general $u \neq \mathbf{0}$, the factor $(1+(u^i)^2)^{3/2}$ varies across dimensions and may either increase or decrease the ratio of diagonal entries; it does not exceed the ratio at $u=\mathbf{0}$, so $\kappa$ is the worst-case upper bound on the condition number.
		
		\medskip
		\noindent\textbf{(iii) Gradient norm bound.}
		For the TD3 backbone, the actor loss is $\mathcal{L}_{\mathrm{actor}}(u) = -Q(\tilde{s}, \phi(u)) + \lambda_{\mathrm{base}}\|u\|_2^2$.
		Differentiating with respect to $u$ and applying the triangle inequality:
		\[
		\|\nabla_u \mathcal{L}_{\mathrm{actor}}\|_2
		= \left\|\nabla_a Q(\tilde{s},a)\big|_{a=\phi(u)} \cdot J(u)
		+ 2\lambda_{\mathrm{base}} u\right\|_2
		\leq \|J(u)\|_2 \cdot \|\nabla_a Q(\tilde{s},a)\|_2 + 2\lambda_{\mathrm{base}}\|u\|_2.
		\]
		Substituting the spectral norm bound from~(ii) yields inequality~\eqref{eq:grad_bound}.
		For the SAC backbone, the entropy term $\alpha\log\pi_\theta(a|\tilde{s})$ contributes an additional gradient term $\alpha\nabla_u \log\pi_\theta(a|\tilde{s})$, which enters linearly and is accounted for as stated.
	\end{proof}
	
	\subsection{Proof of Proposition (Lipschitz Continuity of the Mapping)}
	\label{app:proof:lipschitz}
	\label{prop:lipschitz}
	
	\begin{proof}
		Fix $a_{\mathrm{prev}} \in \mathcal{A}$ and $u, \hat{u} \in \mathbb{R}^d$.
		For each dimension $i$, define $g^i(v) \triangleq R_{\mathrm{eff}}^i(v, a_{\mathrm{prev}}^i) \cdot f(v)$
		so that $a^i - \hat{a}^i = g^i(u^i) - g^i(\hat{u}^i)$.
		Denote $R^{+} \triangleq \min(\delta^i, a_{\max}^i - a_{\mathrm{prev}}^i)$ and $R^{-} \triangleq \min(\delta^i, a_{\mathrm{prev}}^i - a_{\min}^i)$; both satisfy $R^{\pm} \leq \delta^i$.
		
		\medskip
		\noindent\textbf{Case 1: $u^i, \hat{u}^i > 0$.}
		Both map under $R^+$, so
		\[
		|g^i(u^i) - g^i(\hat{u}^i)|
		= R^+ |f(u^i) - f(\hat{u}^i)|.
		\]
		By the Mean Value Theorem and $f'(t) = (1+t^2)^{-3/2} \leq 1$:
		\[
		|f(u^i) - f(\hat{u}^i)| \leq |u^i - \hat{u}^i|.
		\]
		Hence $|g^i(u^i) - g^i(\hat{u}^i)| \leq R^+ |u^i - \hat{u}^i| \leq \delta^i |u^i - \hat{u}^i|$.
		
		\medskip
		\noindent\textbf{Case 2: $u^i, \hat{u}^i < 0$.}
		Both map under $R^-$; the identical argument gives $|g^i(u^i) - g^i(\hat{u}^i)| \leq \delta^i |u^i - \hat{u}^i|$.
		
		\medskip
		\noindent\textbf{Case 3: $u^i \geq 0$ and $\hat{u}^i \leq 0$ (or vice versa).}
		Assume without loss of generality $u^i \geq 0 \geq \hat{u}^i$.
		Then $g^i(u^i) = R^+ f(u^i) \geq 0$ and $g^i(\hat{u}^i) = R^- f(\hat{u}^i) \leq 0$, so
		\begin{align*}
			|g^i(u^i) - g^i(\hat{u}^i)|
			&= R^+ f(u^i) + R^- |f(\hat{u}^i)| \\
			&\leq \delta^i \!\left[f(u^i) + |f(\hat{u}^i)|\right] \\
			&\leq \delta^i \!\left[u^i + |\hat{u}^i|\right]
			= \delta^i |u^i - \hat{u}^i|,
		\end{align*}
		where the second inequality uses $|f(t)| \leq |t|$, and the last equality holds because $u^i \geq 0 \geq \hat{u}^i$.
		
		\medskip
		Combining all three cases, $|a^i - \hat{a}^i| \leq \delta^i |u^i - \hat{u}^i|$ holds for every $i$.
		Squaring and summing over all dimensions:
		\[
		\|\phi(u) - \phi(\hat{u})\|_2^2
		= \sum_{i=1}^d |a^i - \hat{a}^i|^2
		\leq \sum_{i=1}^d (\delta^i)^2 (u^i - \hat{u}^i)^2
		\leq \|\boldsymbol{\delta}\|_2^2 \cdot \|u - \hat{u}\|_2^2.
		\]
		Taking the square root gives the stated Lipschitz bound.
		
		\medskip
		\noindent\textbf{Tightness.}
		Fix any $\varepsilon > 0$ and set $u^i = \varepsilon$, $\hat{u}^i = -\varepsilon$ for all $i$, with $a_{\mathrm{prev}}^i$ in the interior of $\mathcal{A}^i$ so that $R^{\pm} = \delta^i$.
		Then each dimension falls into Case~3, and the inequality chain above becomes an equality to leading order as $\varepsilon \to 0$ (since $f(\varepsilon) \to \varepsilon$ and $|f(-\varepsilon)| \to \varepsilon$).
		Hence $\|\boldsymbol{\delta}\|_2$ is the best possible Lipschitz constant.
	\end{proof}
	
	\subsection{Proof of Theorem~\ref{thm:exploration} (Tight $\ell_\infty$-Coverage of the Reachable Set)}
	\label{app:proof:exploration}
	
	\begin{proof}
		Fix $a_{\mathrm{prev}} \in \mathcal{A}$.
		For each dimension $i$, define the direction-dependent effective radii
		\[
		R^{+,i} \triangleq \min\!\left(\delta^i,\; a_{\max}^i - a_{\mathrm{prev}}^i\right), \qquad
		R^{-,i} \triangleq \min\!\left(\delta^i,\; a_{\mathrm{prev}}^i - a_{\min}^i\right),
		\]
		corresponding to $R_{\mathrm{eff}}^i$ when $u^i > 0$ and $u^i < 0$, respectively.
		
		\medskip
		\noindent\textbf{Step 1: Characterization of $\mathcal{R}_{\mathrm{DD}}$.}
		
		\textit{Inclusion (image is contained in the open box).}
		By Theorem~\ref{thm:constraint}, for every $u \in \mathbb{R}^d$ and every $i$:
		\begin{itemize}
			\item If $u^i > 0$: $\Delta a^i = R^{+,i} f(u^i) \in (0, R^{+,i})$, so $a^i \in (a_{\mathrm{prev}}^i,\; a_{\mathrm{prev}}^i + R^{+,i})$.
			\item If $u^i < 0$: $\Delta a^i = R^{-,i} f(u^i) \in (-R^{-,i}, 0)$, so $a^i \in (a_{\mathrm{prev}}^i - R^{-,i},\; a_{\mathrm{prev}}^i)$.
			\item If $u^i = 0$: $a^i = a_{\mathrm{prev}}^i$.
		\end{itemize}
		In all cases, $a^i \in (a_{\mathrm{prev}}^i - R^{-,i},\; a_{\mathrm{prev}}^i + R^{+,i})$.
		
		\textit{Surjectivity (every interior point is achieved).}
		Fix any target $a^* = (a^{*1}, \ldots, a^{*d})$ satisfying $a^{*i} \in (a_{\mathrm{prev}}^i - R^{-,i},\; a_{\mathrm{prev}}^i + R^{+,i})$ for all $i$.
		Set $\Delta^{*i} \triangleq a^{*i} - a_{\mathrm{prev}}^i$ and consider two subcases.
		
		If $\Delta^{*i} > 0$: the goal is to find $u^{*i} > 0$ satisfying $R^{+,i} f(u^{*i}) = \Delta^{*i}$.
		Since $\Delta^{*i} < R^{+,i}$, the ratio $\Delta^{*i}/R^{+,i} \in (0,1)$.
		The equation $f(u) = c$ for $c \in (0,1)$ has the unique positive solution $u = c/\sqrt{1-c^2}$, giving
		\[
		u^{*i} = \frac{\Delta^{*i}}{\sqrt{(R^{+,i})^2 - (\Delta^{*i})^2}} > 0.
		\]
		For this $u^{*i}$, $R_{\mathrm{eff}}^i(u^{*i}, a_{\mathrm{prev}}^i) = R^{+,i}$ by definition, so $\phi(u^*)^i = a_{\mathrm{prev}}^i + R^{+,i} f(u^{*i}) = a^{*i}$ as required.
		
		If $\Delta^{*i} < 0$: a symmetric construction with $R^{-,i}$ yields the unique negative preimage $u^{*i} = \Delta^{*i}/\sqrt{(R^{-,i})^2 - (\Delta^{*i})^2} < 0$.
		
		If $\Delta^{*i} = 0$: $u^{*i} = 0$ gives $a^{*i} = a_{\mathrm{prev}}^i$.
		
		Since dimensions are decoupled, setting $u^* = (u^{*1}, \ldots, u^{*d})$ achieves $\phi(u^*) = a^*$.
		
		\textit{Conclusion.} The reachable set (image) is
		\[
		\mathcal{R}_{\mathrm{DD}}
		= \prod_{i=1}^d \!\left(a_{\mathrm{prev}}^i - R^{-,i},\; a_{\mathrm{prev}}^i + R^{+,i}\right),
		\]
		and its closure is the closed product interval $\overline{\mathcal{R}_{\mathrm{DD}}} = \prod_{i=1}^d [a_{\mathrm{prev}}^i - R^{-,i},\; a_{\mathrm{prev}}^i + R^{+,i}]$.
		
		\medskip
		\noindent\textbf{Step 2: Agreement with $\mathcal{F}$.}
		Observe that
		\[
		a_{\mathrm{prev}}^i + R^{+,i}
		= a_{\mathrm{prev}}^i + \min\!\left(\delta^i, a_{\max}^i - a_{\mathrm{prev}}^i\right)
		= \min\!\left(a_{\mathrm{prev}}^i + \delta^i,\; a_{\max}^i\right),
		\]
		\[
		a_{\mathrm{prev}}^i - R^{-,i}
		= a_{\mathrm{prev}}^i - \min\!\left(\delta^i, a_{\mathrm{prev}}^i - a_{\min}^i\right)
		= \max\!\left(a_{\mathrm{prev}}^i - \delta^i,\; a_{\min}^i\right).
		\]
		Therefore, for each dimension $i$:
		\[
		\left[a_{\mathrm{prev}}^i - R^{-,i},\; a_{\mathrm{prev}}^i + R^{+,i}\right]
		= \left[a_{\mathrm{prev}}^i - \delta^i,\; a_{\mathrm{prev}}^i + \delta^i\right] \cap \left[a_{\min}^i,\; a_{\max}^i\right].
		\]
		Taking the Cartesian product over all dimensions:
		\[
		\overline{\mathcal{R}_{\mathrm{DD}}}
		= \prod_{i=1}^d \!\left(\left[a_{\mathrm{prev}}^i - \delta^i, a_{\mathrm{prev}}^i + \delta^i\right] \cap [a_{\min}^i, a_{\max}^i]\right)
		= \mathcal{F},
		\]
		where the last equality uses the product structure of both the rate-constraint set and $\mathcal{A}$.
		This establishes \eqref{eq:reach_tight}.
		
		\medskip
		\noindent\textbf{Step 3: Volume ratio.}
		When $a_{\mathrm{prev}}^i$ is sufficiently far from the position boundaries so that $R^{+,i} = R^{-,i} = \delta^i$ for all $i$, the Lebesgue measure of $\overline{\mathcal{R}_{\mathrm{DD}}}$ is
		\[
		V_{\mathrm{DD}} = \prod_{i=1}^d \!\left(R^{+,i} + R^{-,i}\right) = \prod_{i=1}^d 2\delta^i.
		\]
		The $d$-dimensional $\ell_2$ ball with radius $R = \min_i \delta^i$ has volume
		\[
		V_{\mathrm{SRad}} = \frac{\pi^{d/2}}{\Gamma(d/2+1)}\left(\min_i \delta^i\right)^d.
		\]
		Dividing:
		\[
		\frac{V_{\mathrm{DD}}}{V_{\mathrm{SRad}}}
		= \frac{\prod_{i=1}^d 2\delta^i}{\dfrac{\pi^{d/2}}{\Gamma(d/2+1)}\left(\min_i \delta^i\right)^d}
		= \frac{2^d\,\Gamma(d/2+1)}{\pi^{d/2}} \cdot \frac{\prod_i \delta^i}{\left(\min_i \delta^i\right)^d},
		\]
		which establishes \eqref{eq:volume_ratio}. \qquad
	\end{proof}
	
	\section{Structure and Hyperparameters}
	\label{app:hyperparams}
	
	The complete DD-SRad algorithm is presented in Algorithm~\ref{alg:ddsrad}.
	
	\begin{algorithm}[htbp]
		\captionsetup{font=scriptsize, labelfont=scriptsize}
		\caption{DD-SRad (Off-Policy Backbone)}
		\label{alg:ddsrad}
		\KwIn{Per-dimension rate limits $\boldsymbol{\delta}$, latent action regularization coefficient $\lambda_{\text{base}}$, position bounds $[a_{\min}, a_{\max}]$, off-policy backbone $\mathcal{B}$}
		Initialize augmented replay buffer $\mathcal{D}$, policy network $\pi_\theta$, action-value function $Q_\phi$, and target network $Q_{\phi'}$\;
		\For{each training step}{
			\textbf{Environment interaction:} Read $s_t$ from the environment, retrieve previous action $a_{t-1}$ from the augmented state, construct augmented state $\tilde{s}_t = (s_t, a_{t-1})$\;
			\textbf{Constrained action generation (DD-SRad core):} Sample latent action $u_t \sim \pi_\theta(\cdot \mid \tilde{s}_t)$\;
			Compute effective radius $R_{\text{eff}}^i(u_t^i, a_{t-1}^i)$ independently per dimension $i$ (Eq.~\ref{eq:reff}); apply per-dimension spherical squashing to obtain $a_t$ (Eq.~\ref{eq:ddsrad})\;
			Execute $a_t$, obtain $r_t$ and $s_{t+1}$, store $(\tilde{s}_t, a_t, r_t, \tilde{s}_{t+1})$ in $\mathcal{D}$\;
			\textbf{Q-function update (determined by backbone $\mathcal{B}$):} Sample a batch from $\mathcal{D}$; update $Q_\phi$ according to the Bellman target of backbone $\mathcal{B}$\;
			\textbf{Policy gradient update:} Minimize $\mathcal{L}_{\text{actor}}$ per Eq.~(\ref{eq:actor_loss}); backpropagate gradient through the DD-SRad Jacobian (Proposition~\ref{prop:gradient}) via the chain rule to $\theta$; under the SAC backbone, compute $\log\pi_\theta(a|\tilde{s})$ via change of variables per Eq.~(\ref{eq:sac_logprob})\;
			Update target network $\phi' \leftarrow \tau\phi + (1-\tau)\phi'$\;
		}
	\end{algorithm}
	
	All MuJoCo benchmark experiments were run on a single NVIDIA RTX A5000 GPU (24\,GB VRAM), with 1.5--2M training steps per environment-backbone-constraint configuration; single-run training time per environment is approximately 2--3 hours. IsaacLab simulation experiments were also conducted on a single A5000 using 4096 parallel simulation environments: H1 rough terrain was trained for 50k gradient steps with approximately 6 hours per run; G1 flat terrain was trained for 100k gradient steps with approximately 12 hours per run; 5 seeds were run serially, with complete H1 and G1 experiments taking approximately 30 hours and 60 hours respectively. Code is based on PyTorch 3.11 and IsaacLab 5.1.0, with Ubuntu 22.04 as the server operating system.
	
	Table~\ref{tab:hyperparams_full} lists the complete network and training hyperparameters used in all experiments. Unless otherwise specified, all experiments use the same network architecture and training configuration.
	
	\begin{table}[htbp]
		\centering
		\caption{Network architecture and training hyperparameters (common to all experiments)}
		\label{tab:hyperparams_full}
		\footnotesize
		\begin{tabular}{llc}
			\toprule
			Category & Parameter & Value \\
			\midrule
			\multirow{4}{*}{Network Architecture}
			& Actor hidden layers & [256, 256] \\
			& Critic hidden layers & [256, 256] \\
			& Activation function & ReLU \\
			& Output activation & Tanh (Actor mean), Identity (log\_std) \\
			\midrule
			\multirow{5}{*}{Optimizer}
			& Actor learning rate & $3 \times 10^{-4}$ \\
			& Critic learning rate & $3 \times 10^{-4}$ \\
			& Alpha learning rate & $3 \times 10^{-4}$ \\
			& Optimizer type & Adam \\
			& Gradient clipping & 1.0 \\
			\midrule
			\multirow{7}{*}{Training parameters}
			& Batch size & 256 \\
			& Replay buffer size & 1,000,000 \\
			& Learning starts & 10,000 \\
			& Train frequency & 1 \\
			& Gradient steps & 1 \\
			& Target update $\tau$ & 0.005 \\
			& Discount $\gamma$ & 0.99 \\
			\midrule
			\multirow{3}{*}{Other}
			& Evaluation frequency & Every 5,000 steps \\
			& Evaluation episodes & 5 \\
			& Random seeds & 12345, 22345, 32345, 42345, 52345 \\
			\bottomrule
		\end{tabular}
	\end{table}
	
	Table~\ref{tab:constraint_params} lists the constraint parameter settings for each MuJoCo experiment.
	
	\begin{table}[htbp]
		\centering
		\caption{Constraint parameter settings (MuJoCo experiments, by experiment group)}
		\label{tab:constraint_params}
		\footnotesize
		\begin{tabular}{lllc}
			\toprule
			Experiment & Environment & Constraint type & $\boldsymbol{\delta}$ \\
			\midrule
			\multirow{8}{*}{§\ref{sec:benchmark} (Exp. 1/2)}
			& \multirow{2}{*}{Ant-v5}
			& Wide homogeneous & $[1.0]^8$ \\
			&  & Tight heterogeneous & $[0.2^4, 0.5^4]$ \\
			& \multirow{2}{*}{Humanoid-v5}
			& Wide homogeneous & $[0.4]^{17}$ \\
			&  & Tight heterogeneous & $[0.8^6, 0.5^6, 0.2^5]$ \\
			& \multirow{2}{*}{Hopper-v5}
			& Wide homogeneous & $[1.0]^{3}$ \\
			&  & Tight heterogeneous & $[0.2, 0.5, 0.5]$ \\
			& \multirow{2}{*}{HalfCheetah-v5}
			& Wide homogeneous & $[1.0]^{6}$ \\
			&  & Tight heterogeneous & $[0.2^3, 0.5^3]$ \\
			\midrule
			\multirow{2}{*}{§\ref{sec:sensitivity} (Exp. 3)}
			& Ant-v5 (wide) & All methods & $[1.0]^8$ \\
			& Ant-v5 (tight) & All methods & $[0.2^4, 0.5^4]$ \\
			\bottomrule
		\end{tabular}
	\end{table}
	
	Table~\ref{tab:sim_constraint_params} lists the per-joint-group constraint parameters for each robot in the simulation experiments.
	
	\begin{table}[htbp]
		\centering
		\caption{Constraint parameters for simulation experiments (from official technical specifications; normalized to joint position increment units, $\delta_i = \omega_{\max}^i \times \Delta t \times \eta$, $\Delta t=0.02$\,s, $\eta$ is the safety factor)}
		\label{tab:sim_constraint_params}
		\footnotesize
		\begin{tabular}{llccc}
			\toprule
			Robot & Terrain & Joint group & No. joints & $\delta$ (normalized) \\
			\midrule
			\multirow{3}{*}{Unitree H1} & \multirow{3}{*}{Rough}
			& Hip yaw/roll/pitch (proximal) & 6 & 0.40 \\
			&& Knee (mid)                   & 2 & 0.25 \\
			&& Ankle (distal)               & 2 & 0.18 \\
			\midrule
			\multirow{4}{*}{Unitree G1} & \multirow{4}{*}{Flat}
			& Hip pitch/roll (proximal) & 4 & 0.40 \\
			&& Hip yaw/knee (mid)        & 4 & 0.35 \\
			&& Ankle pitch               & 2 & 0.25 \\
			&& Ankle roll/torso          & 3 & 0.18--0.25 \\
			\bottomrule
		\end{tabular}
	\end{table}
	
	Table~\ref{tab:baseline_params} lists the hyperparameters for each baseline method.
	
	\begin{table}[htbp]
		\centering
		\caption{Baseline method hyperparameters}
		\label{tab:baseline_params}
		\footnotesize
		\begin{tabular}{llc}
			\toprule
			Method & Parameter & Value \\
			\midrule
			DD-SRad
			& $\lambda_{\text{base}}$ & 0.005 \\
			\midrule
			\multirow{3}{*}{SAC-BoxPre+}
			& Projection & Per-dimension clip ($\ell_\infty$ hyperrectangle) \\
			& Critic training & Pre-projected actions + penalty \\
			& Penalty coefficient & Equivalent to DD-SRad $\lambda_{\text{base}}$ \\
			\midrule
			\multirow{2}{*}{SAC-Post(QP)}
			& Projection & Per-dimension clip post-processing \\
			& Critic training & Pre-projected actions (no penalty) \\
			\midrule
			\multirow{2}{*}{SRad-Strict}
			& Global radius $R$ & $\min_i \delta^i$ \\
			& Other parameters & Same as DD-SRad \\
			\midrule
			\multirow{3}{*}{SRad-QP}
			& Mapping & SRad ($R=\min_i\delta^i$) \\
			& Post-processing & QP projection to $\mathcal{F}_t$ \\
			& Other parameters & Same as DD-SRad \\
			\bottomrule
		\end{tabular}
	\end{table}
	
	Note: the network architecture, learning rates, and buffer sizes of all baseline methods are identical to DD-SRad to ensure fair comparison.
	
	\section{Extended Benchmark Performance}
	\label{app:bench_curves}
	
	Figure~\ref{fig:bench_reward} presents the mean return learning curves across four MuJoCo environments (Ant-v5, Humanoid-v5, HalfCheetah-v5, Hopper-v5) under tight heterogeneous constraints, for both SAC and TD3 backbones, covering all six methods. Shaded regions indicate mean $\pm$ standard deviation over 5 random seeds. Under wide homogeneous constraints all methods converge to similar levels; the curves under tight heterogeneous constraints are shown here to visualize the training stability differences discussed in \S\ref{sec:benchmark}, in particular the systematic collapse of SRad-Strict (high CV) versus the stable convergence of DD-SRad.
	
	\begin{figure*}[htbp]
		\centering
		\begin{subfigure}{0.24\linewidth}
			\includegraphics[width=\linewidth]{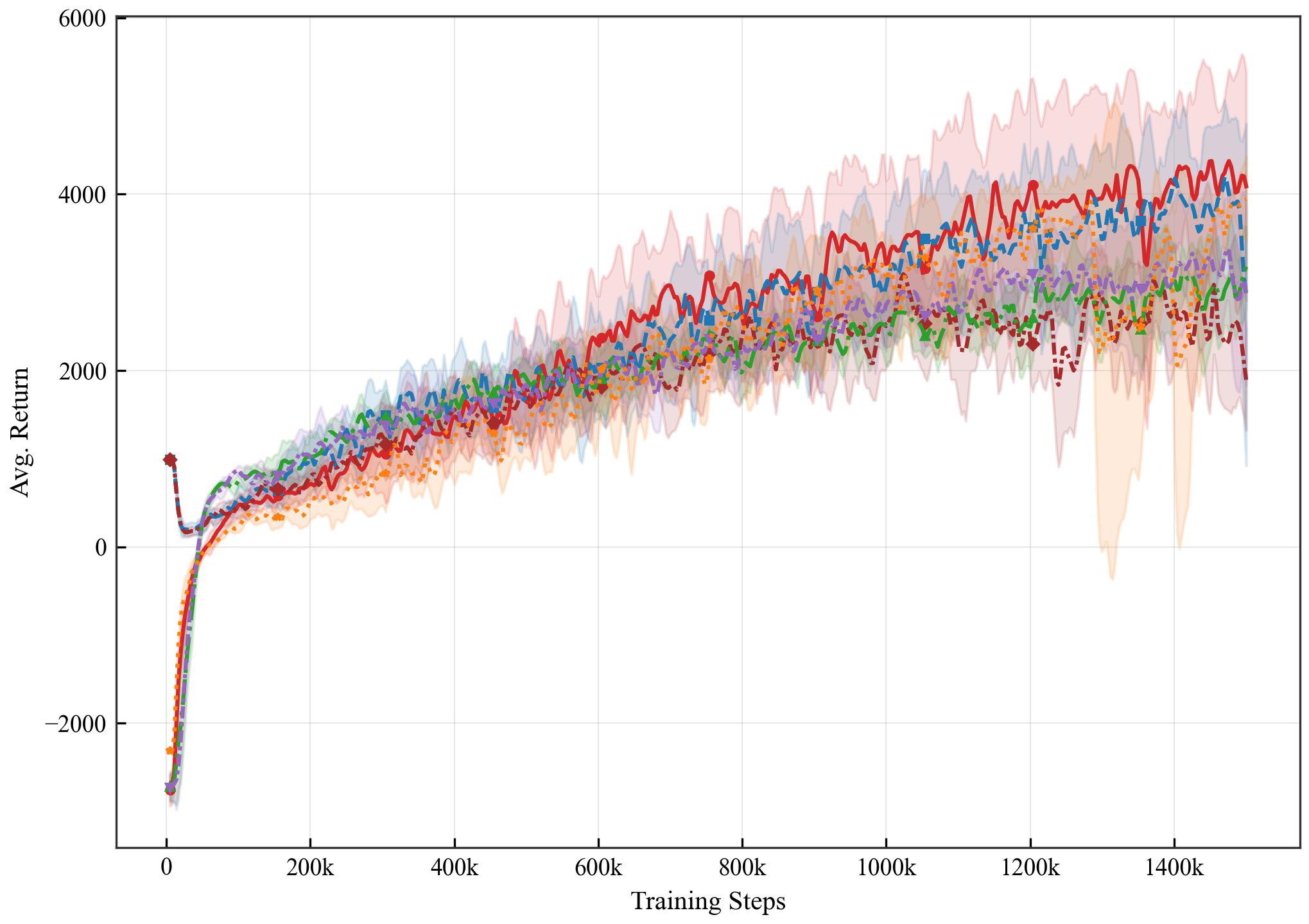}
			\caption{Ant-SAC-W}
		\end{subfigure}
		\hfill
		\begin{subfigure}{0.24\linewidth}
			\includegraphics[width=\linewidth]{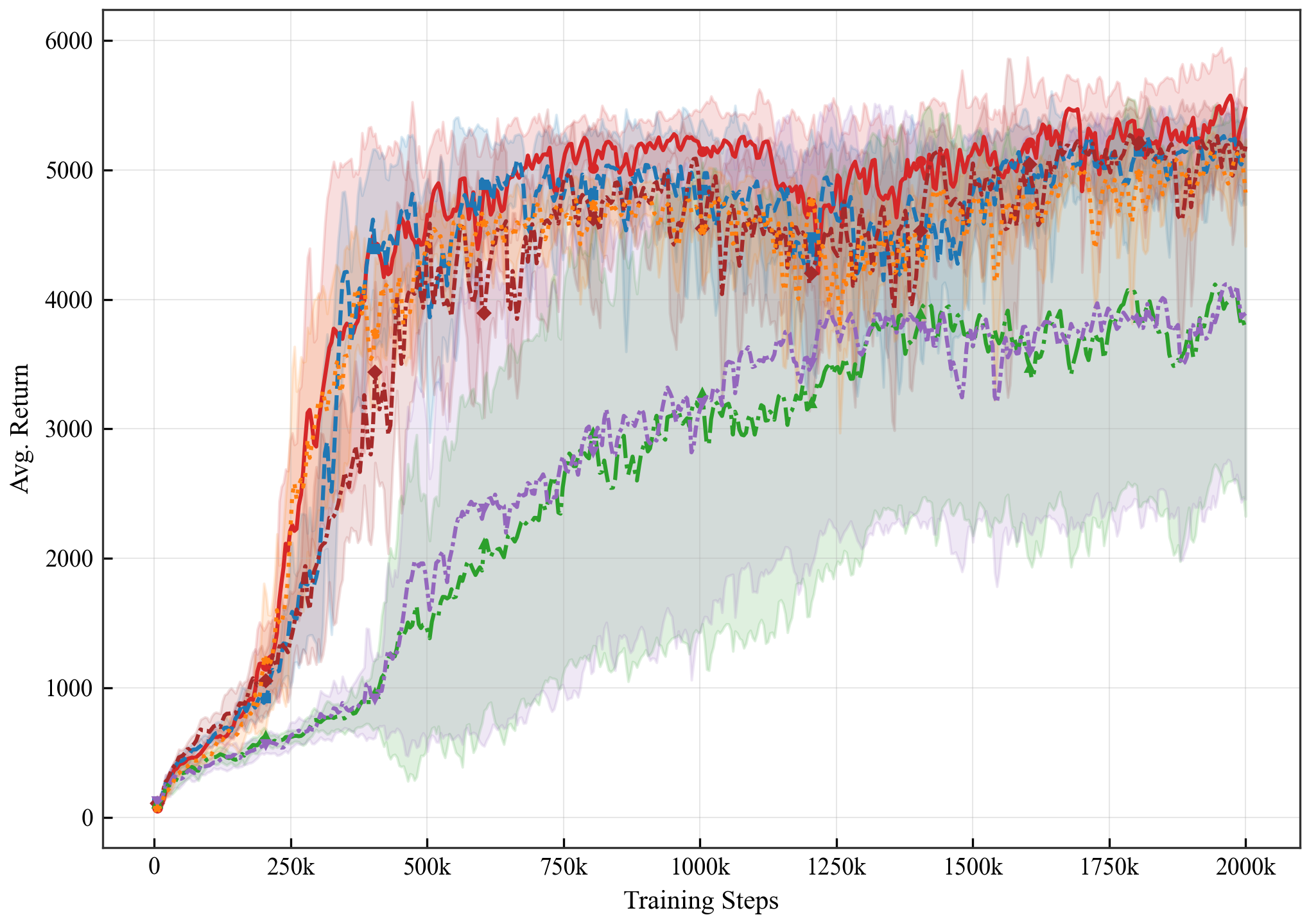}
			\caption{Humanoid-SAC-W}
		\end{subfigure}
		\vspace{0.5em}
		\begin{subfigure}{0.24\linewidth}
			\includegraphics[width=\linewidth]{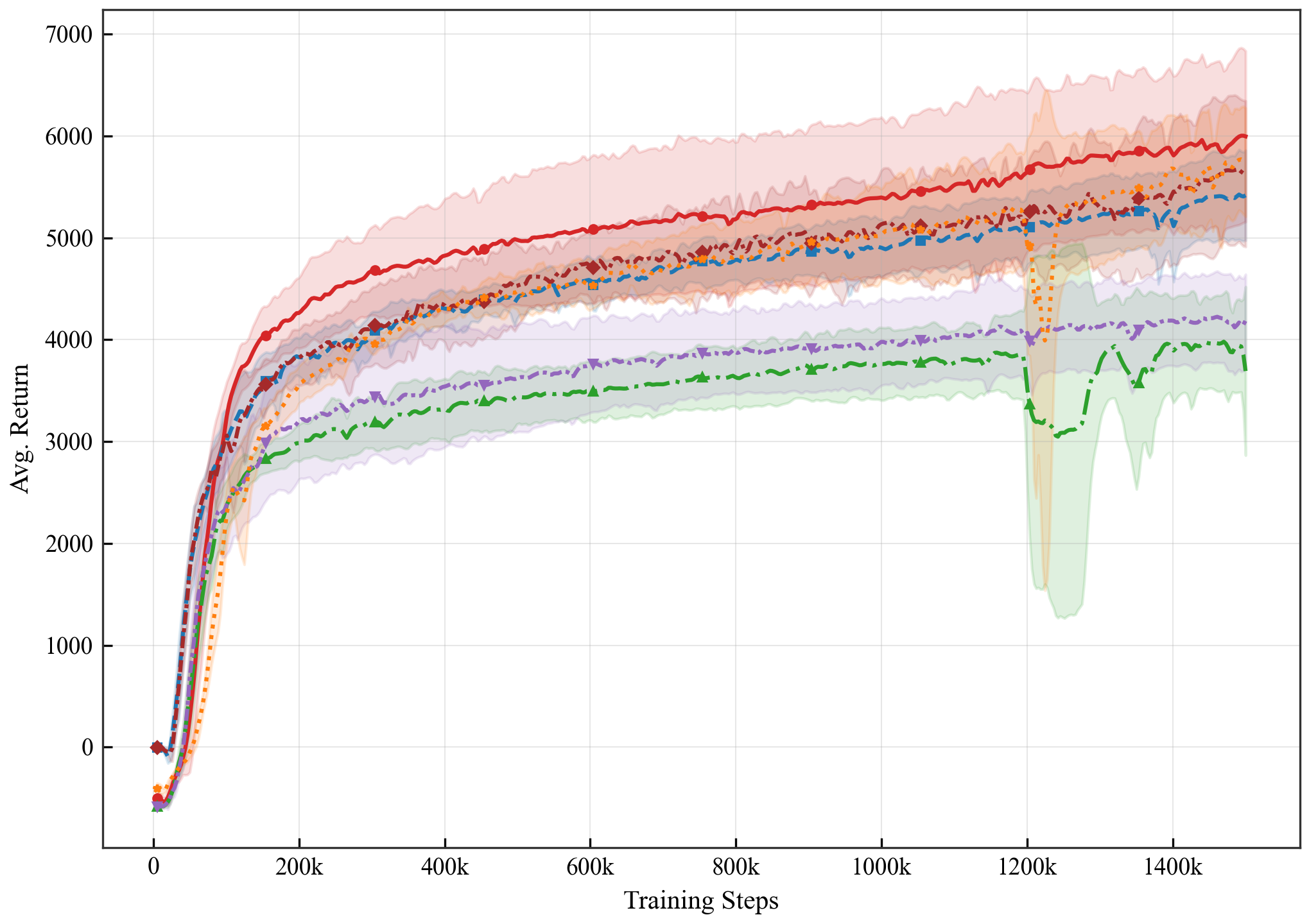}
			\caption{HalfCheetah-SAC-W}
		\end{subfigure}
		\hfill
		\begin{subfigure}{0.24\linewidth}
			\includegraphics[width=\linewidth]{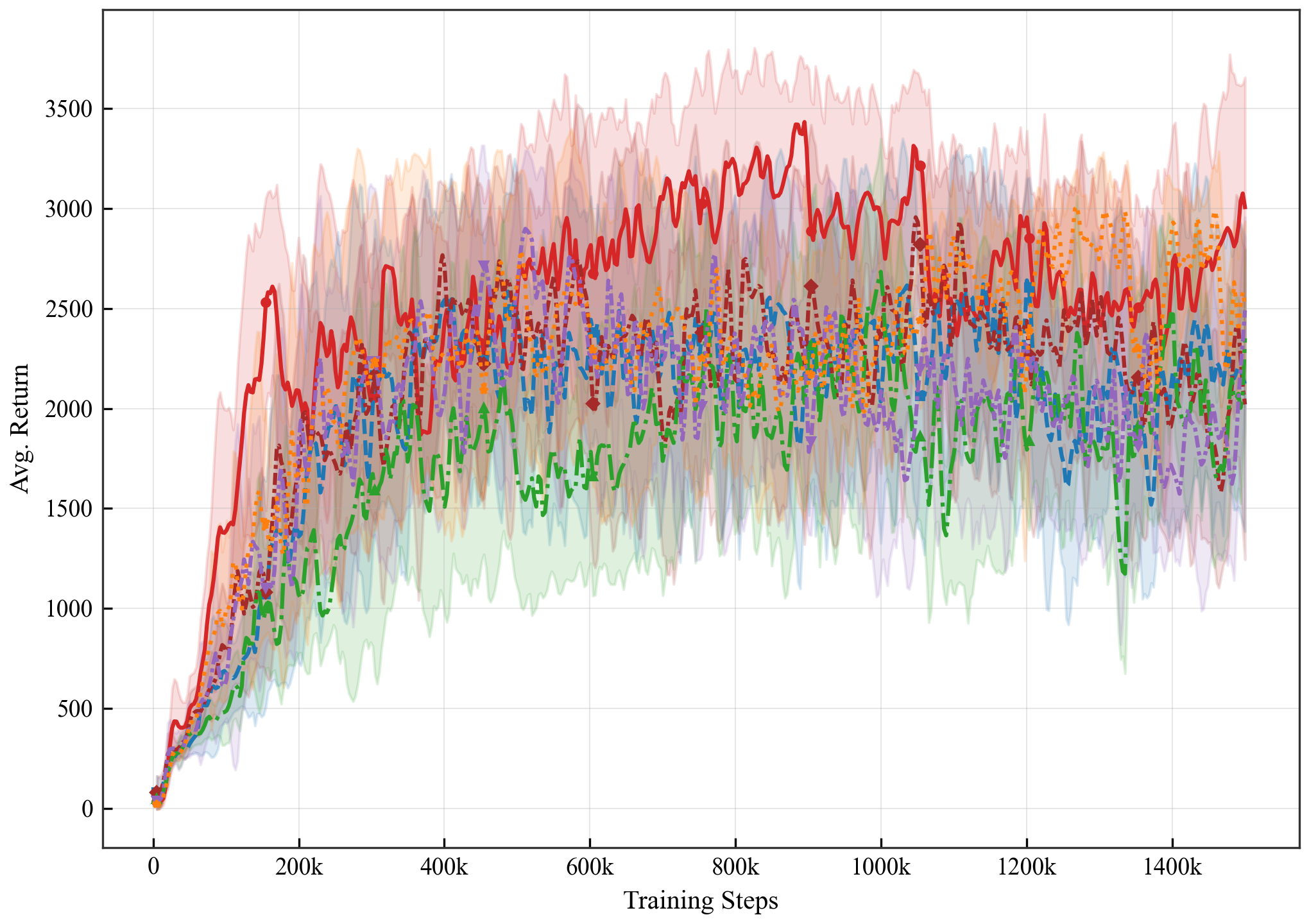}
			\caption{Hopper-SAC-W}
		\end{subfigure}
		
		\vspace{0.2em}
		\begin{subfigure}{0.24\linewidth}
			\includegraphics[width=\linewidth]{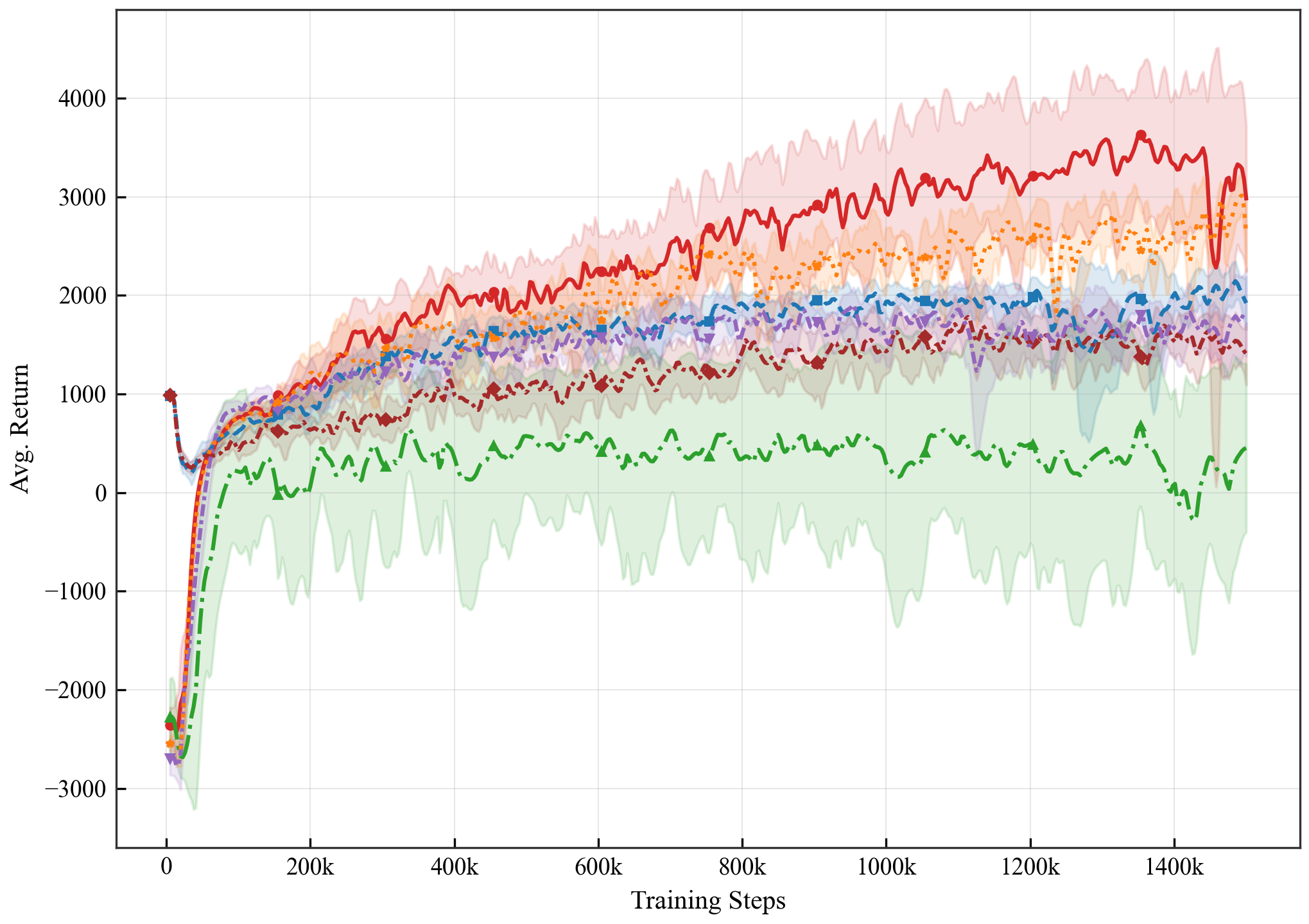}
			\caption{Ant-SAC-S}
		\end{subfigure}
		\hfill
		\begin{subfigure}{0.24\linewidth}
			\includegraphics[width=\linewidth]{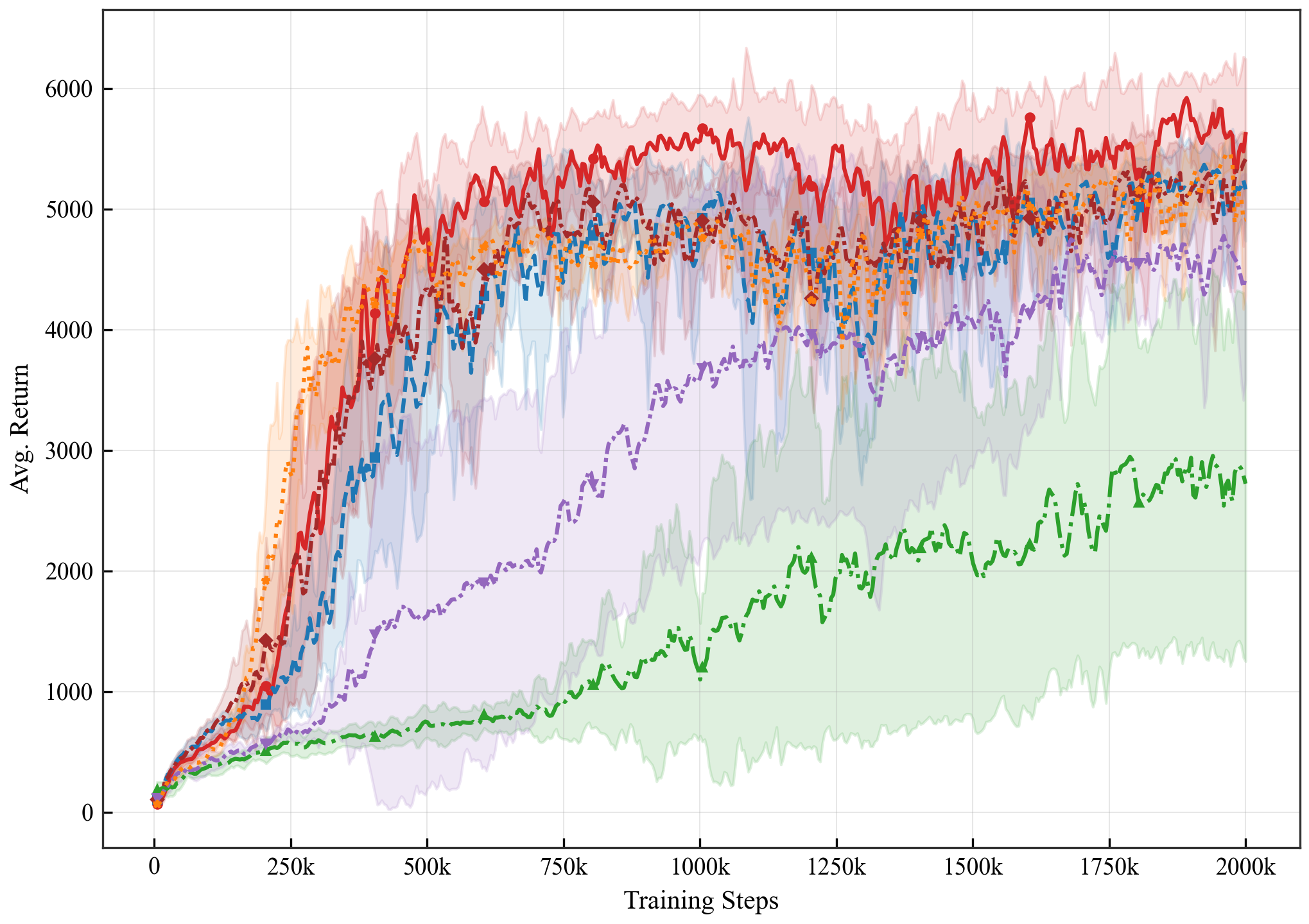}
			\caption{Humanoid-SAC-S}
		\end{subfigure}
		\vspace{0.5em}
		\begin{subfigure}{0.24\linewidth}
			\includegraphics[width=\linewidth]{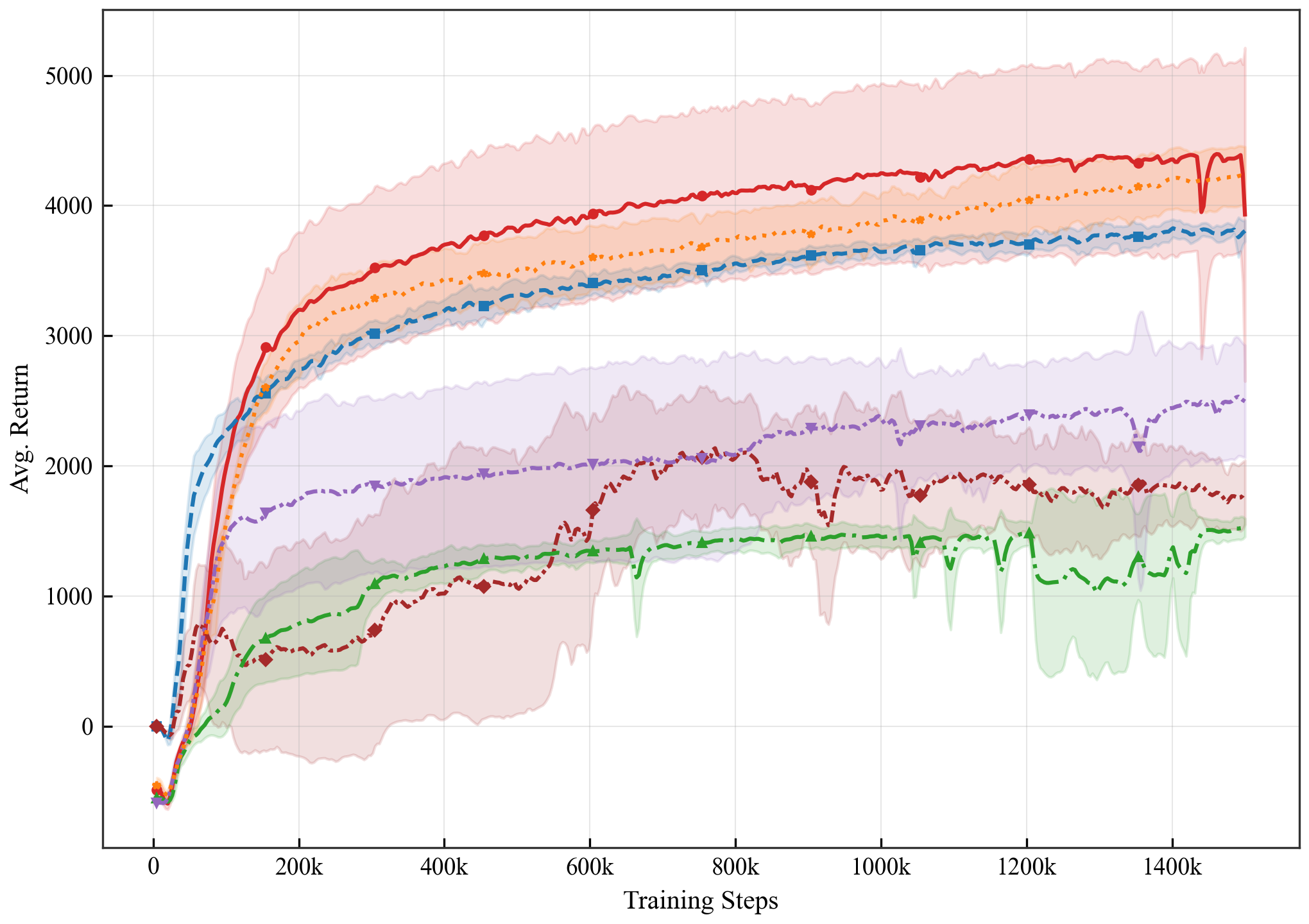}
			\caption{HalfCheetah-SAC-S}
		\end{subfigure}
		\hfill
		\begin{subfigure}{0.24\linewidth}
			\includegraphics[width=\linewidth]{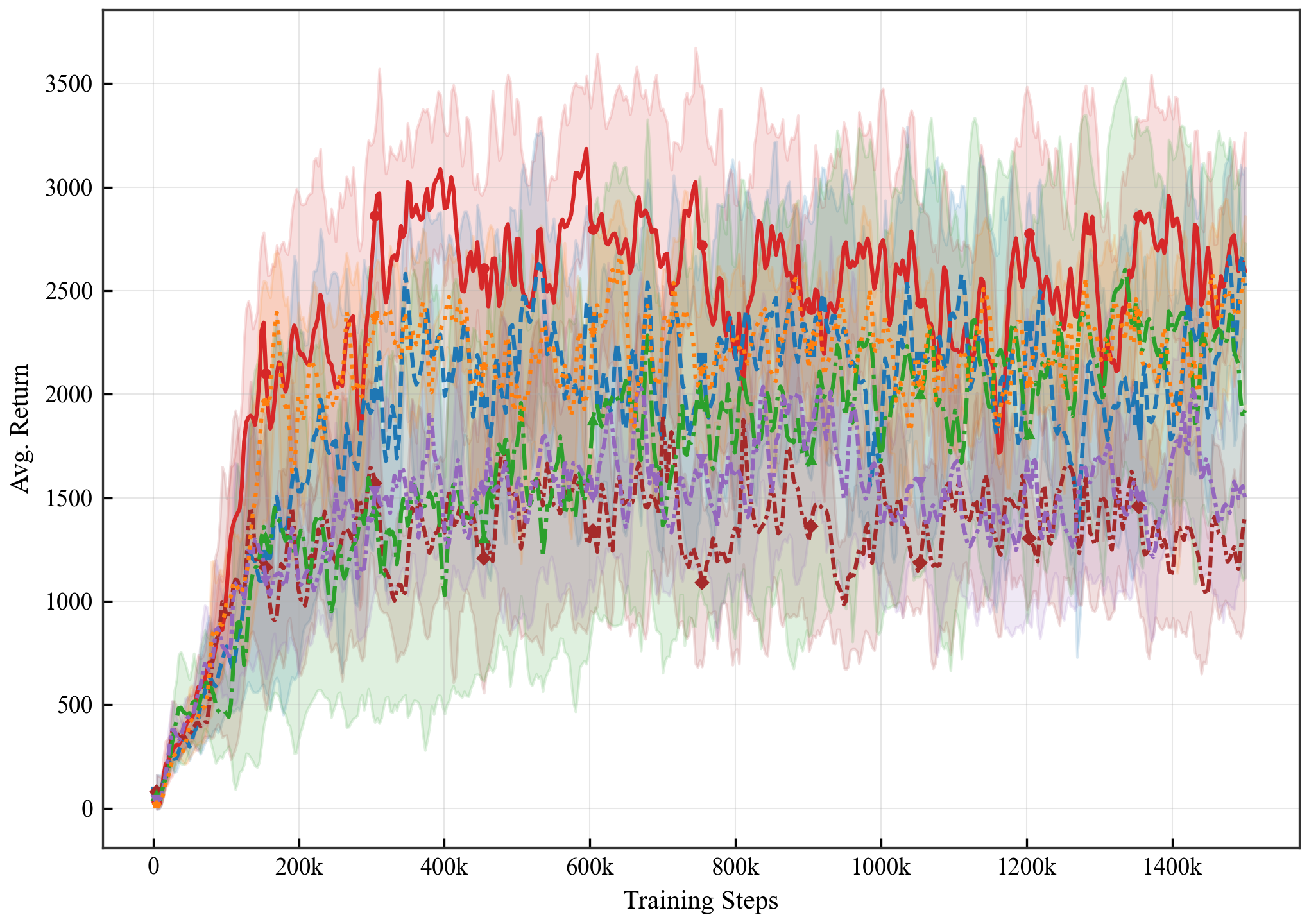}
			\caption{Hopper-SAC-S}
		\end{subfigure}
		
		\begin{subfigure}{0.24\linewidth}
			\includegraphics[width=\linewidth]{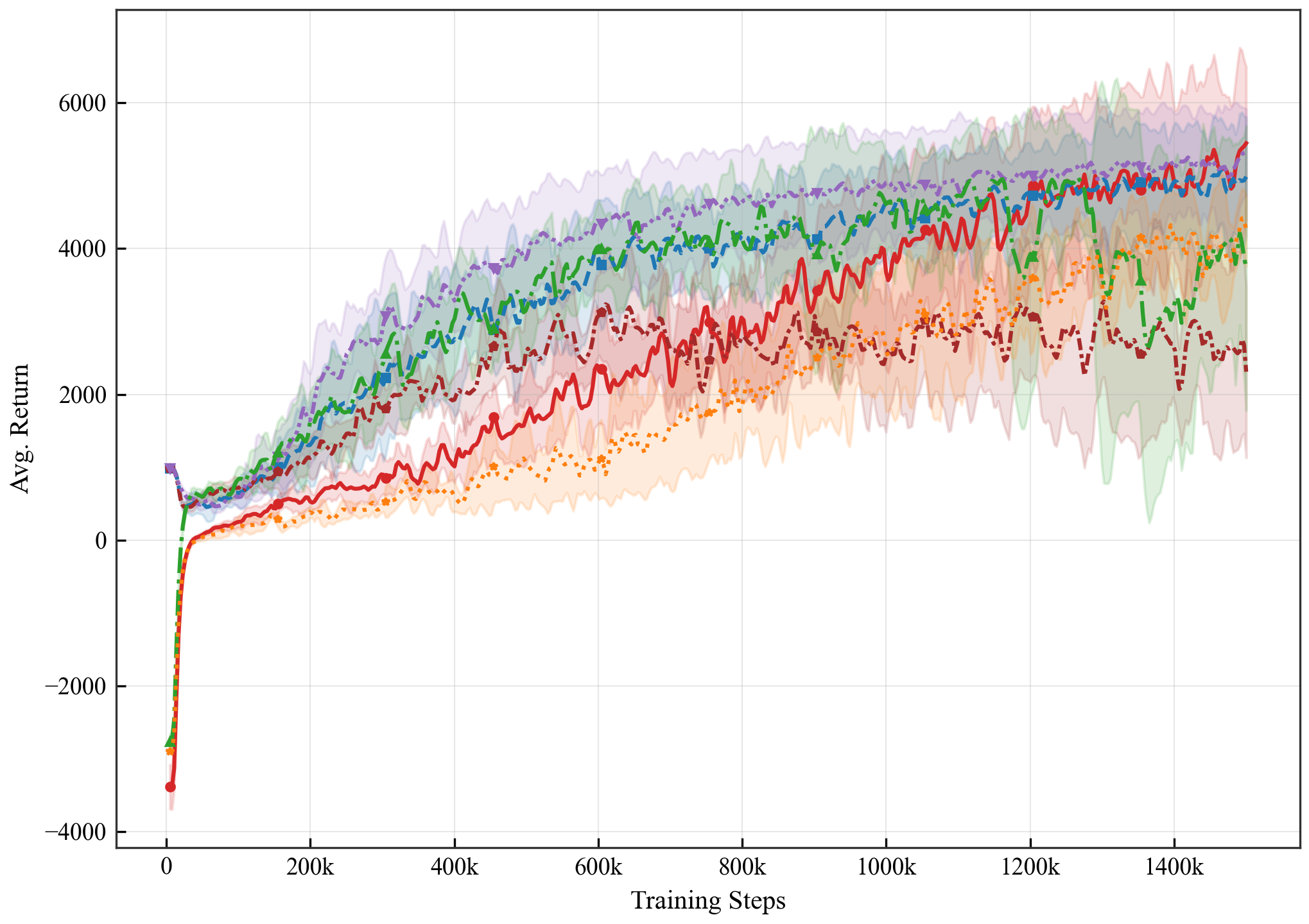}
			\caption{Ant-TD3-W}
		\end{subfigure}
		\hfill
		\begin{subfigure}{0.24\linewidth}
			\includegraphics[width=\linewidth]{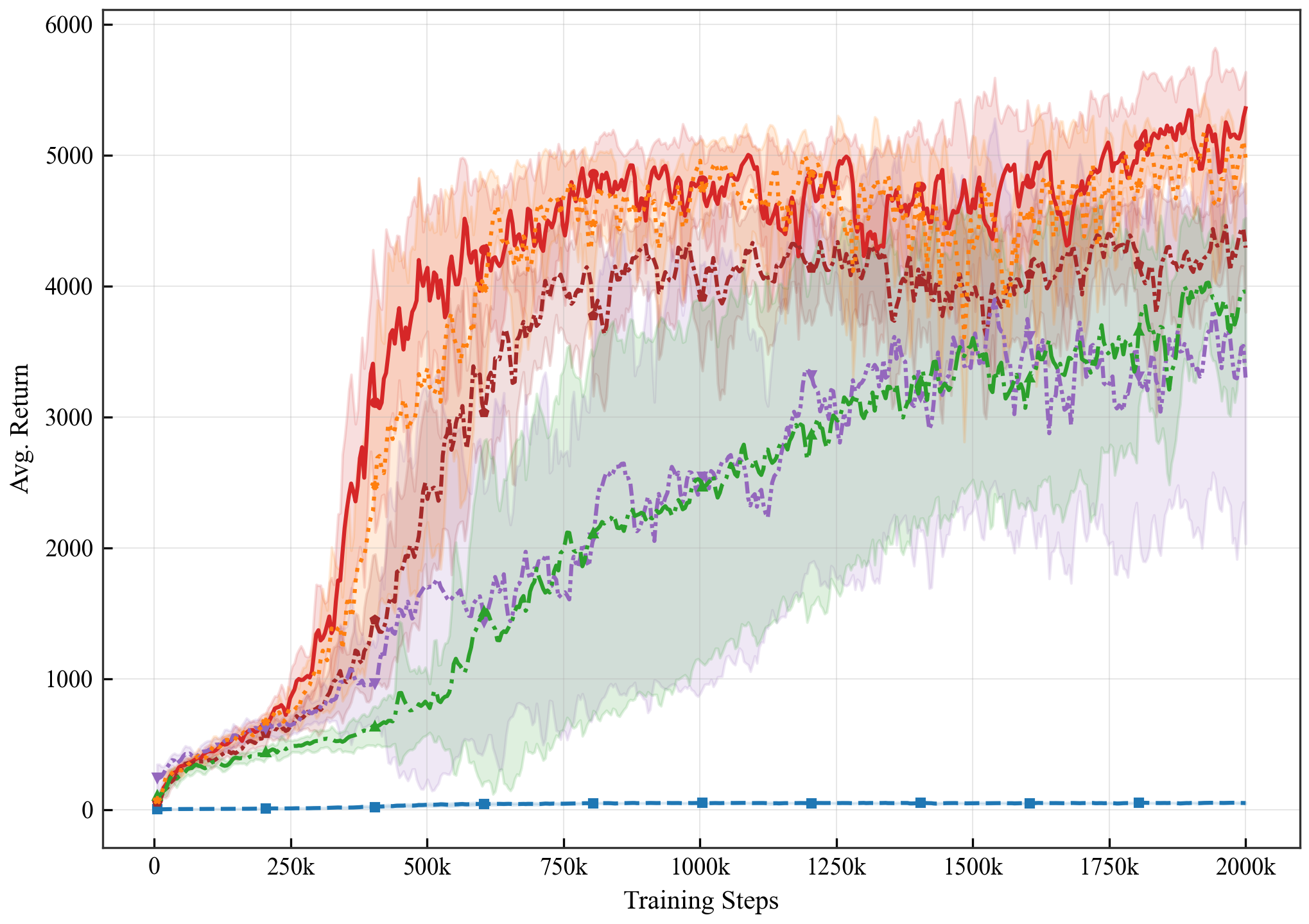}
			\caption{Humanoid-TD3-W}
		\end{subfigure}
		\vspace{0.5em}
		\begin{subfigure}{0.24\linewidth}
			\includegraphics[width=\linewidth]{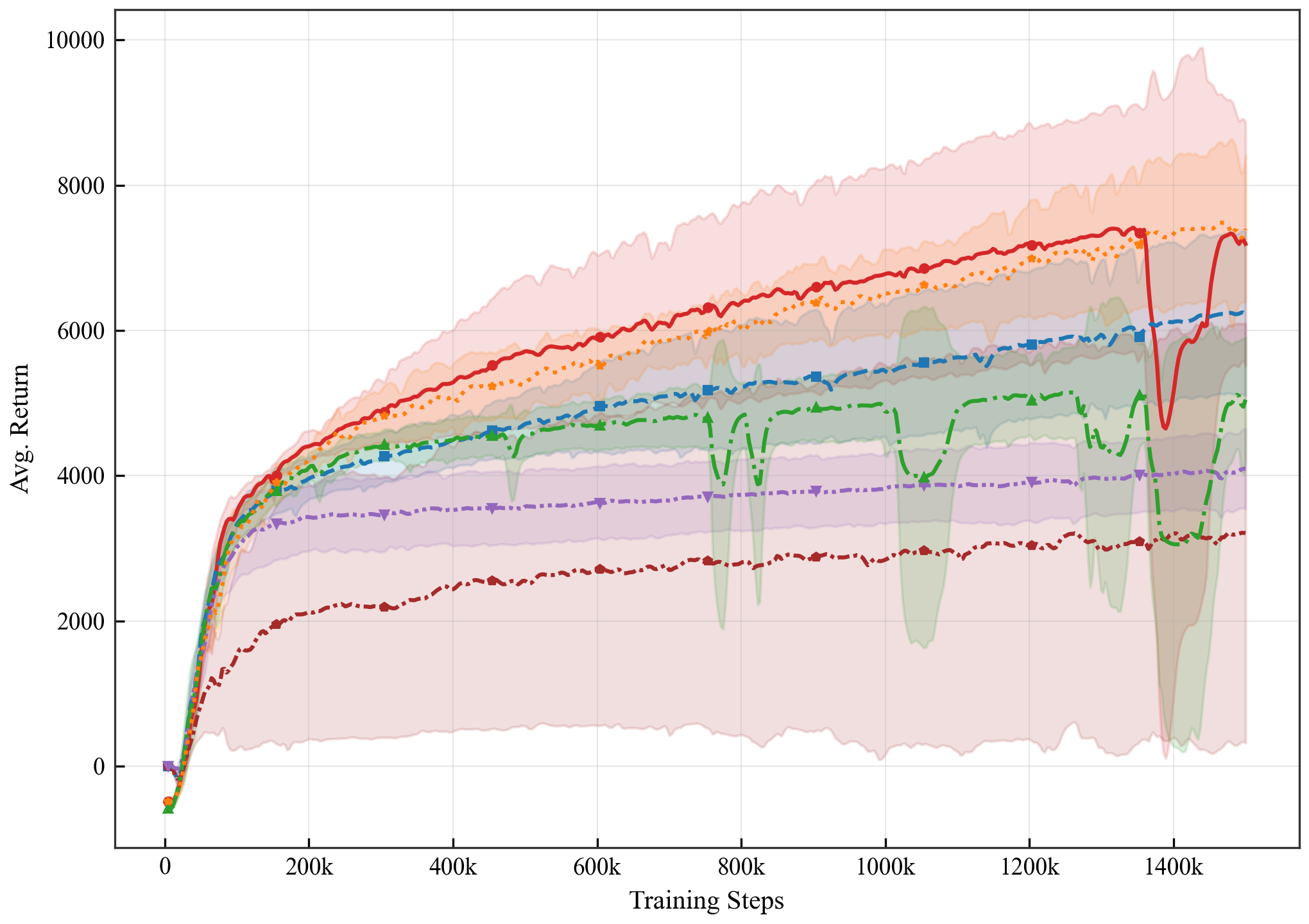}
			\caption{HalfCheetah-TD3-W}
		\end{subfigure}
		\hfill
		\begin{subfigure}{0.24\linewidth}
			\includegraphics[width=\linewidth]{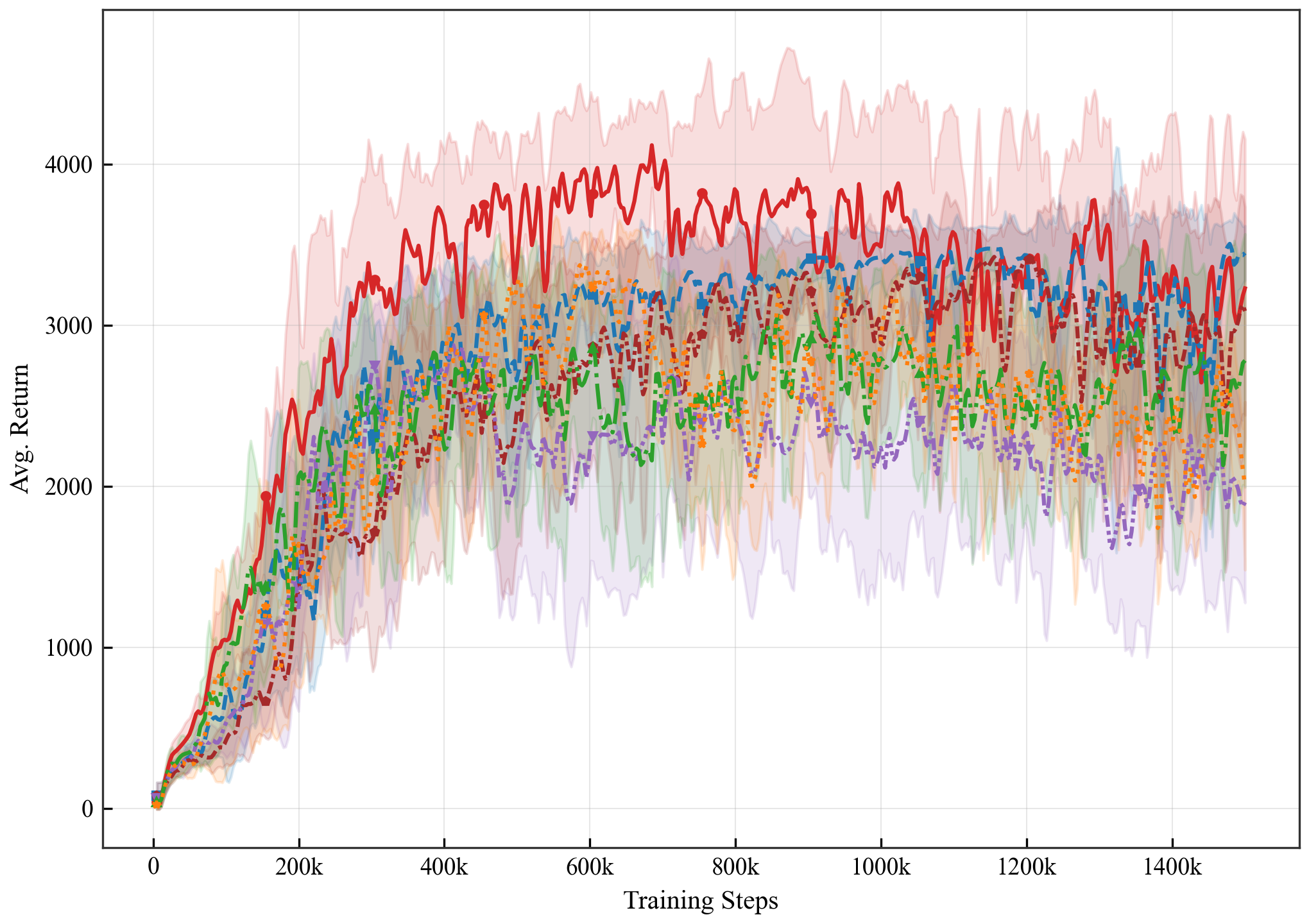}
			\caption{Hopper-TD3-W}
		\end{subfigure}
		
		\vspace{0.2em}
		\begin{subfigure}{0.24\linewidth}
			\includegraphics[width=\linewidth]{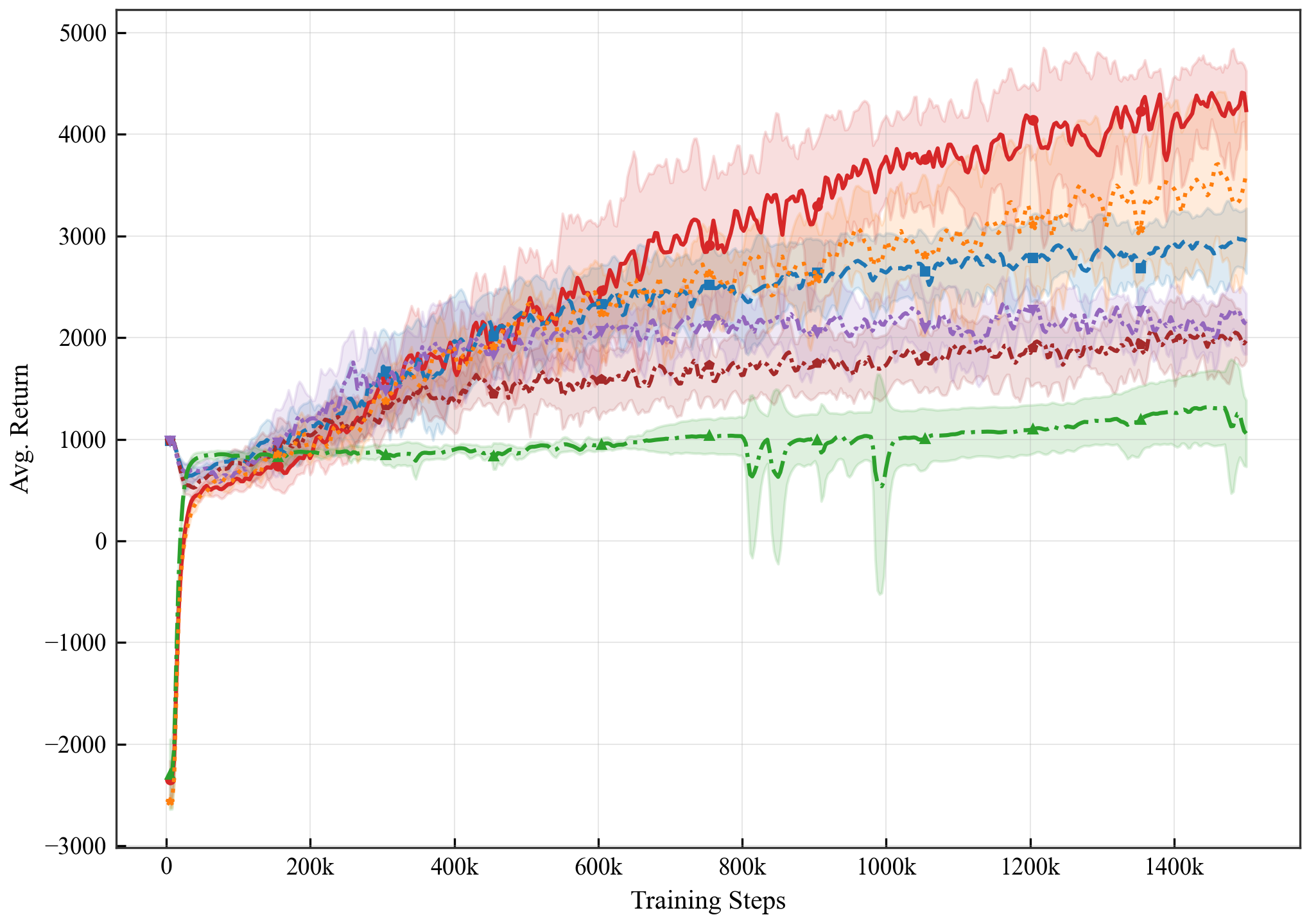}
			\caption{Ant-TD3-S}
		\end{subfigure}
		\hfill
		\begin{subfigure}{0.24\linewidth}
			\includegraphics[width=\linewidth]{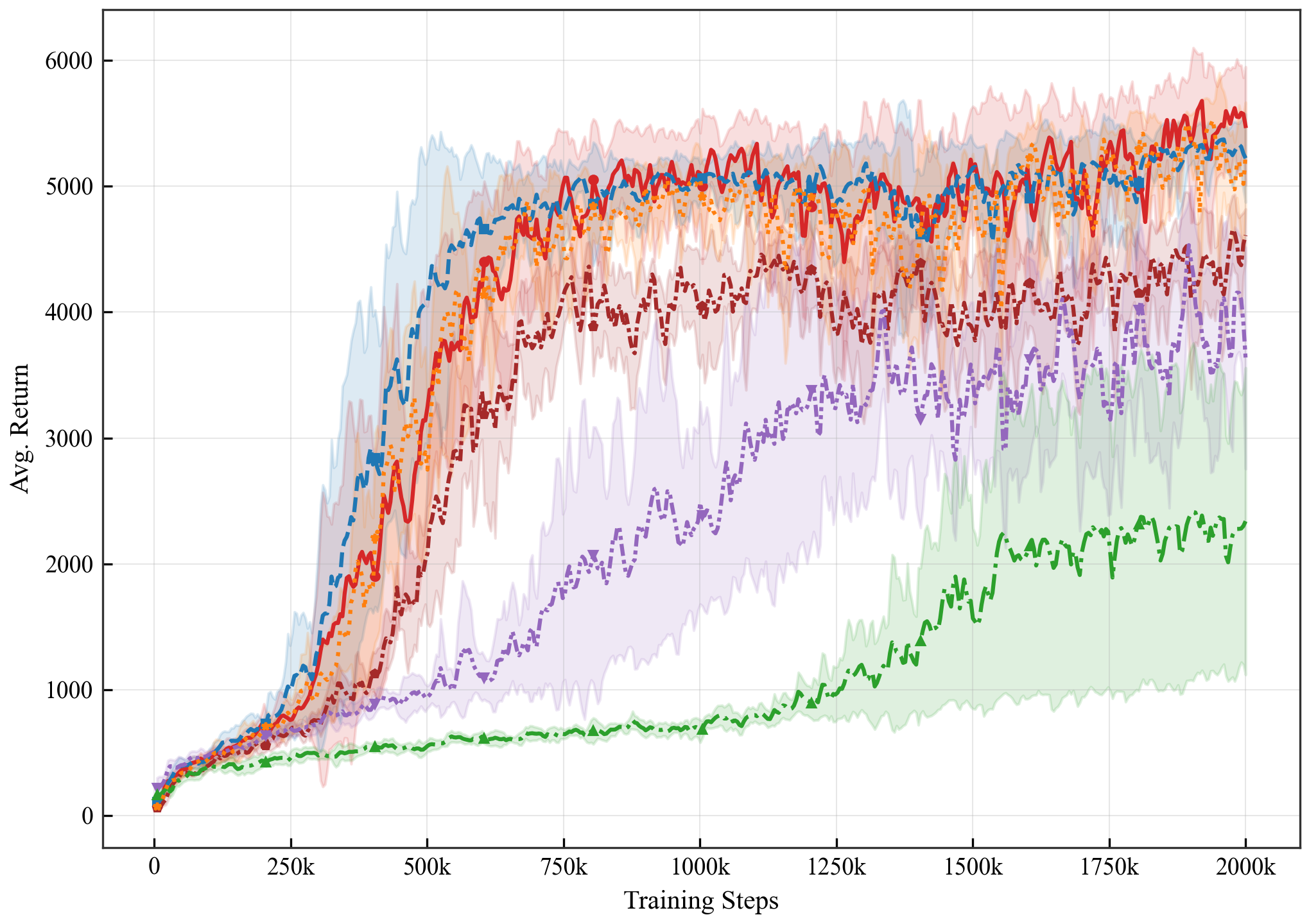}
			\caption{Humanoid-TD3-S}
		\end{subfigure}
		\vspace{0.5em}
		\begin{subfigure}{0.24\linewidth}
			\includegraphics[width=\linewidth]{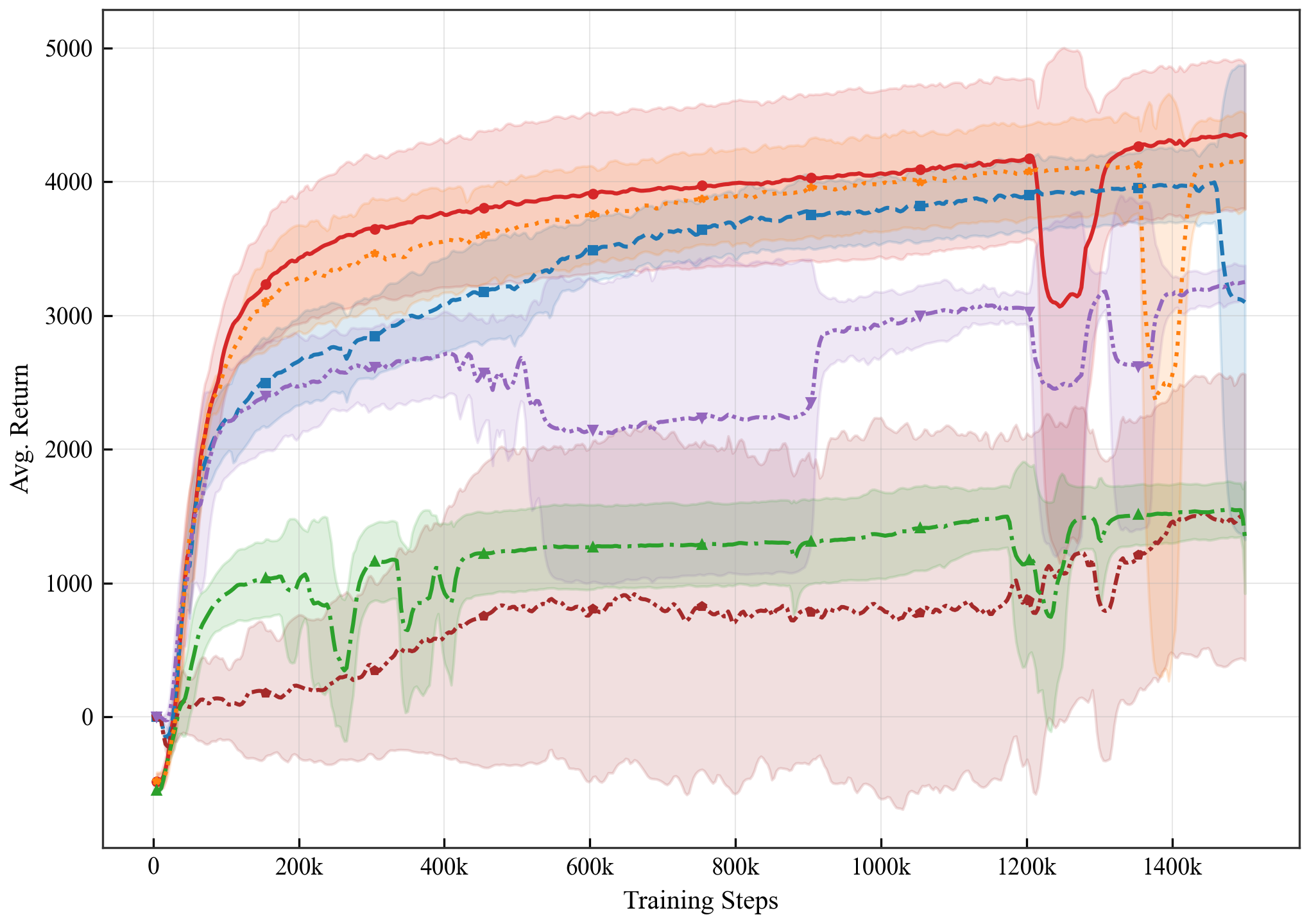}
			\caption{HalfCheetah-TD3-S}
		\end{subfigure}
		\hfill
		\begin{subfigure}{0.24\linewidth}
			\includegraphics[width=\linewidth]{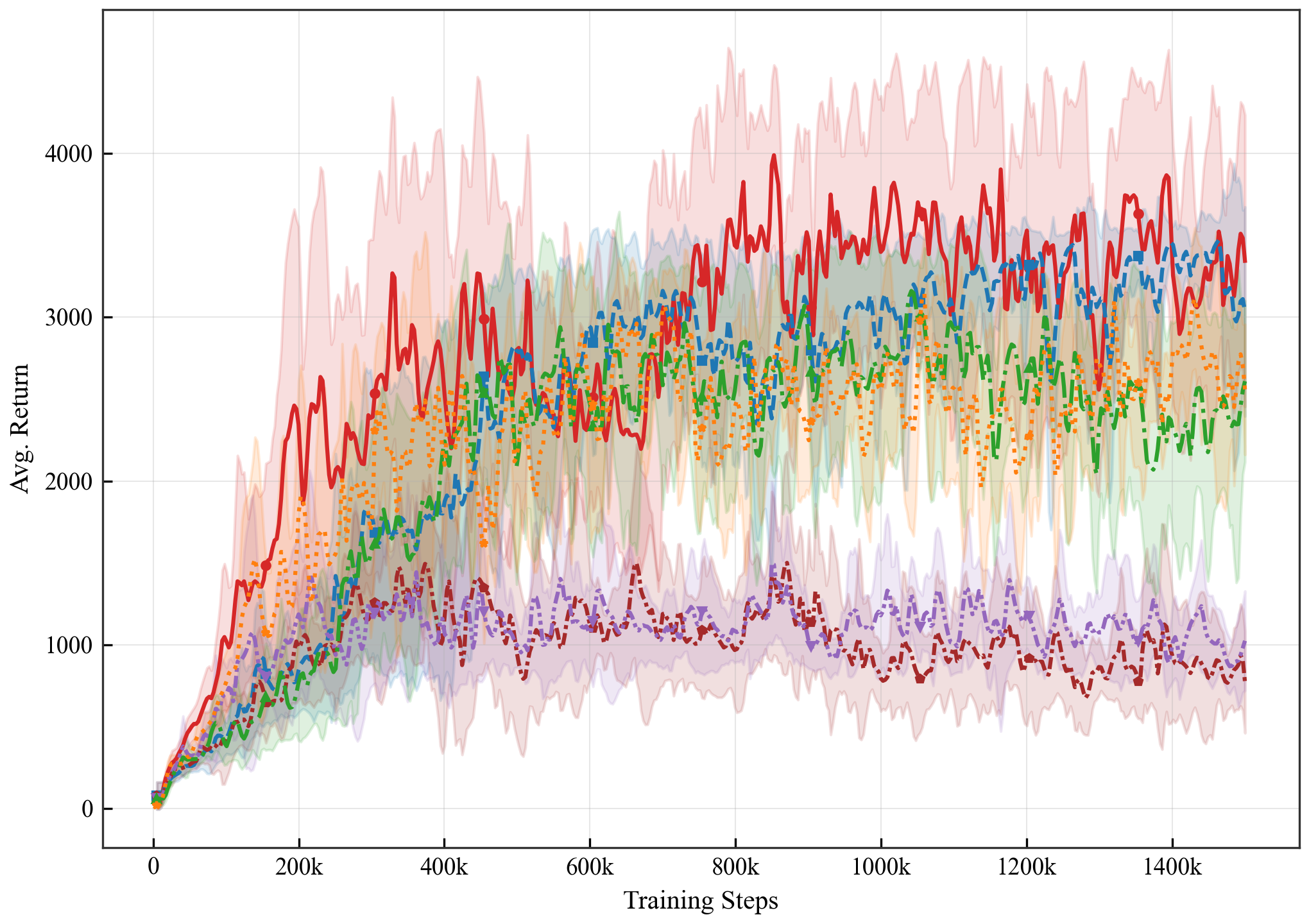}
			\caption{Hopper-TD3-S}
		\end{subfigure}
		
		\vspace{0.15cm}
		\centering
		\begin{tabular}{@{}l@{\hspace{1.5em}}l@{\hspace{1.5em}}l@{\hspace{1.5em}}l@{\hspace{1.5em}}l@{\hspace{1.5em}}l@{}}
			\colorbox{BoxPre}{\rule{0pt}{1pt}\rule{8pt}{0pt}} \raisebox{-2.0pt}{\scriptsize BoxPre+} &
			\colorbox{QP}{\rule{0pt}{1pt}\rule{8pt}{0pt}} \raisebox{-2.0pt}{\scriptsize QP} &
			\colorbox{SRadQP}{\rule{0pt}{1pt}\rule{8pt}{0pt}} \raisebox{-2.0pt}{\scriptsize SRad-QP} &
			\colorbox{SRadStrict}{\rule{0pt}{1pt}\rule{8pt}{0pt}} \raisebox{-2.0pt}{\scriptsize SRad-Strict} &
			\colorbox{dtanh}{\rule{0pt}{1pt}\rule{8pt}{0pt}} \raisebox{-2.0pt}{\scriptsize D-Tanh} &
			\colorbox{DDSRad}{\rule{0pt}{1pt}\rule{8pt}{0pt}} \raisebox{-2.0pt}{\scriptsize DD-SRad}
		\end{tabular}
		
		\caption{SAC/TD3-Based benchmark performance comparison across four environments.}
		\label{fig:bench_reward}
	\end{figure*}
	
	Table~\ref{tab:benchmark_wide} presents the quantitative performance from the §\ref{sec:benchmark} standard benchmark experiments under wide homogeneous constraints (Humanoid-v5: $\boldsymbol{\delta}=[0.4]^d$; all others: $\boldsymbol{\delta}=[1.0]^d$), for reference alongside the tight heterogeneous constraint results in the main Table~\ref{tab:benchmark_main}. Under wide constraints, the rate limits exert weak actual constraint force on the policy, with CVR approaching 0 for all methods and Return gaps narrowing substantially---this pattern itself validates the design expectation that DD-SRad introduces no additional performance loss when constraints are non-binding.
	
	\renewcommand{\arraystretch}{1.15}
	\begin{table*}[htbp]
		\scriptsize
		\centering
		\captionsetup{font=scriptsize, labelfont=scriptsize}
		\caption{Supplementary benchmark results under wide homogeneous constraints (mean$\pm$std over 5 seeds; best in \textbf{bold}; $\dagger$ see footnote)}
		\label{tab:benchmark_wide}
		
		\textbf{(a) SAC Backbone}\\[0.5mm]
		\begin{tabular}{l ccc ccc}
			\toprule
			& \multicolumn{3}{c}{\textbf{Ant-v5}\small($\boldsymbol{\delta}=[1.0]^8$, $\kappa=1.0$)}
			& \multicolumn{3}{c}{\textbf{Humanoid-v5}\small($\boldsymbol{\delta}=[0.4]^{17}$, $\kappa=1.0$)} \\
			\cmidrule(lr){2-4}\cmidrule(lr){5-7}
			\textbf{Method} & Return $\uparrow$ & CV $\downarrow$ & Util
			& Return $\uparrow$ & CV $\downarrow$ & Util \\
			\midrule
			BoxPre+
			& $3737\pm729$           & $19.5$          & $0.323\pm0.014$
			& $5144\pm145$           & $2.8$           & $0.158\pm0.014$ \\
			Post(QP)$^\dagger$
			& $2437\pm467$           & $19.1$          & $0.341\pm0.011$
			& $5163\pm175$           & $3.4$           & $0.289\pm0.018$ \\
			SRad-Strict
			& $2952\pm385$           & $\mathbf{13.0}$ & $0.244\pm0.009$
			& $3918\pm1350$          & $34.5$          & $0.057\pm0.016$ \\
			SRad-QP
			& $3122\pm719$           & $23.0$          & $0.253\pm0.012$
			& $3903\pm1324$          & $33.9$          & $0.060\pm0.015$ \\
			D-Tanh
			& $4039\pm725$           & $18.0$          & $0.301\pm0.009$
			& $5212\pm135$           & $\mathbf{2.6}$  & $0.253\pm0.029$ \\
			\rowcolor{blue!10}
			DD-SRad (Ours)
			& $\mathbf{4199\pm957}$  & $22.8$          & $0.334\pm0.026$
			& $\mathbf{5384\pm261}$  & $4.8$           & $0.241\pm0.049$ \\
			\midrule
			& \multicolumn{3}{c}{\textbf{HalfCheetah-v5}\small($\boldsymbol{\delta}=[1.0]^6$, $\kappa=1.0$)}
			& \multicolumn{3}{c}{\textbf{Hopper-v5}\small($\boldsymbol{\delta}=[1.0]^3$, $\kappa=1.0$)} \\
			\cmidrule(lr){2-4}\cmidrule(lr){5-7}
			\textbf{Method} & Return $\uparrow$ & CV $\downarrow$ & Util
			& Return $\uparrow$ & CV $\downarrow$ & Util \\
			\midrule
			BoxPre+
			& $5383\pm412$           & $\mathbf{7.7}$  & $0.528\pm0.009$
			& $2226\pm437$           & $19.6$          & $0.145\pm0.009$ \\
			Post(QP)$^\dagger$
			& $5582\pm704$           & $12.6$          & $0.578\pm0.052$
			& $2037\pm425$           & $20.9$          & $0.167\pm0.009$ \\
			SRad-Strict
			& $3921\pm469$           & $12.0$          & $0.273\pm0.020$
			& $1992\pm145$           & $\mathbf{7.3}$  & $0.193\pm0.016$ \\
			SRad-QP
			& $4186\pm439$           & $10.5$          & $0.270\pm0.006$
			& $1915\pm188$           & $9.8$           & $0.185\pm0.017$ \\
			D-Tanh
			& $5646\pm517$           & $9.2$           & $0.626\pm0.018$
			& $2649\pm194$           & $\mathbf{7.3}$  & $0.361\pm0.016$ \\
			\rowcolor{blue!10}
			DD-SRad (Ours)
			& $\mathbf{5925\pm782}$  & $13.2$          & $0.555\pm0.103$
			& $\mathbf{2764\pm405}$  & $14.7$          & $0.325\pm0.012$ \\
			\bottomrule
		\end{tabular}
		
		\vspace{3mm}
		\scriptsize
		\textbf{(b) TD3 Backbone}\\[0.5mm]
		\begin{tabular}{l ccc ccc}
			\toprule
			& \multicolumn{3}{c}{\textbf{Ant-v5}\small($\boldsymbol{\delta}=[1.0]^8$, $\kappa=1.0$)}
			& \multicolumn{3}{c}{\textbf{Humanoid-v5}\small($\boldsymbol{\delta}=[0.4]^{17}$, $\kappa=1.0$)} \\
			\cmidrule(lr){2-4}\cmidrule(lr){5-7}
			\textbf{Method} & Return $\uparrow$ & CV $\downarrow$ & Util
			& Return $\uparrow$ & CV $\downarrow$ & Util \\
			\midrule
			BoxPre+
			& $4906\pm708$           & $14.4$          & $0.307\pm0.021$
			& $5018\pm144$           & $2.9$           & $0.154\pm0.025$ \\
			Post(QP)$^\dagger$
			& $2613\pm1011$          & $38.7$          & $0.309\pm0.033$
			& $4289\pm111$           & $\mathbf{2.6}$  & $0.537\pm0.044$ \\
			SRad-Strict
			& $3756\pm1287$          & $34.3$          & $0.236\pm0.015$
			& $3884\pm404$           & $10.4$          & $0.127\pm0.020$ \\
			SRad-QP
			& $4973\pm716$           & $14.4$          & $0.214\pm0.009$
			& $3483\pm1009$          & $29.0$          & $0.018\pm0.005$ \\
			D-Tanh
			& $4058\pm483$           & $\mathbf{11.9}$ & $0.353\pm0.005$
			& $5027\pm144$           & $2.9$           & $0.242\pm0.047$ \\
			\rowcolor{blue!10}
			DD-SRad (Ours)
			& $\mathbf{5088\pm881}$  & $17.3$          & $0.362\pm0.031$
			& $\mathbf{5107\pm343}$  & $6.7$           & $0.207\pm0.020$ \\
			\midrule
			& \multicolumn{3}{c}{\textbf{HalfCheetah-v5}\small($\boldsymbol{\delta}=[1.0]^6$, $\kappa=1.0$)}
			& \multicolumn{3}{c}{\textbf{Hopper-v5}\small($\boldsymbol{\delta}=[1.0]^3$, $\kappa=1.0$)} \\
			\cmidrule(lr){2-4}\cmidrule(lr){5-7}
			\textbf{Method} & Return $\uparrow$ & CV $\downarrow$ & Util
			& Return $\uparrow$ & CV $\downarrow$ & Util \\
			\midrule
			BoxPre+
			& $6198\pm1093$          & $17.6$          & $0.510\pm0.059$
			& $3157\pm126$           & $\mathbf{4.0}$  & $0.178\pm0.006$ \\
			Post(QP)$^\dagger$
			& $3163\pm2881$          & $91.1$          & $0.589\pm0.102$
			& $2915\pm473$           & $16.2$          & $0.281\pm0.023$ \\
			SRad-Strict
			& $4119\pm1665$          & $40.4$          & $0.264\pm0.035$
			& $2598\pm556$           & $21.4$          & $0.168\pm0.017$ \\
			SRad-QP
			& $4051\pm521$           & $\mathbf{12.9}$ & $0.339\pm0.023$
			& $2030\pm406$           & $20.0$          & $0.163\pm0.018$ \\
			D-Tanh
			& $6596\pm942$           & $14.3$          & $0.611\pm0.064$
			& $2303\pm122$           & $5.3$           & $0.371\pm0.020$ \\
			\rowcolor{blue!10}
			DD-SRad (Ours)
			& $\mathbf{6650\pm2477}$ & $37.2$          & $0.523\pm0.119$
			& $\mathbf{3216\pm478}$  & $14.9$          & $0.337\pm0.017$ \\
			\bottomrule
		\end{tabular}
		
		\vspace{1.5mm}
		\scriptsize
		\textbf{Notes:}
		CVR\,${=}$\,0 for all parameterization-based methods across all configurations, confirmed by Theorem~\ref{thm:constraint}.
		Return gaps across methods narrow substantially compared to Table~\ref{tab:benchmark_main}
		(tight heterogeneous constraints),
		consistent with the expectation that DD-SRad's advantage is most pronounced
		when constraints are binding.
		$^\dagger$\,SAC-Post(QP)/TD3-Post(QP): Util is inflated by clip boundary-dwelling
		(see Table~\ref{tab:benchmark_main} Notes).
	\end{table*}
	\renewcommand{\arraystretch}{1}
	
	\section{Per-Dimension Utilization Radar Charts Under Benchmark}
	\label{app:radar_benchmark}
	
	To quantify how fully each method exploits its per-dimension rate budget, define the constraint utilization of dimension $i$ as
	\begin{equation}
		u^i \;=\; \frac{1}{T}\sum_{t=1}^{T} \frac{|\Delta a_t^i|}{\delta^i},
	\end{equation}
	where $\Delta a_t^i = a_t^i - a_{t-1}^i$ is the executed action increment at step $t$ and $\delta^i$ is the per-dimension rate limit. By construction $u^i \in [0, 1]$, with $u^i = 1$ indicating that the constraint is saturated on average and $u^i = 0$ indicating the dimension is never moved.
	
	Figure~\ref{fig:radar_all} presents the post-convergence per-dimension utilization radar charts across four MuJoCo environments under tight heterogeneous constraints, for both SAC and TD3 backbones. Each axis of each polygon corresponds to one action dimension, with the radial value indicating the 5-seed mean utilization $\mathbb{E}[|\Delta a^i|]/\delta^i$ for that dimension. The DD-SRad polygon (red) hugs the boundary corresponding to each dimension's $\delta^i$; BoxPre+ (blue) exhibits systematic inward contraction on high-budget dimensions (axes with larger $\delta^i$); SRad-Strict (green) shows the most severe collapse on high-budget dimensions, consistent with Theorem~\ref{thm:exploration}'s prediction that $\ell_2$ parameterization achieves utilization of approximately $\min_j\delta^j/\delta^i$ on those dimensions; QP Post (dark red) exhibits inflated Util due to boundary-dwelling behavior but carries no geometric coverage meaning (see §\ref{sec:benchmark} footnote). Under wide homogeneous constraints (panels (a)(c)(i)(k)), performance gaps across methods narrow, consistent with the expectation that constraints are non-binding.
	
	\begin{figure*}[htbp]
		\centering
		\begin{subfigure}{0.24\linewidth}
			\includegraphics[width=\linewidth]{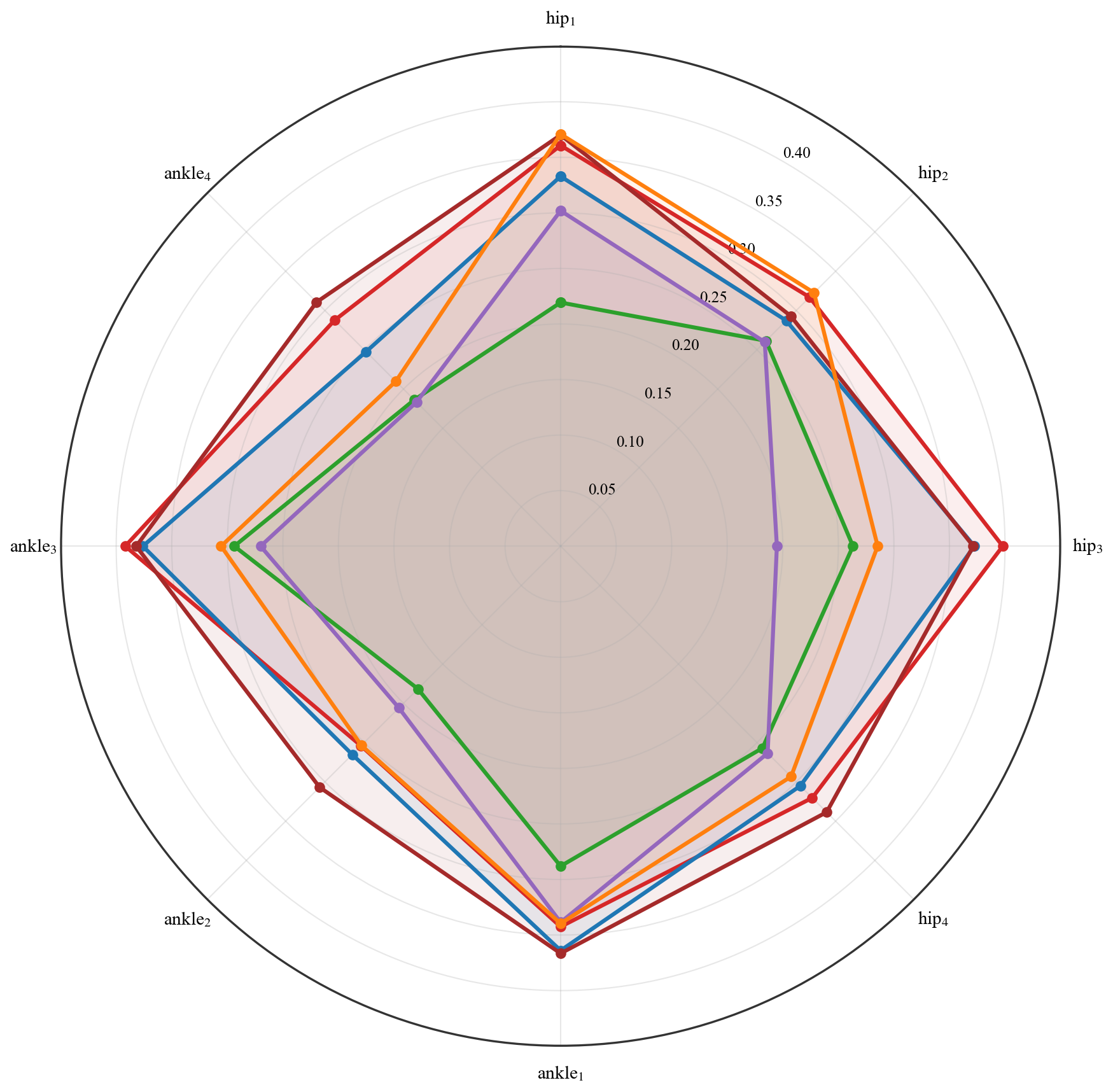}
			\caption{Ant-SAC-W}
		\end{subfigure}
		\hfill
		\begin{subfigure}{0.24\linewidth}
			\includegraphics[width=\linewidth]{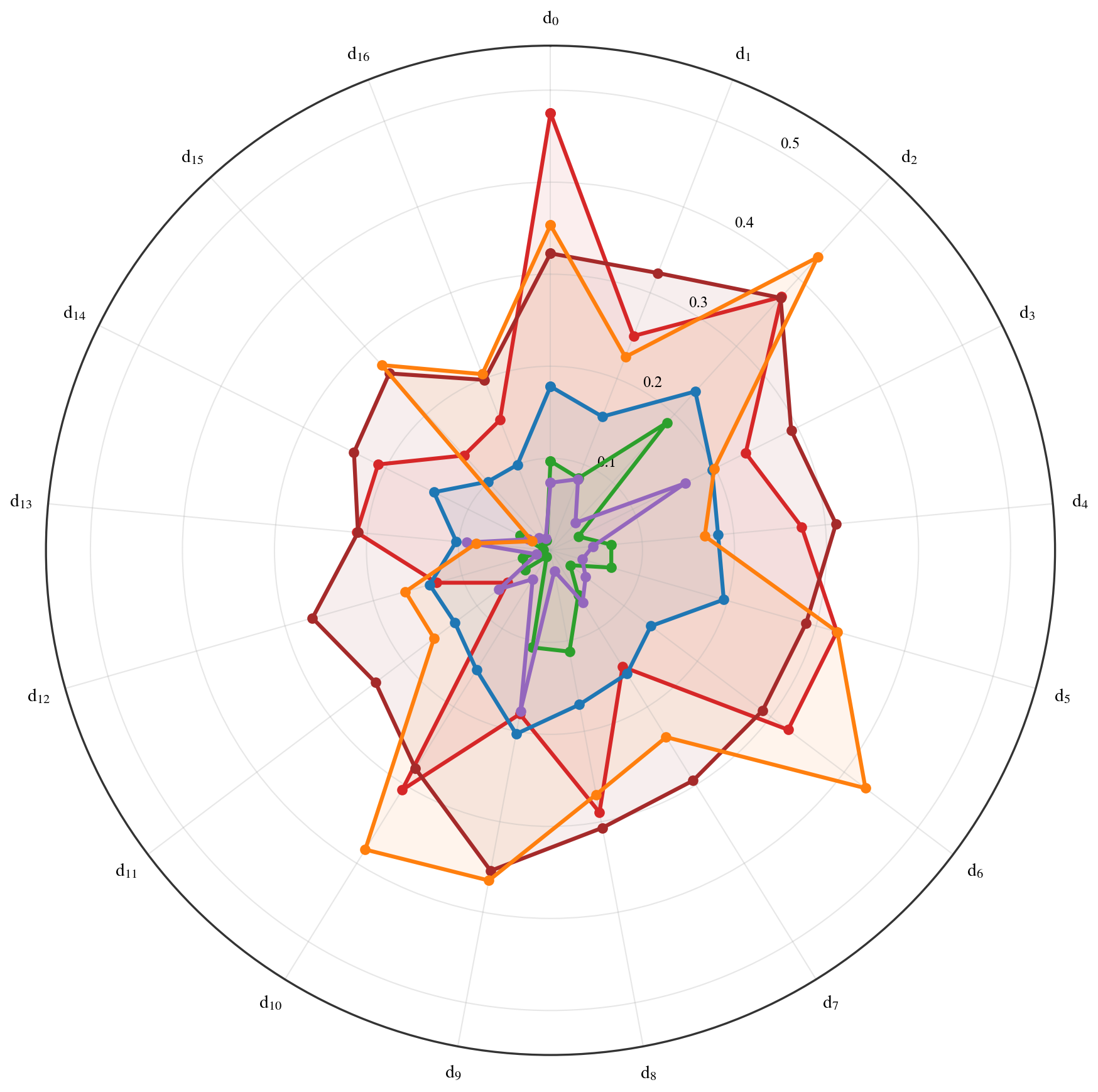}
			\caption{Humanoid-SAC-W}
		\end{subfigure}
		\vspace{0.5em}
		\begin{subfigure}{0.24\linewidth}
			\includegraphics[width=\linewidth]{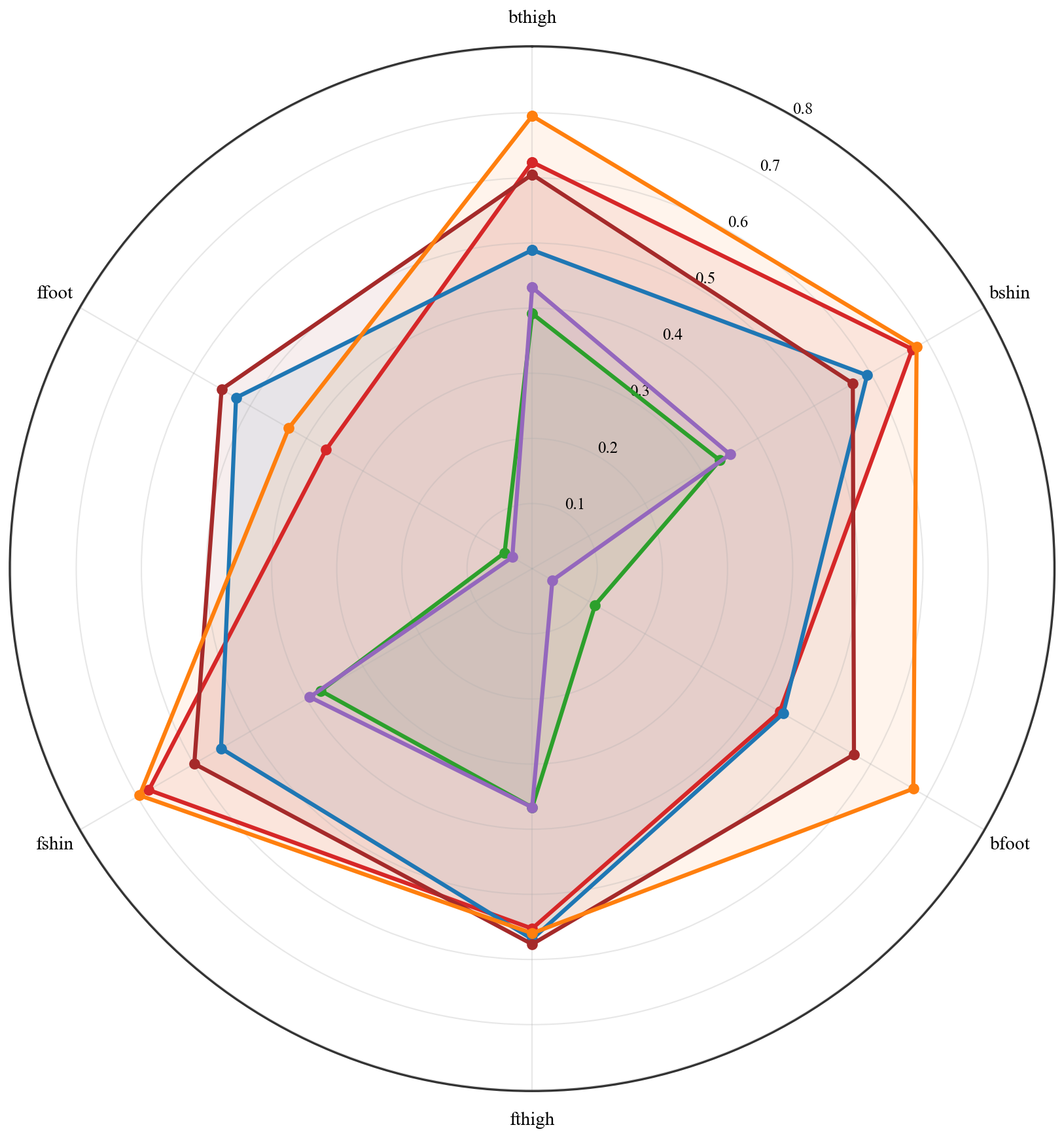}
			\caption{HalfCheetah-SAC-W}
		\end{subfigure}
		\hfill
		\begin{subfigure}{0.24\linewidth}
			\includegraphics[width=\linewidth]{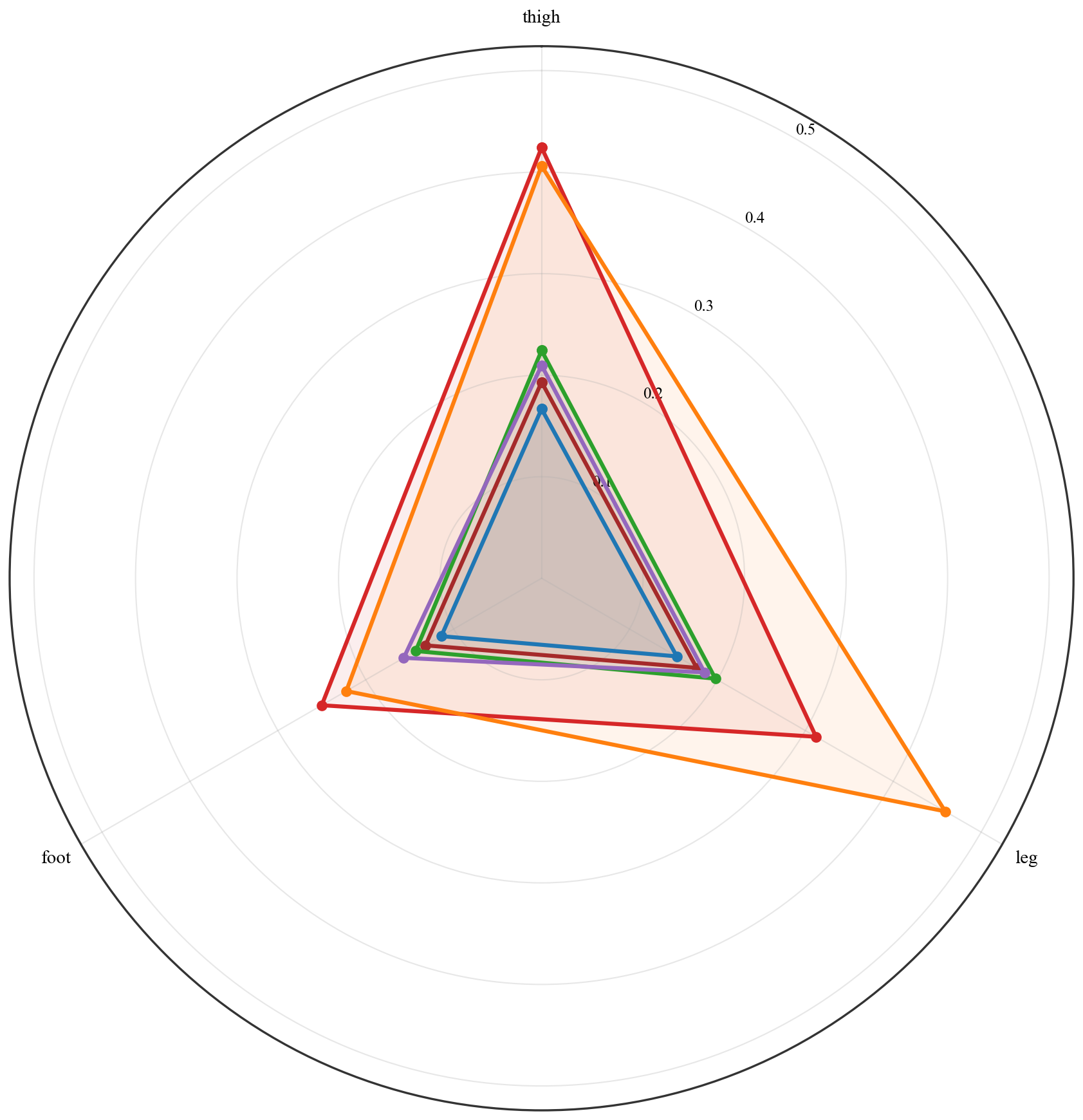}
			\caption{Hopper-SAC-W}
		\end{subfigure}
		
		\vspace{0.2em}
		\begin{subfigure}{0.24\linewidth}
			\includegraphics[width=\linewidth]{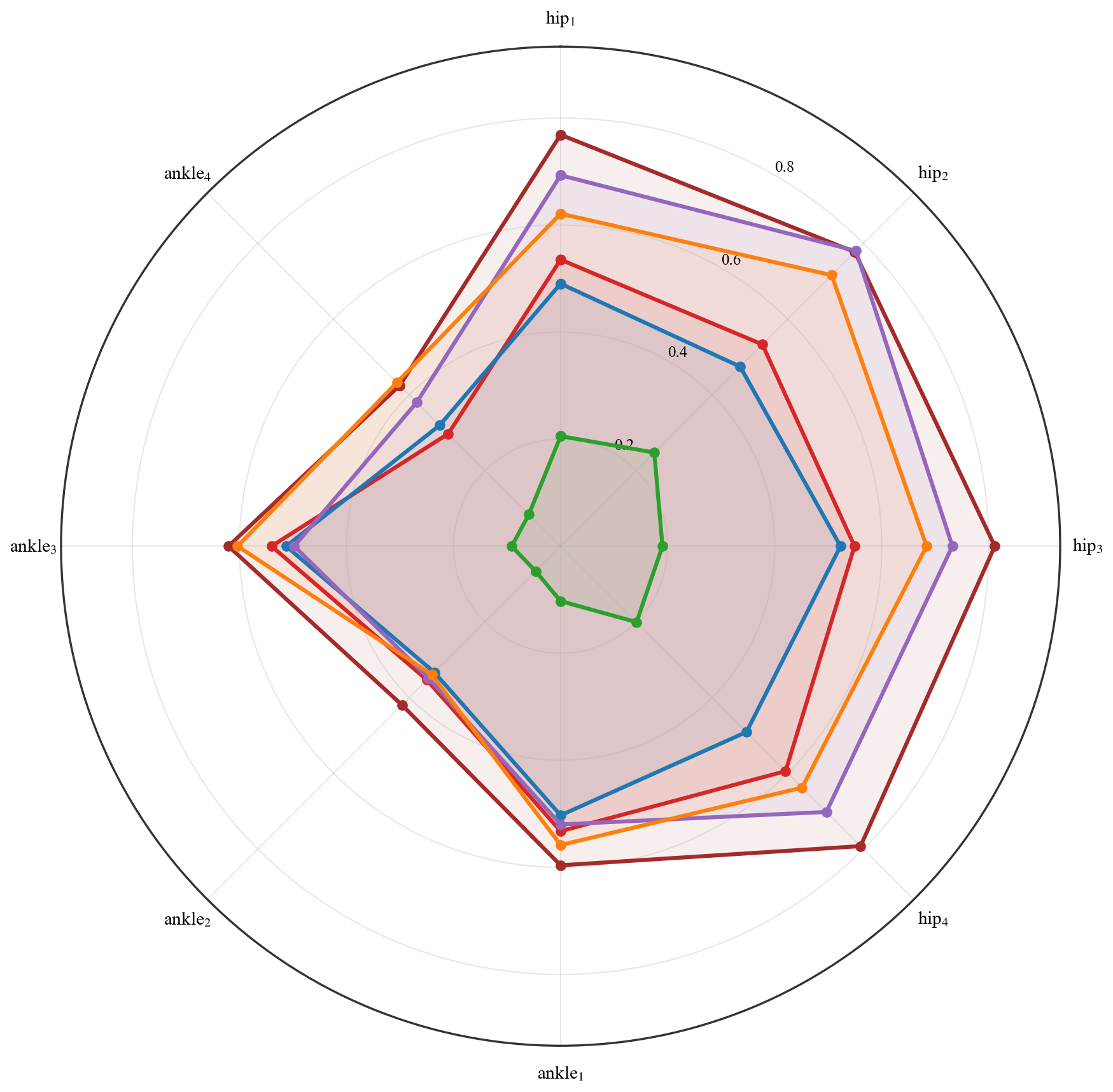}
			\caption{Ant-SAC-S}
		\end{subfigure}
		\hfill
		\begin{subfigure}{0.24\linewidth}
			\includegraphics[width=\linewidth]{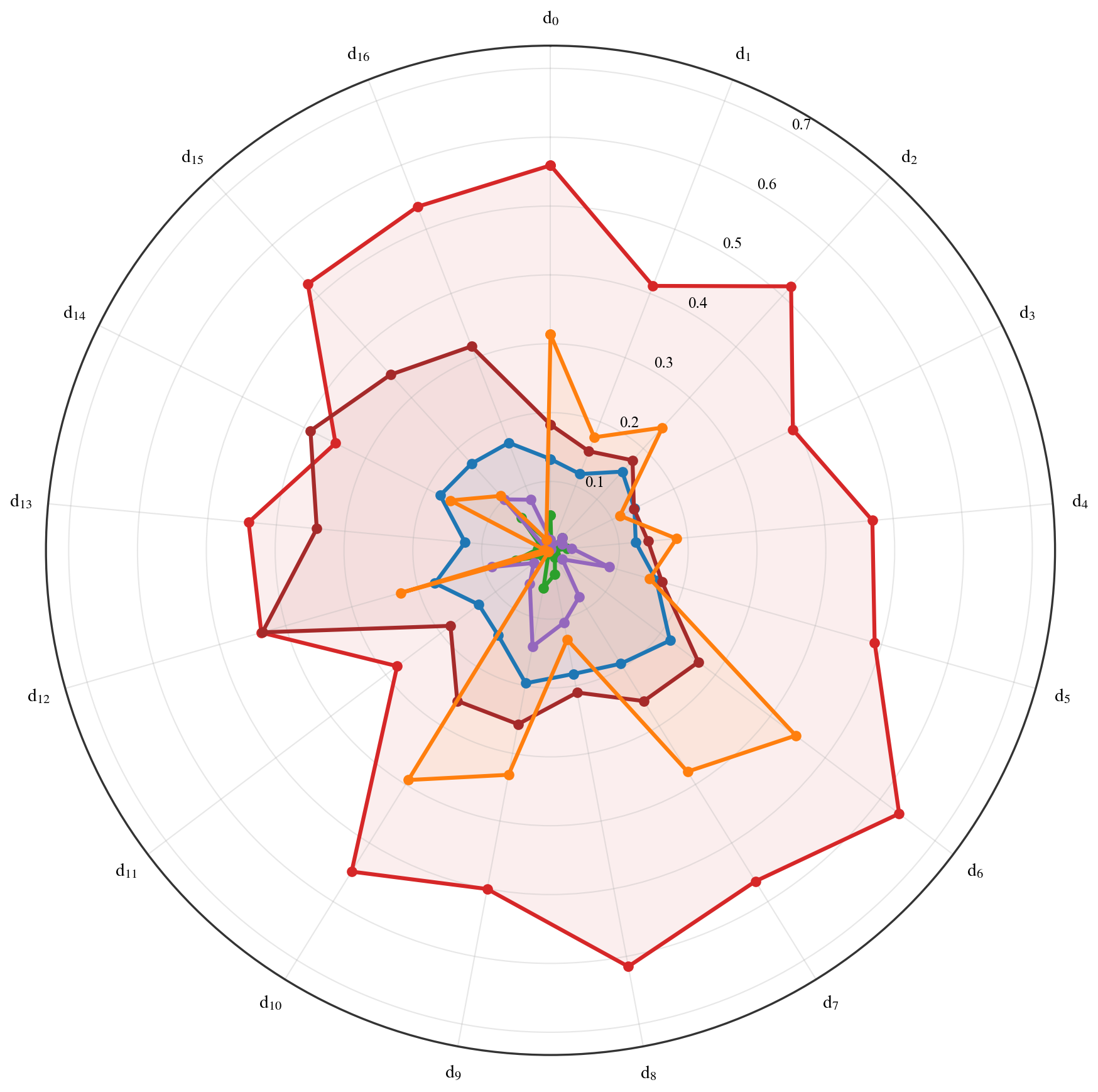}
			\caption{Humanoid-SAC-S}
		\end{subfigure}
		\vspace{0.5em}
		\begin{subfigure}{0.24\linewidth}
			\includegraphics[width=\linewidth]{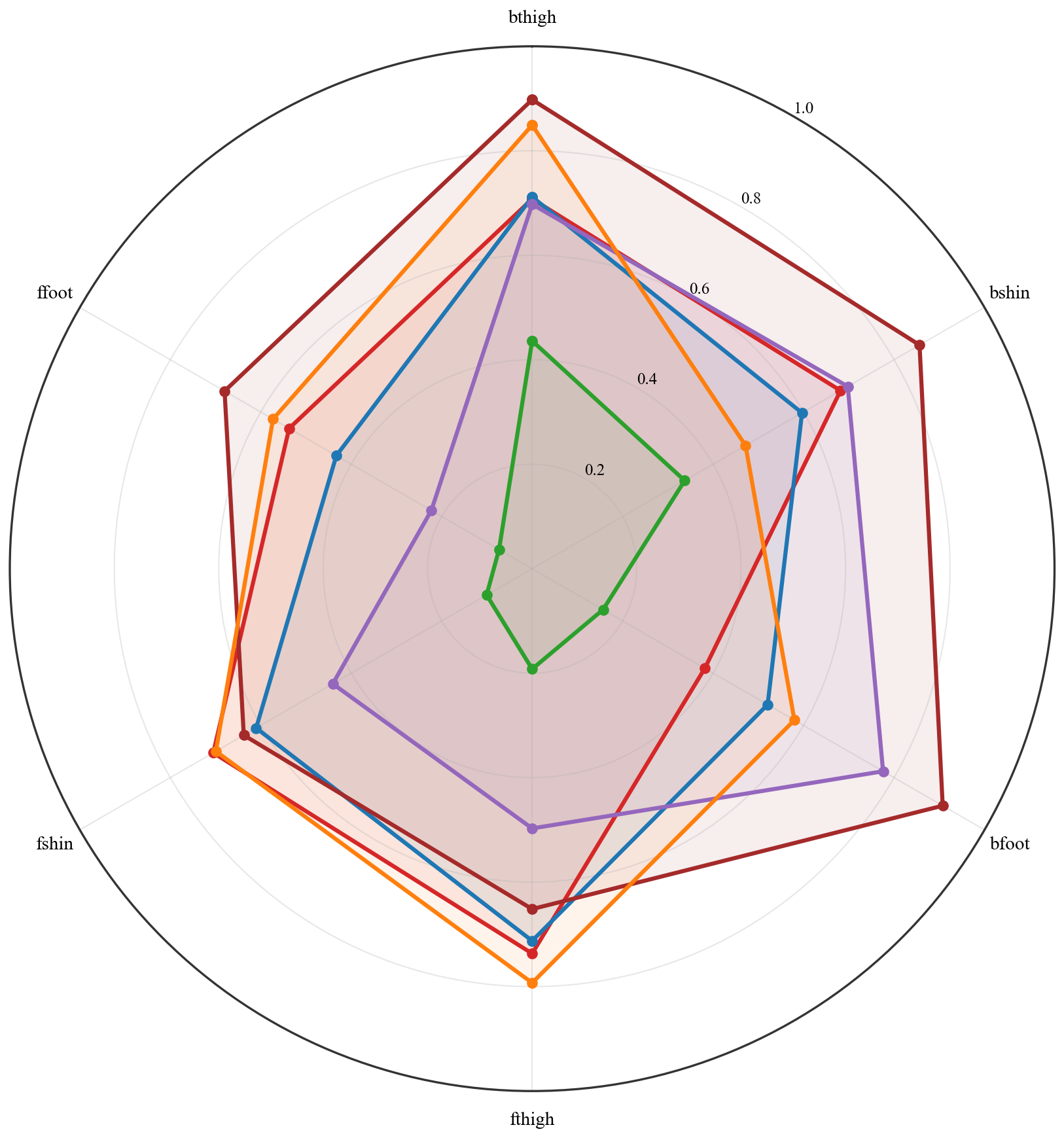}
			\caption{HalfCheetah-SAC-S}
		\end{subfigure}
		\hfill
		\begin{subfigure}{0.24\linewidth}
			\includegraphics[width=\linewidth]{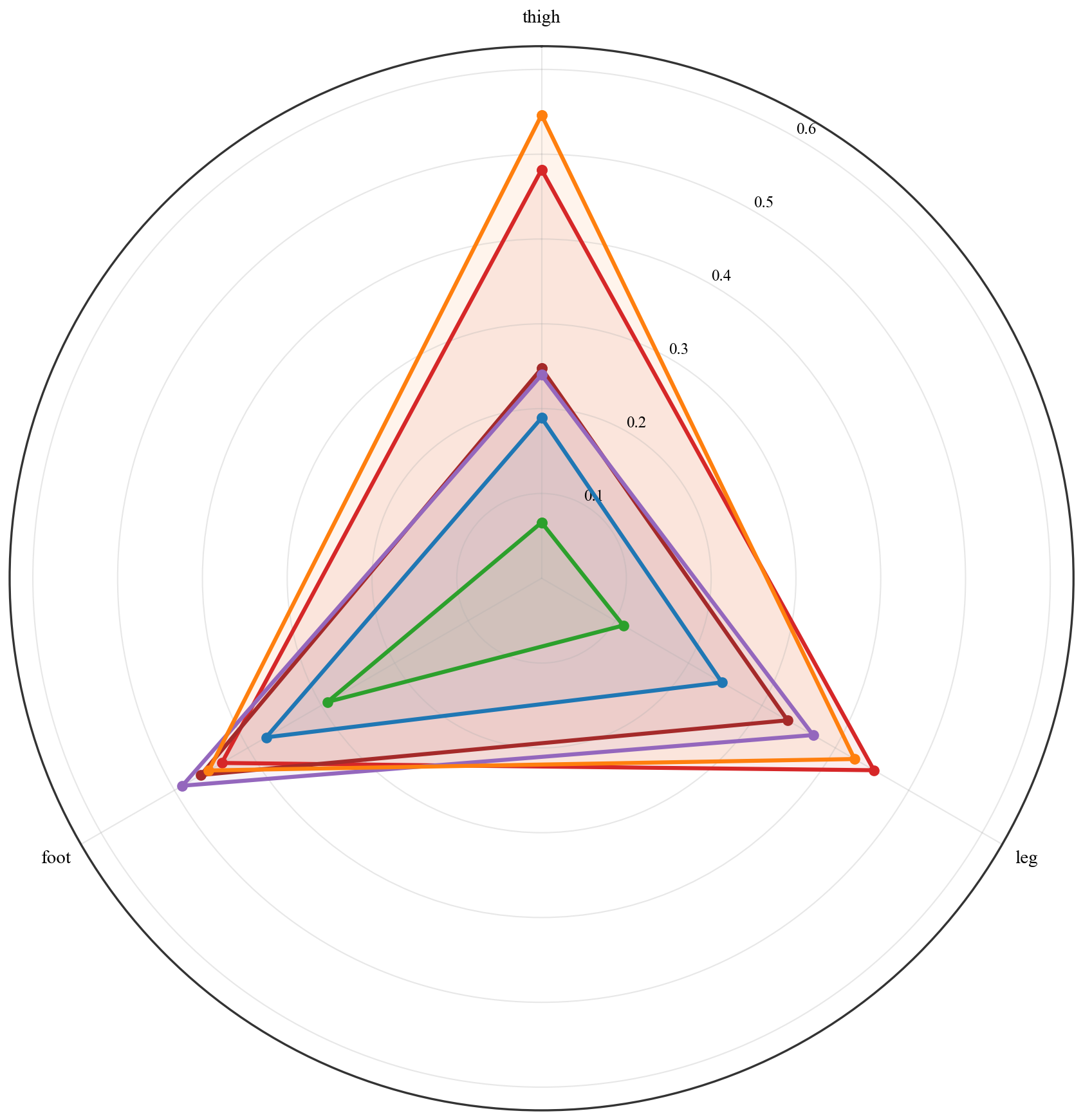}
			\caption{Hopper-SAC-S}
		\end{subfigure}
		
		\begin{subfigure}{0.24\linewidth}
			\includegraphics[width=\linewidth]{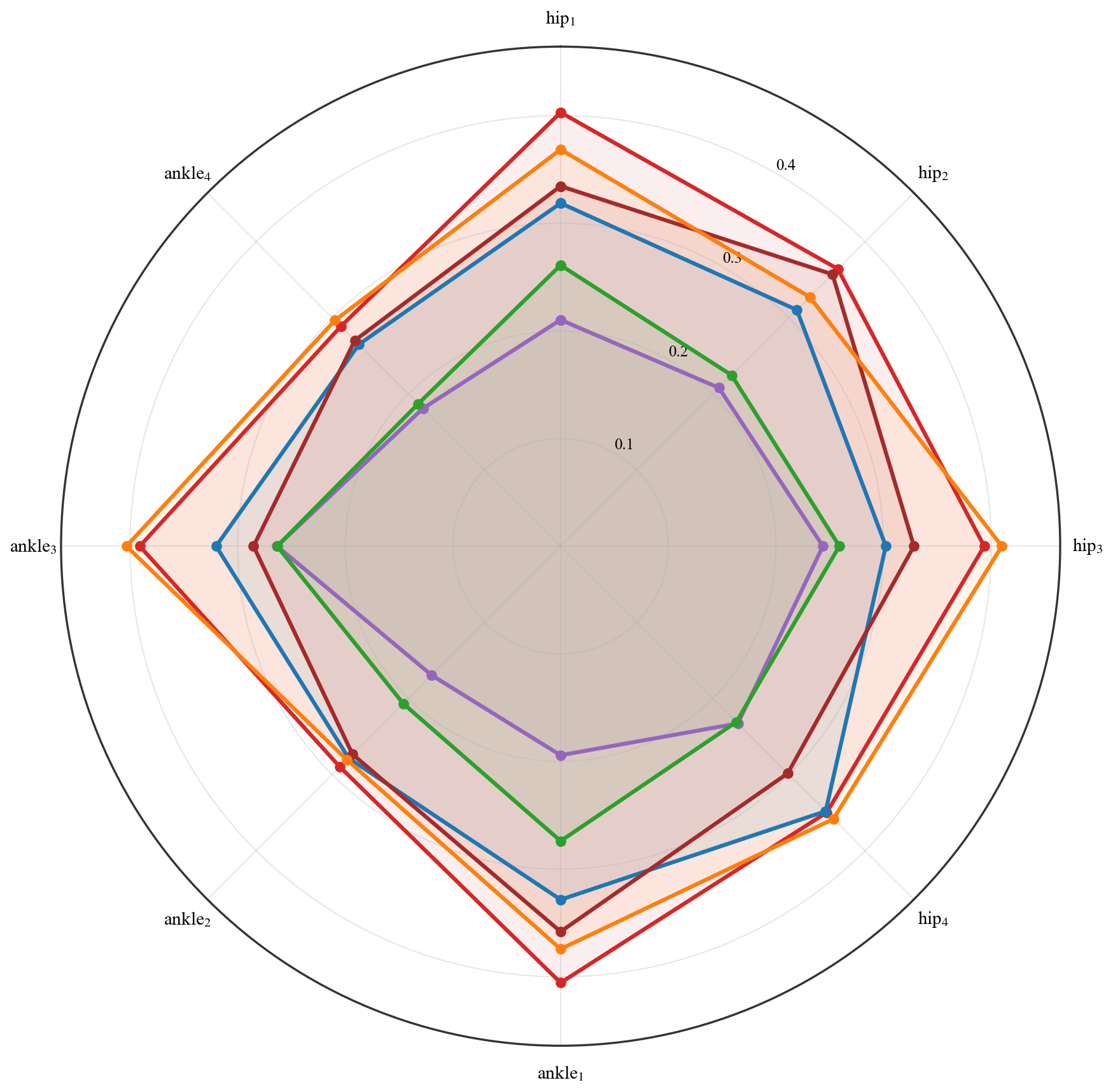}
			\caption{Ant-TD3-W}
		\end{subfigure}
		\hfill
		\begin{subfigure}{0.24\linewidth}
			\includegraphics[width=\linewidth]{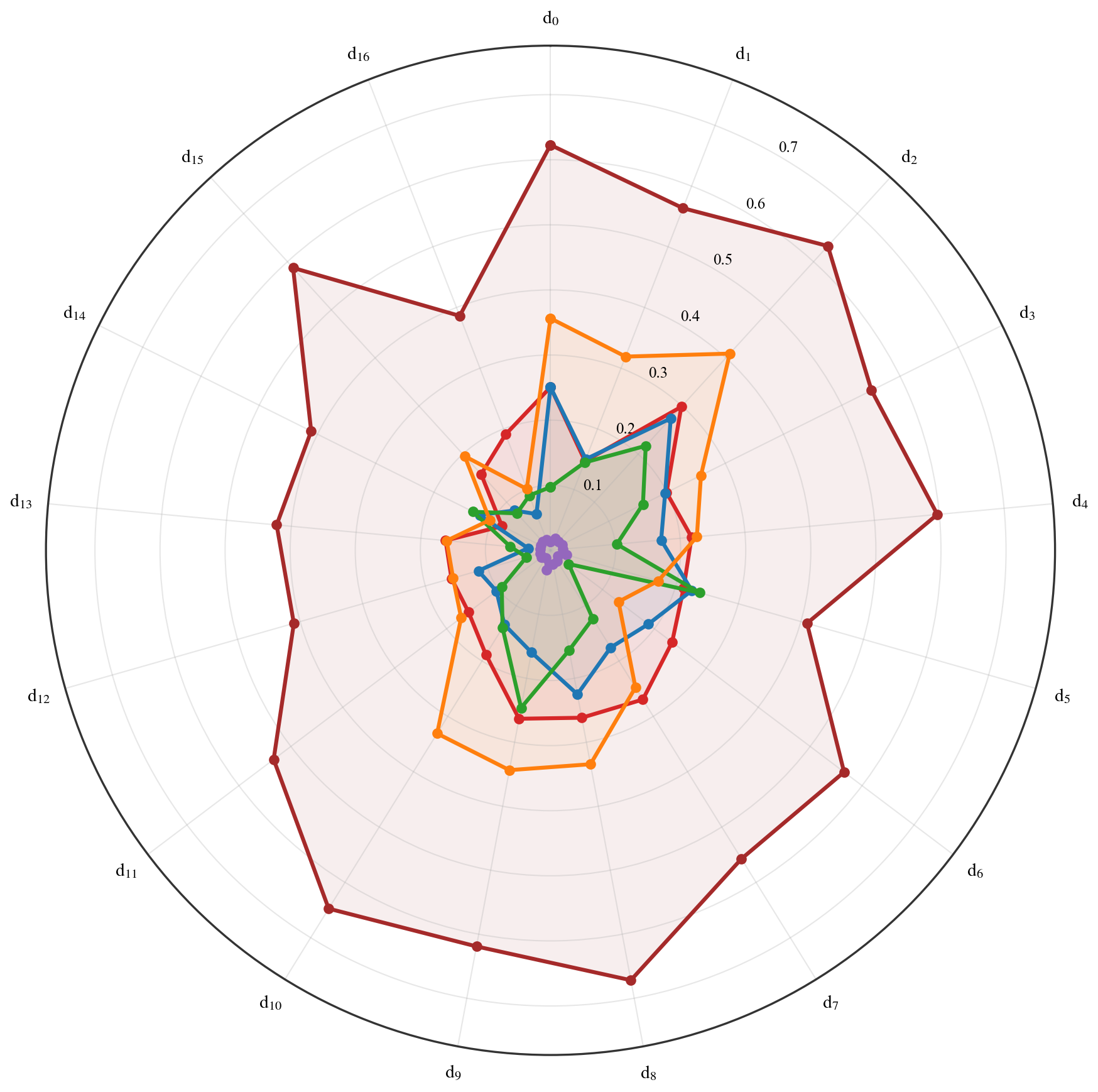}
			\caption{Humanoid-TD3-W}
		\end{subfigure}
		\vspace{0.5em}
		\begin{subfigure}{0.24\linewidth}
			\includegraphics[width=\linewidth]{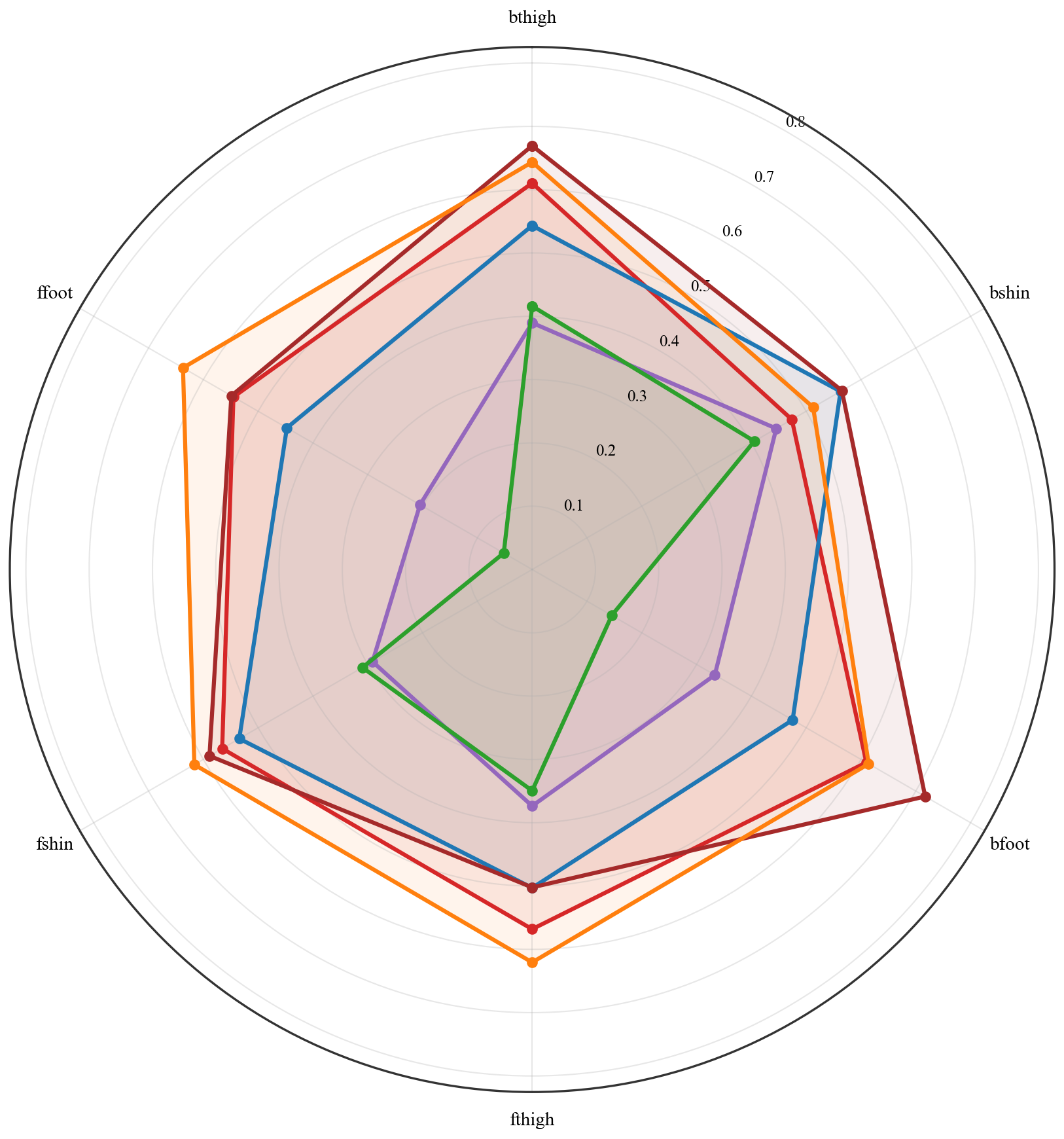}
			\caption{HalfCheetah-TD3-W}
		\end{subfigure}
		\hfill
		\begin{subfigure}{0.24\linewidth}
			\includegraphics[width=\linewidth]{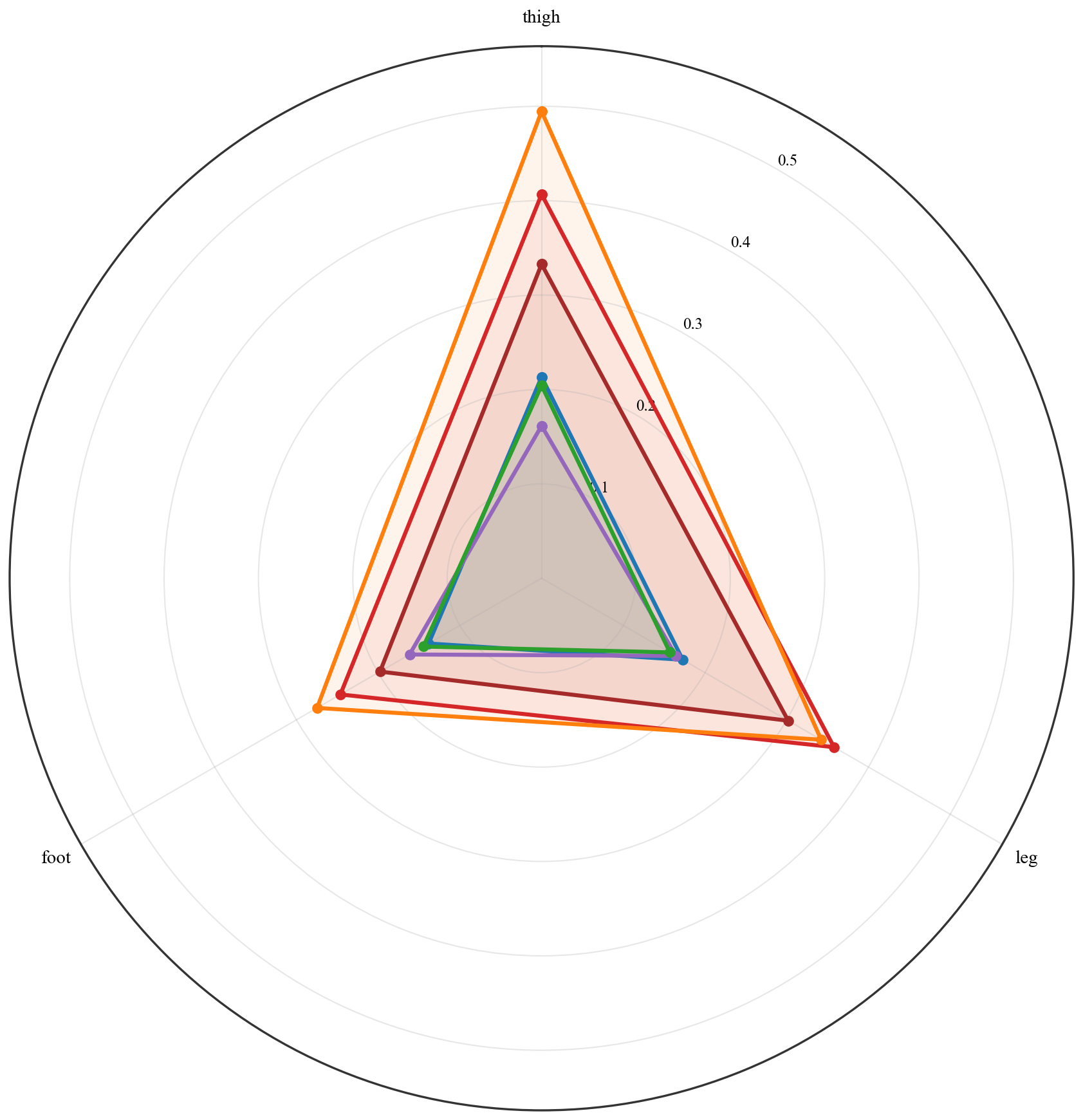}
			\caption{Hopper-TD3-W}
		\end{subfigure}
		
		\vspace{0.2em}
		\begin{subfigure}{0.24\linewidth}
			\includegraphics[width=\linewidth]{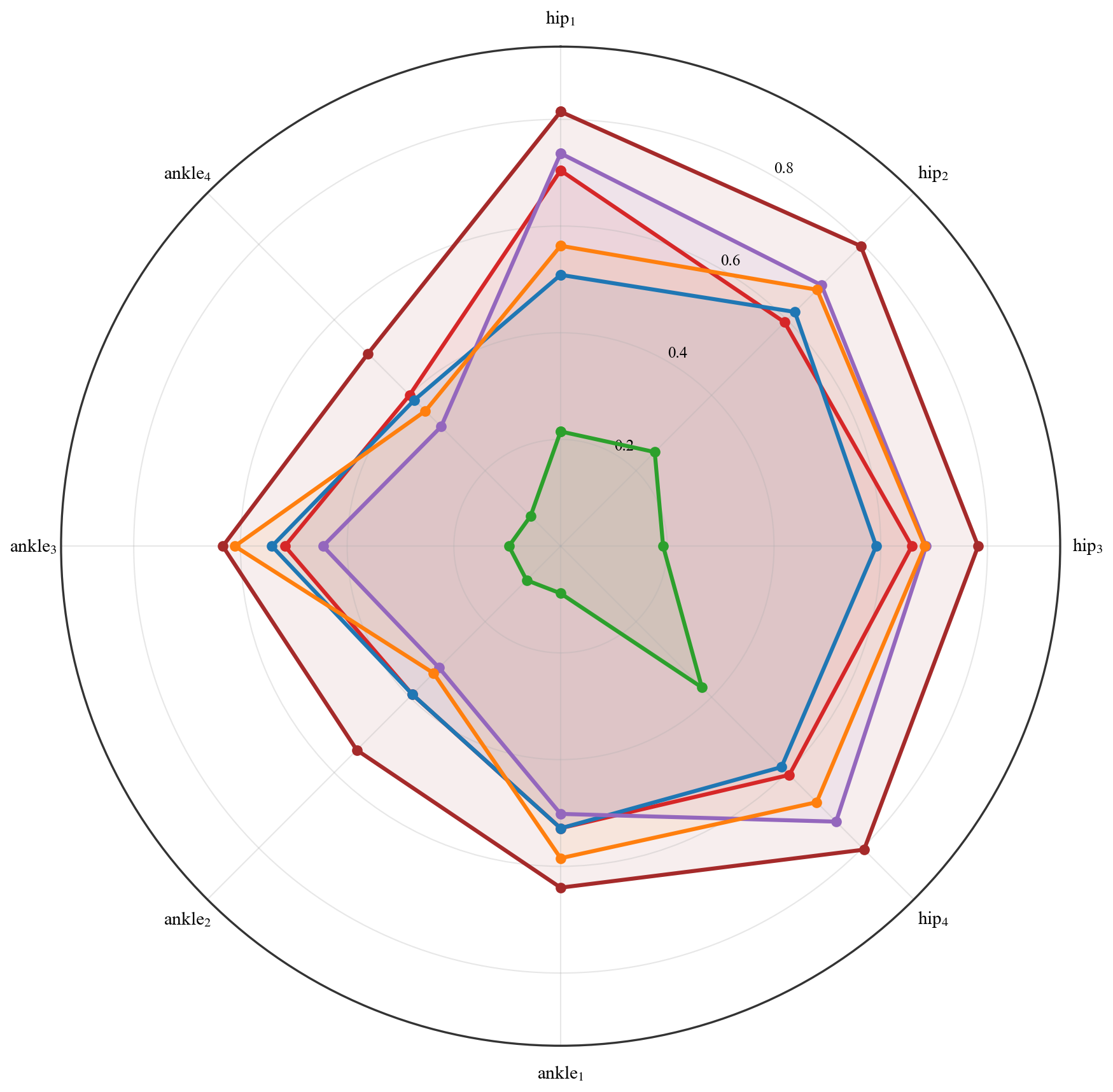}
			\caption{Ant-TD3-S}
		\end{subfigure}
		\hfill
		\begin{subfigure}{0.24\linewidth}
			\includegraphics[width=\linewidth]{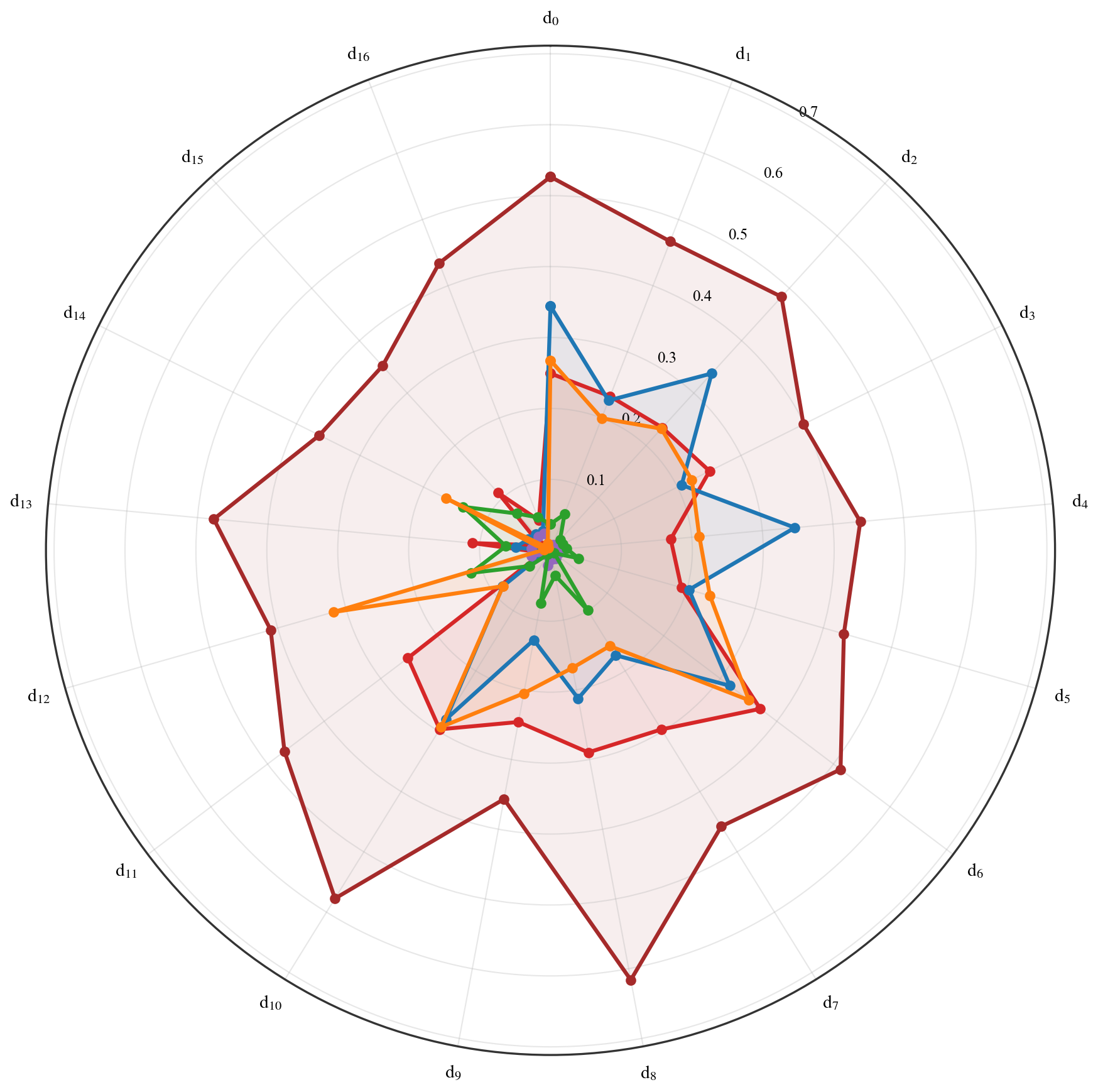}
			\caption{Humanoid-TD3-S}
		\end{subfigure}
		\vspace{0.5em}
		\begin{subfigure}{0.24\linewidth}
			\includegraphics[width=\linewidth]{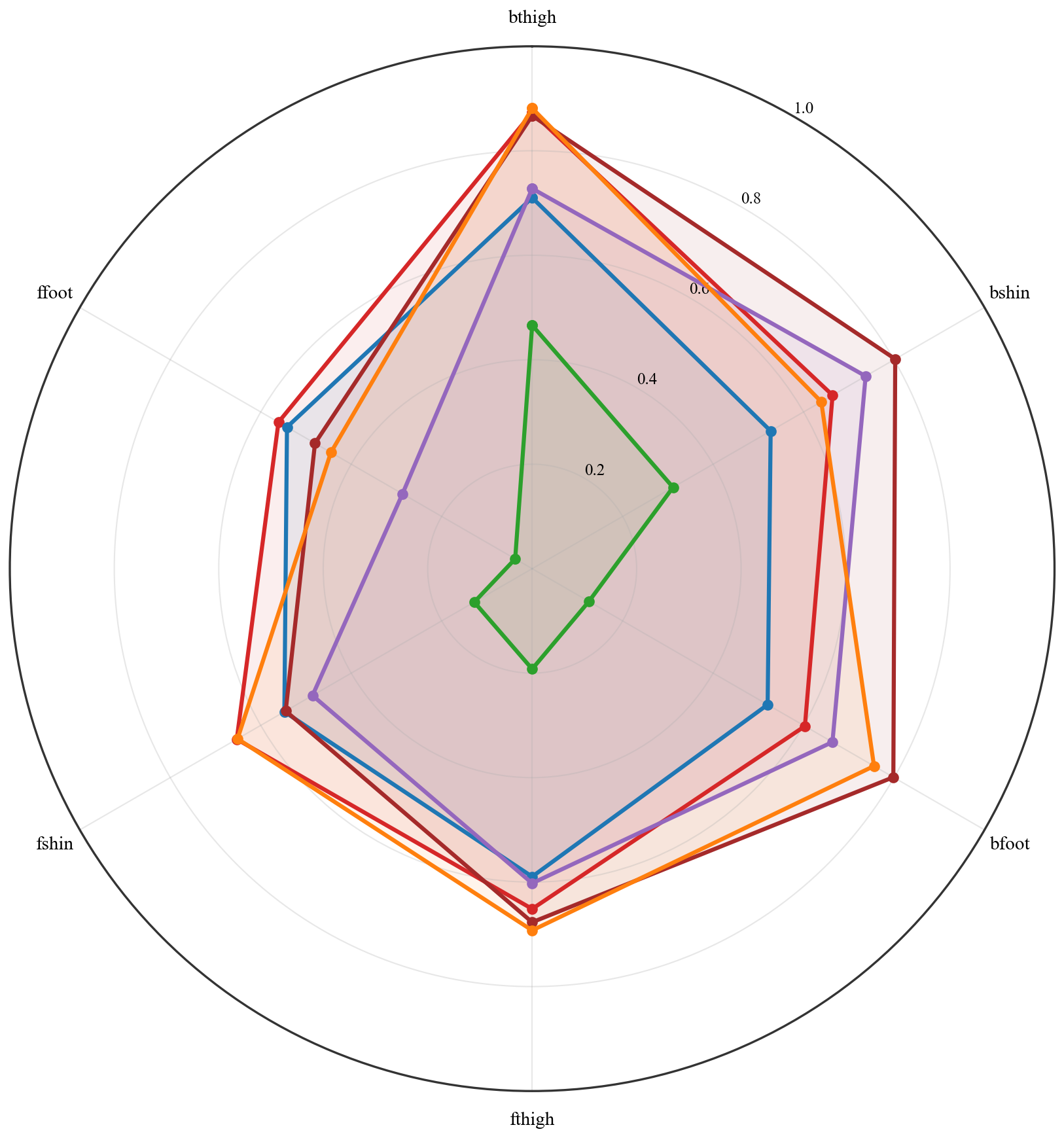}
			\caption{HalfCheetah-TD3-S}
		\end{subfigure}
		\hfill
		\begin{subfigure}{0.24\linewidth}
			\includegraphics[width=\linewidth]{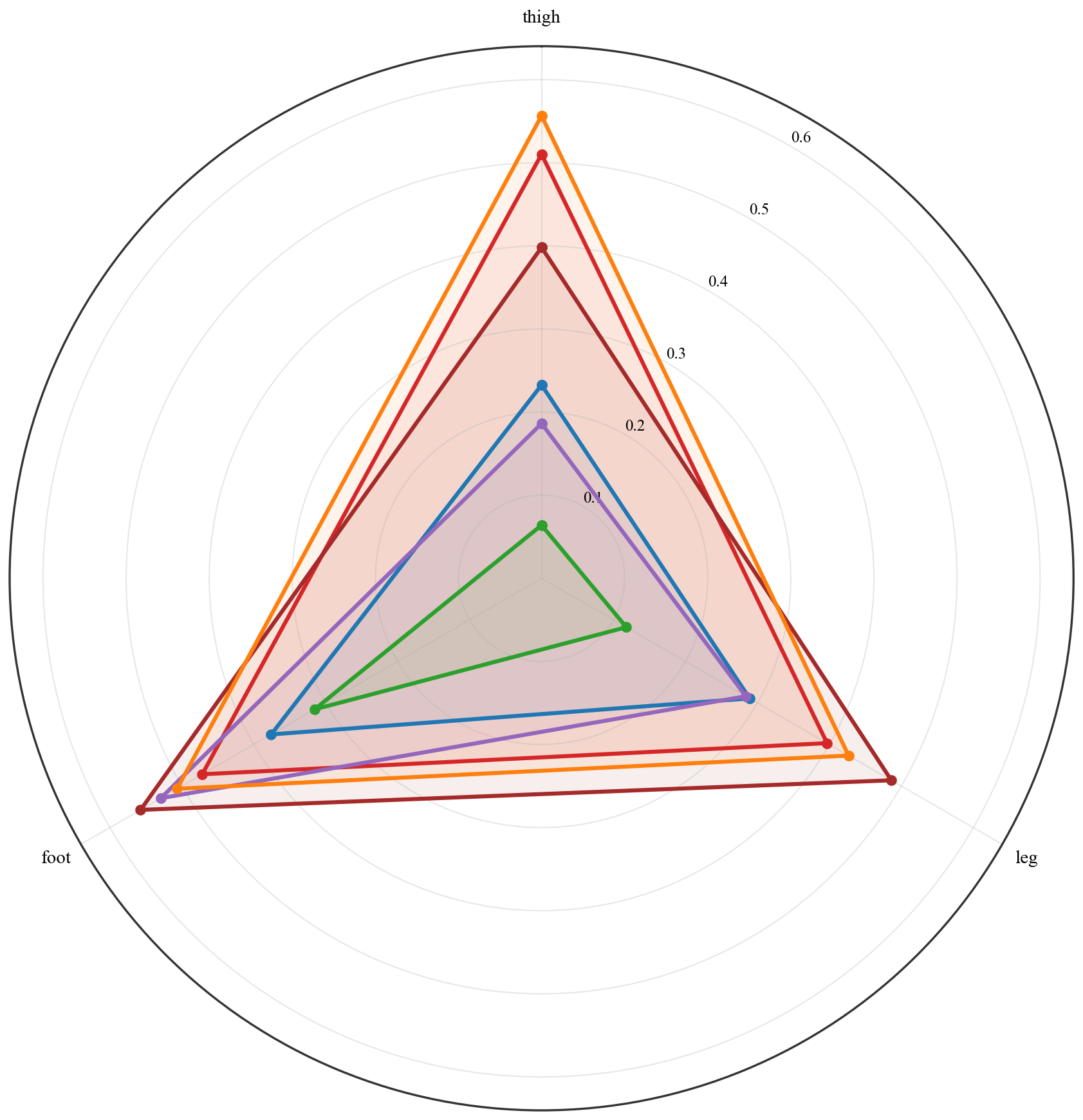}
			\caption{Hopper-TD3-S}
		\end{subfigure}
		
		\vspace{0.15cm}
		\centering
		\begin{tabular}{@{}l@{\hspace{1.5em}}l@{\hspace{1.5em}}l@{\hspace{1.5em}}l@{\hspace{1.5em}}l@{\hspace{1.5em}}l@{}}
			\colorbox{BoxPre}{\rule{0pt}{1pt}\rule{8pt}{0pt}} \raisebox{-2.0pt}{\scriptsize BoxPre+} &
			\colorbox{QP}{\rule{0pt}{1pt}\rule{8pt}{0pt}} \raisebox{-2.0pt}{\scriptsize QP} &
			\colorbox{SRadQP}{\rule{0pt}{1pt}\rule{8pt}{0pt}} \raisebox{-2.0pt}{\scriptsize SRad-QP} &
			\colorbox{SRadStrict}{\rule{0pt}{1pt}\rule{8pt}{0pt}} \raisebox{-2.0pt}{\scriptsize SRad-Strict} &
			\colorbox{dtanh}{\rule{0pt}{1pt}\rule{8pt}{0pt}} \raisebox{-2.0pt}{\scriptsize D-Tanh} &
			\colorbox{DDSRad}{\rule{0pt}{1pt}\rule{8pt}{0pt}} \raisebox{-2.0pt}{\scriptsize DD-SRad}
		\end{tabular}
		
		\caption{Per-dimension constraint utilization radar charts of all methods on four environments under tight heterogeneous constraints (SAC/TD3 backbones).}
		\label{fig:radar_all}
	\end{figure*}
	
	\section{Per-Dimension Utilization Radar Charts Under Heterogeneous Constraints}
	\label{app:hetero_radar}
	
	Figure~\ref{fig:hetero_radar_all} presents the per-dimension constraint utilization radar charts of DD-SRad under wide homogeneous (blue) and tight heterogeneous (red) constraints, for all four environments with SAC and TD3 backbones. Each axis corresponds to one action dimension, with radial value indicating $\mathbb{E}[|\Delta a^i|]/\delta^i$. The layered polygon structure under tight constraints---where each dimension independently approaches its own $\delta^i$ boundary rather than being uniformly truncated to $\min_i\delta^i$---provides visual corroboration of the $\ell_\infty$-tight coverage property established in Theorem~\ref{thm:exploration} and the quantitative Gap analysis of \S\ref{sec:hetero_test}.
	
	\begin{figure}[htbp]
		\centering
		\begin{subfigure}{0.24\linewidth}
			\includegraphics[width=\linewidth]{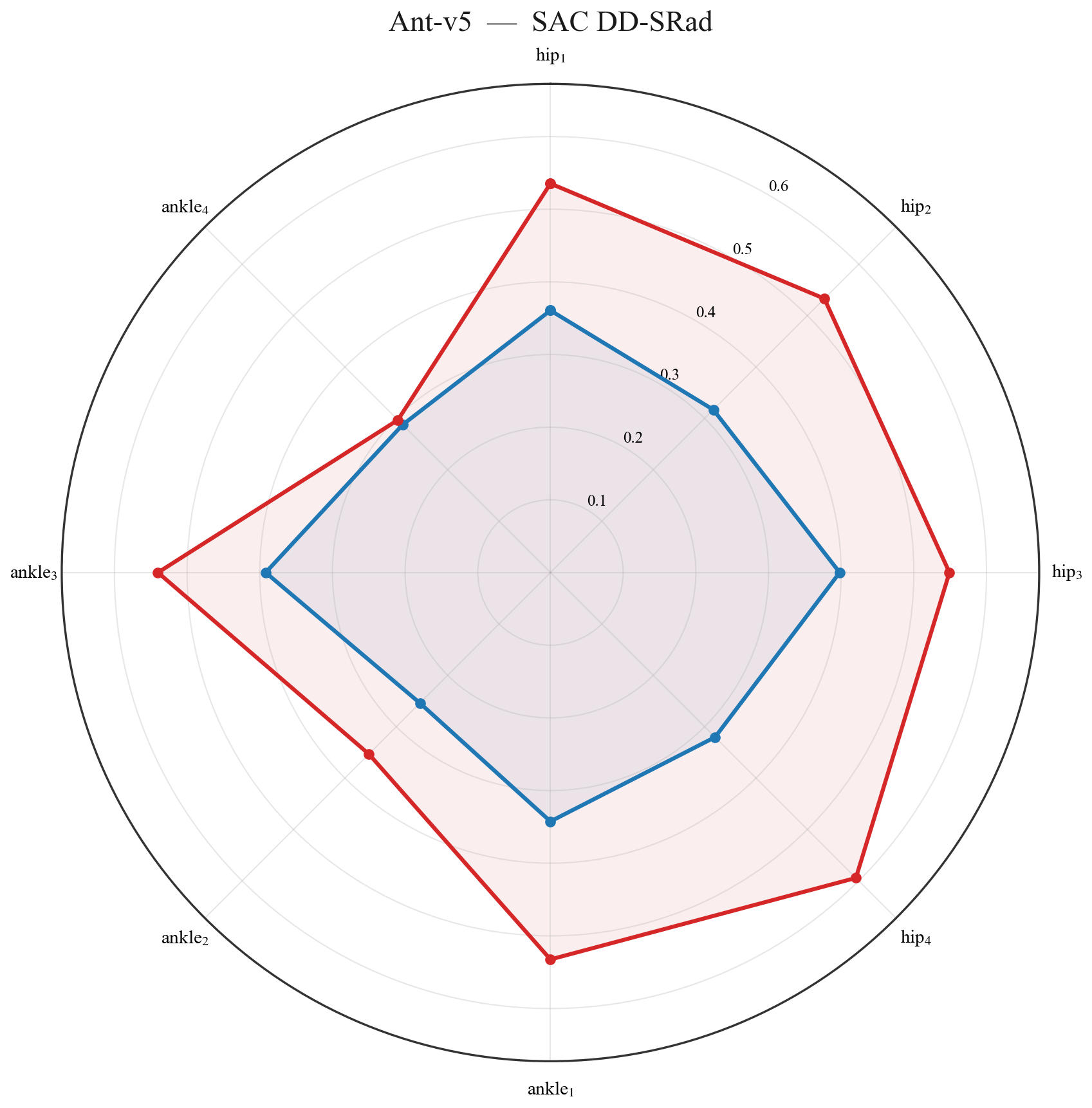}
			\caption{Ant-SAC}
		\end{subfigure}
		\hfill
		\begin{subfigure}{0.24\linewidth}
			\includegraphics[width=\linewidth]{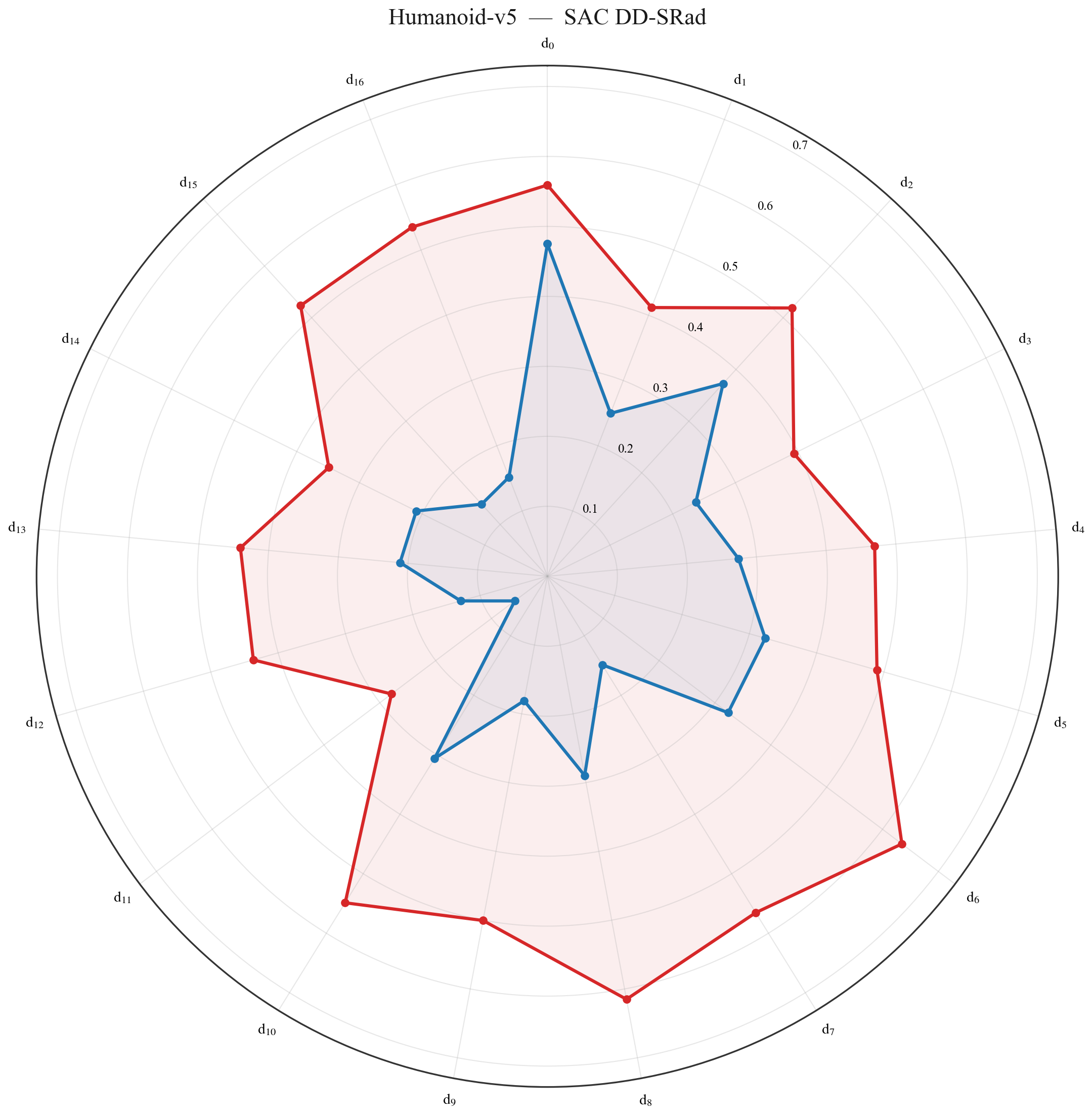}
			\caption{Humanoid-SAC}
		\end{subfigure}
		\hfill
		\begin{subfigure}{0.22\linewidth}
			\includegraphics[width=\linewidth]{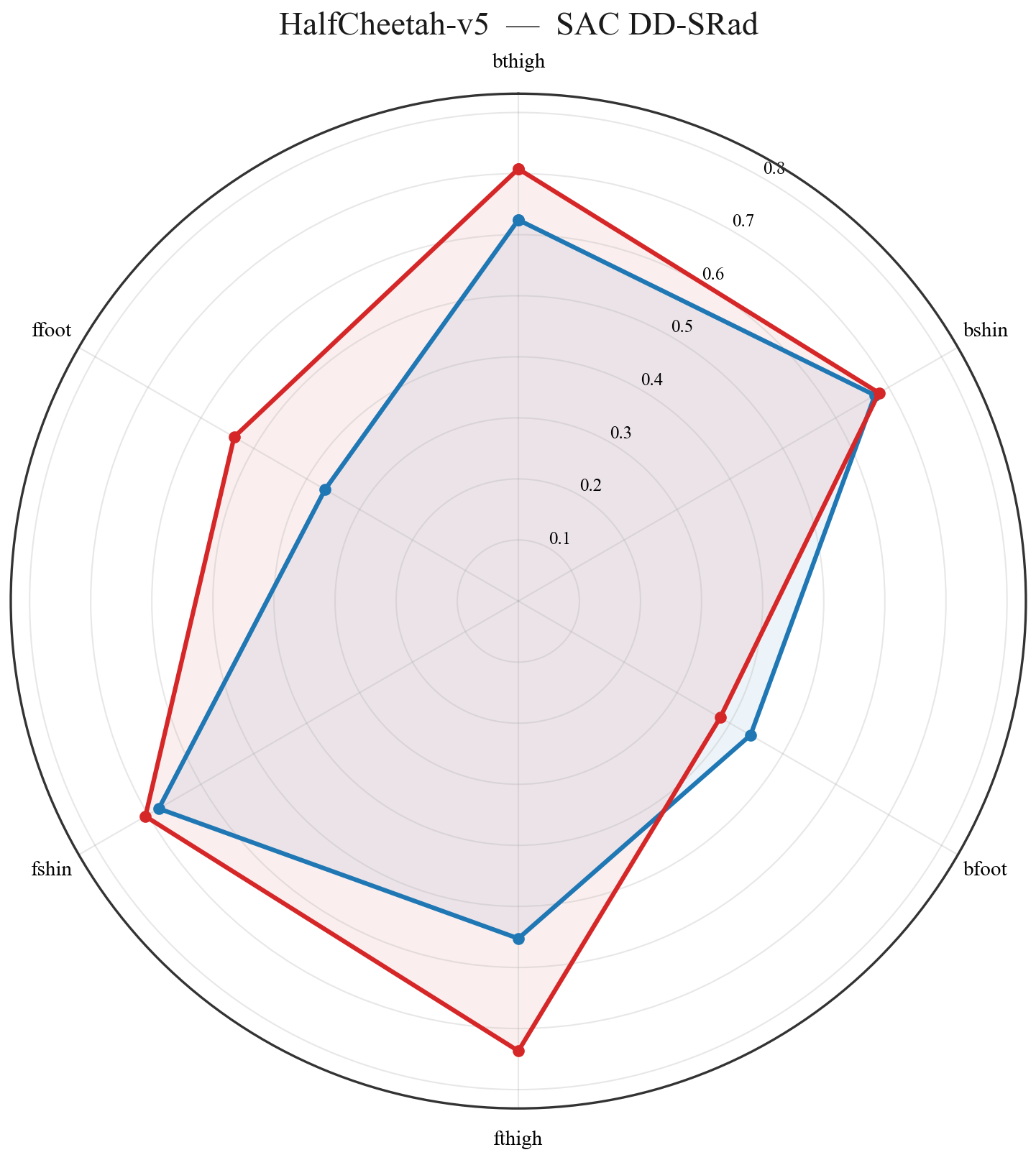}
			\caption{HalfCheetah-SAC}
		\end{subfigure}
		\hfill
		\begin{subfigure}{0.22\linewidth}
			\includegraphics[width=\linewidth]{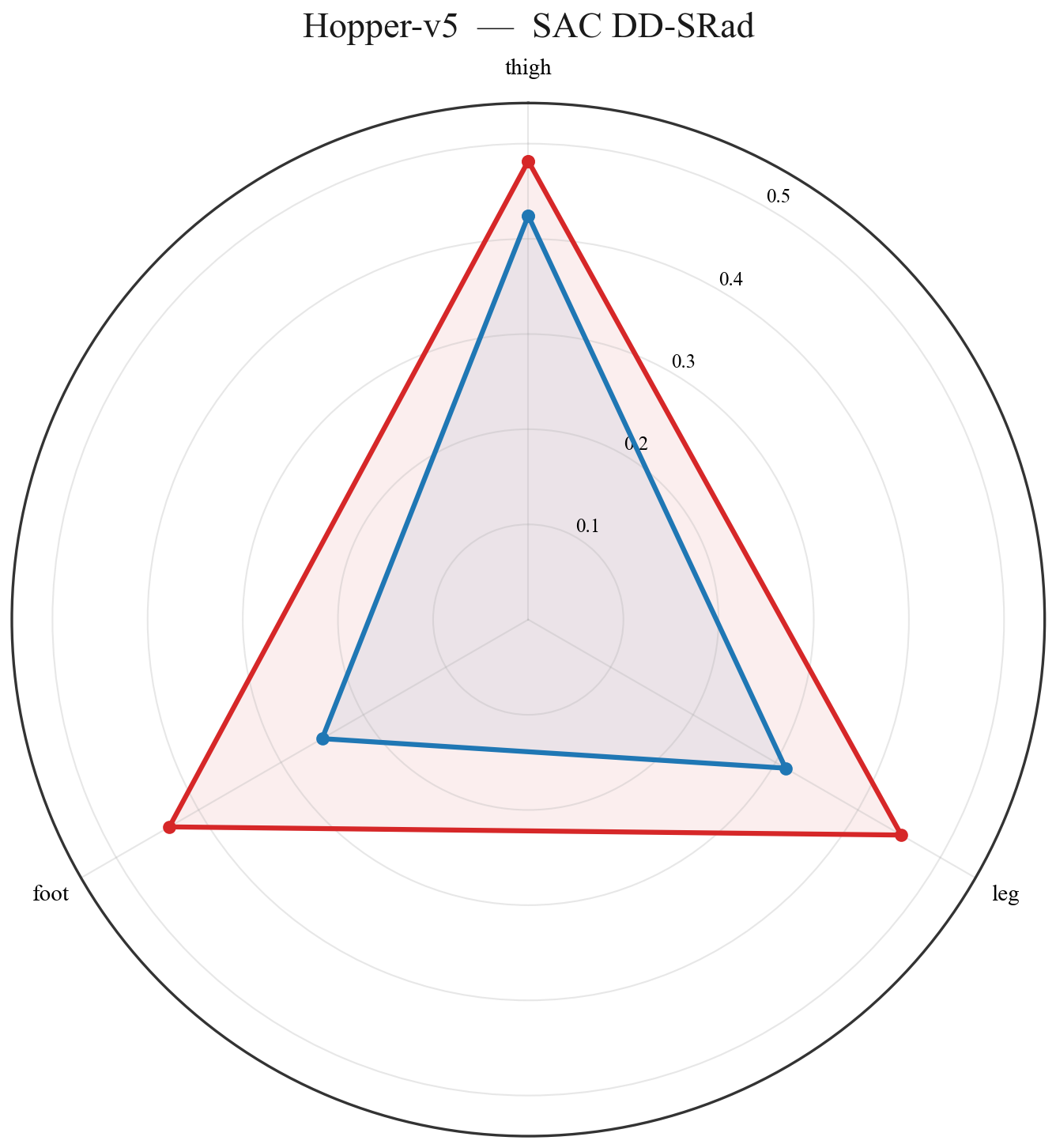}
			\caption{Hopper-SAC}
		\end{subfigure}
		
		\begin{subfigure}{0.24\linewidth}
			\includegraphics[width=\linewidth]{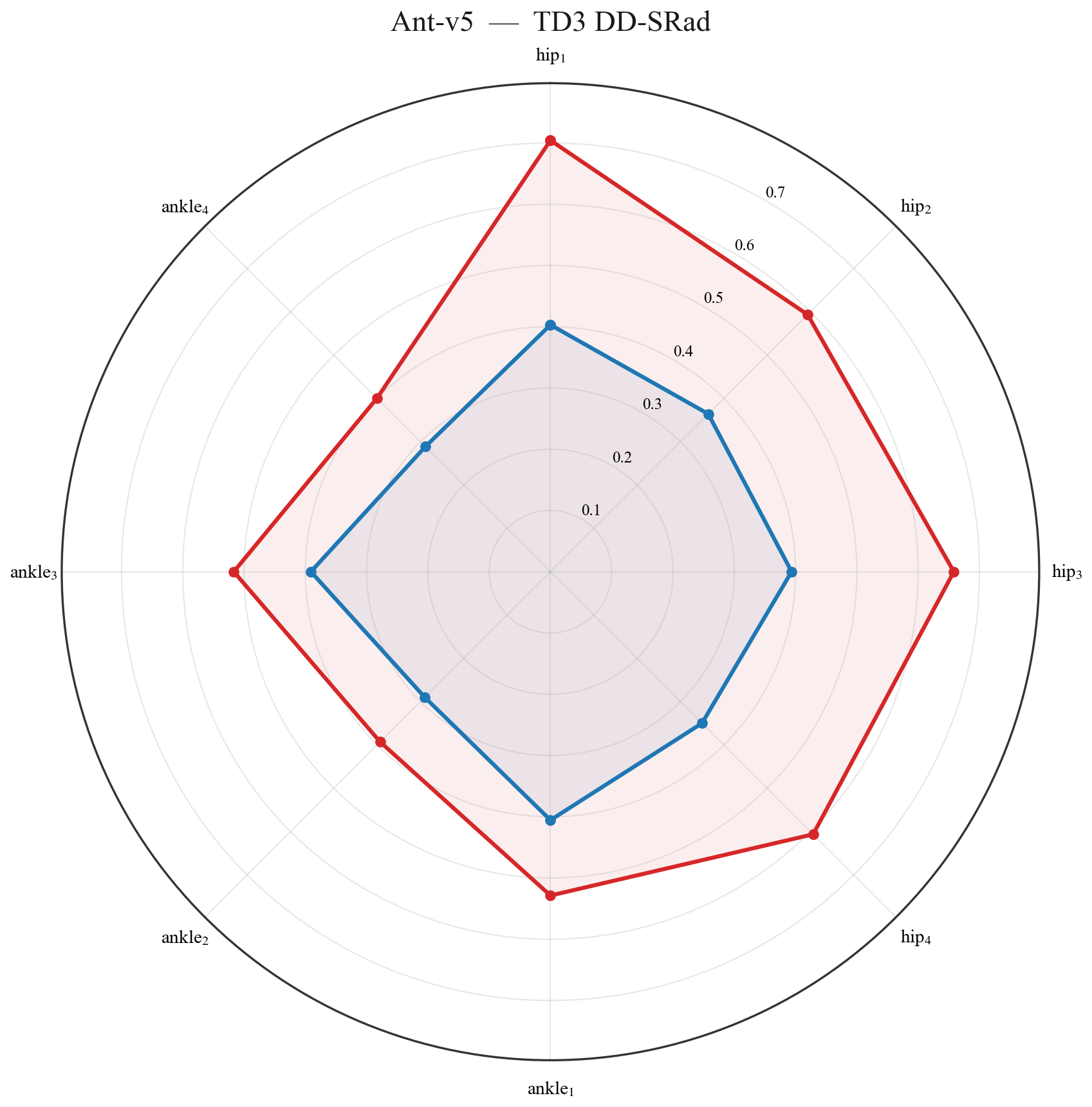}
			\caption{Ant-TD3}
		\end{subfigure}
		\hfill
		\begin{subfigure}{0.24\linewidth}
			\includegraphics[width=\linewidth]{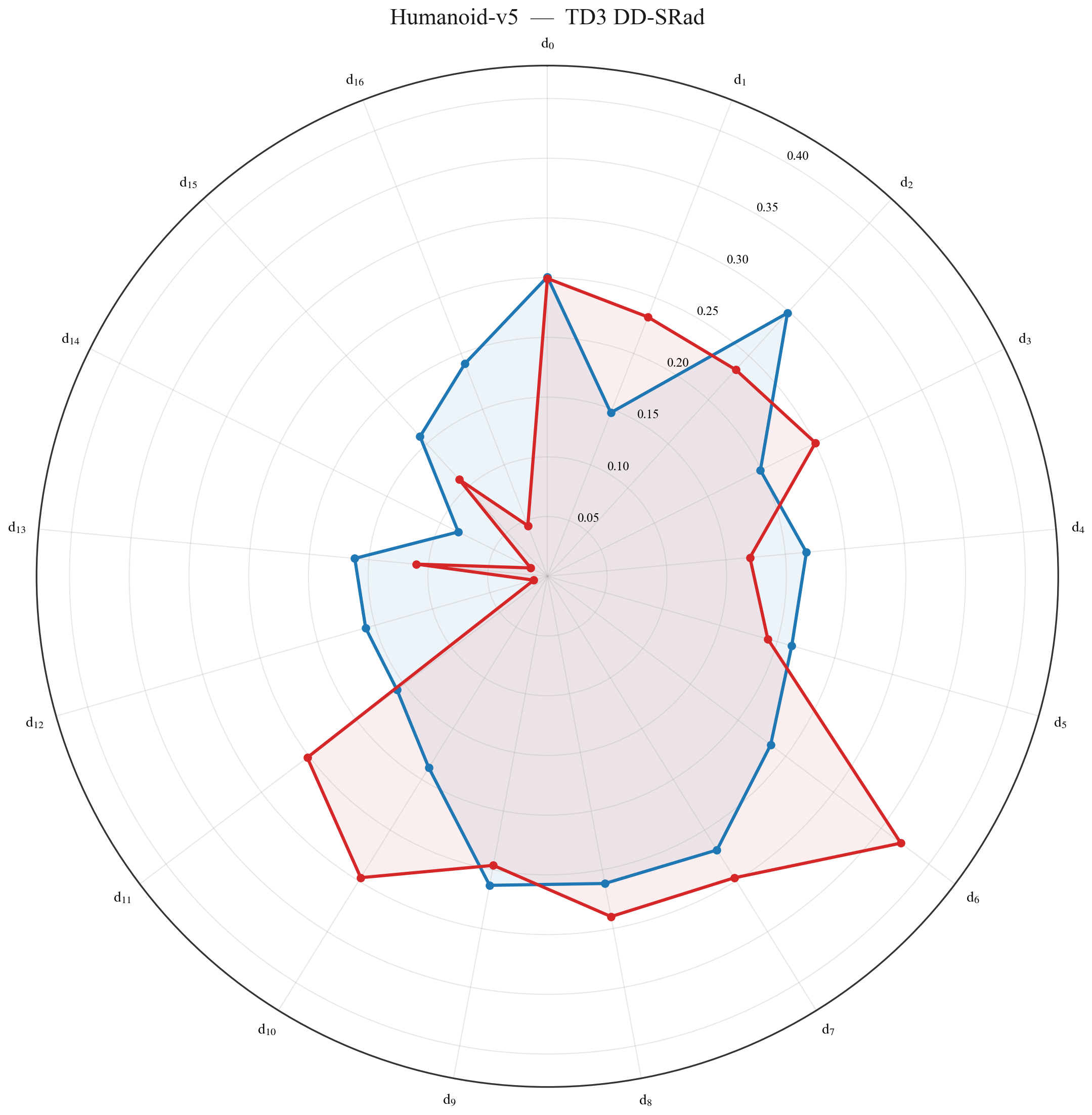}
			\caption{Humanoid-TD3}
		\end{subfigure}
		\hfill
		\begin{subfigure}{0.22\linewidth}
			\includegraphics[width=\linewidth]{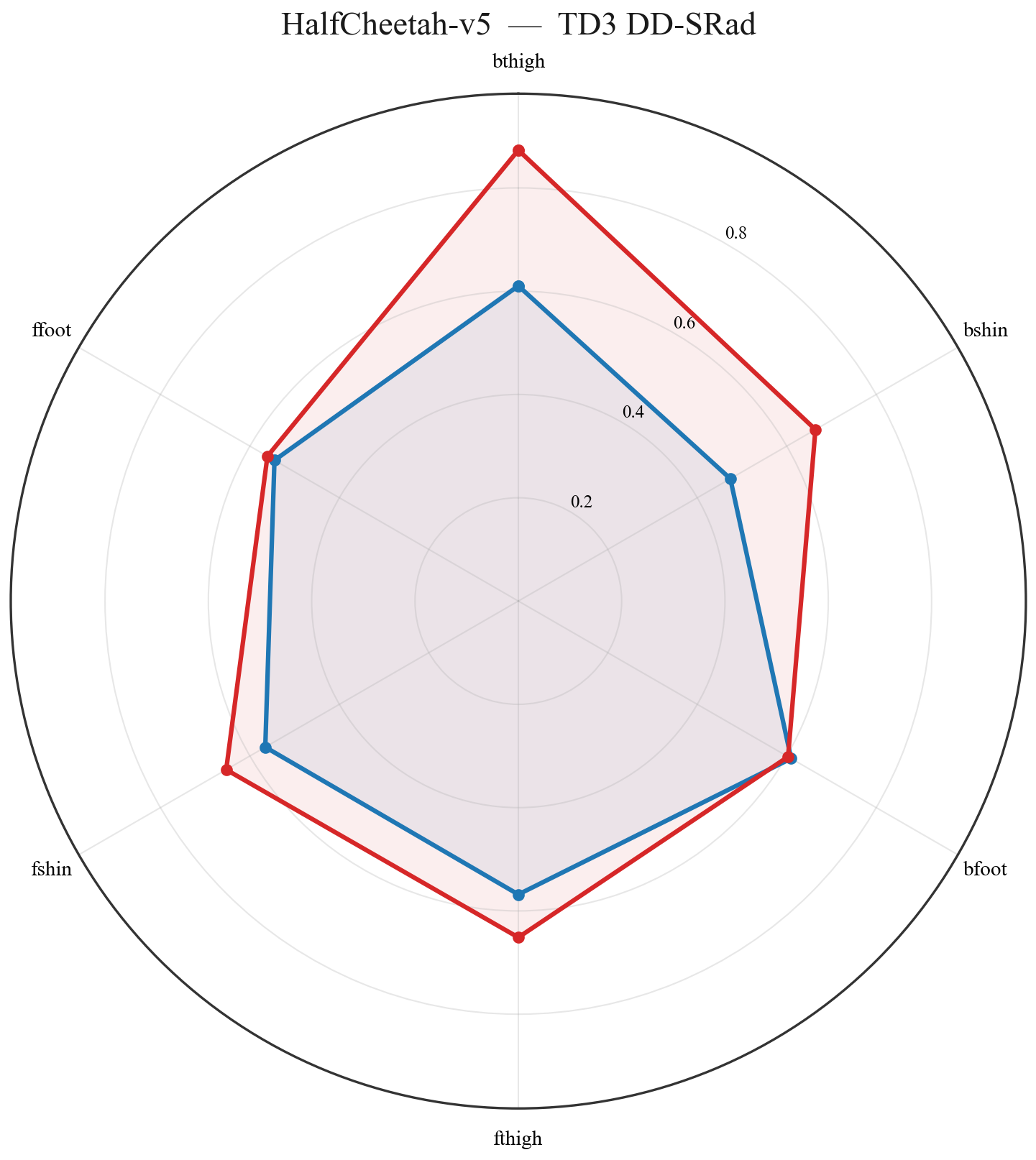}
			\caption{HalfCheetah-TD3}
		\end{subfigure}
		\hfill
		\begin{subfigure}{0.22\linewidth}
			\includegraphics[width=\linewidth]{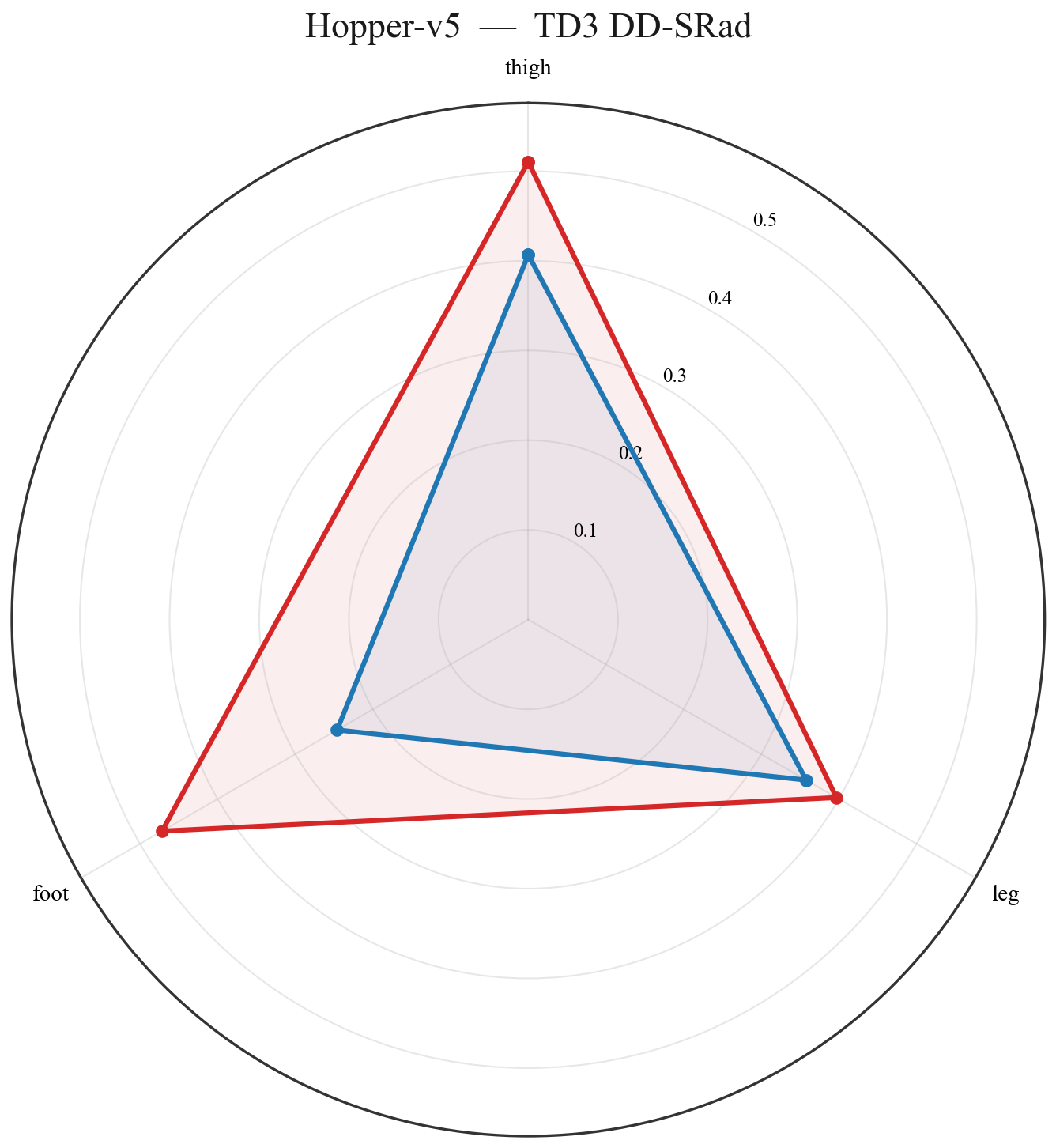}
			\caption{Hopper-TD3}
		\end{subfigure}
		\caption{Per-dimension constraint utilization radar charts of DD-SRad under wide homogeneous (blue) and tight heterogeneous (red) constraints with SAC/TD3 backbones.}
		\label{fig:hetero_radar_all}
	\end{figure}
	
	\section{Simulation Reward Settings}
	\label{app:reward}
	
	Both simulation robots adopt custom reward functions that completely replace the IsaacLab default reward terms. Table~\ref{tab:reward_functions} presents the main components of the reward $r_t = \sum_k w_k \cdot \phi_k$, where the curriculum coefficient is $c_f = \min(\bar{L}/L^*,\, 1)$, $\bar{L}$ is the exponential moving average episode length, and $L^*$ is the target episode length.
	
	\renewcommand{\arraystretch}{1.15}
	\begin{table}[htbp]
		\centering\footnotesize
		\caption{Reward function definitions for simulation robots (H1 Rough \& G1 Flat)}
		\label{tab:reward_functions}
		
		\begin{subtable}{\linewidth}
			\centering
			\caption{Unitree H1 Rough}
			\begin{tabular*}{\linewidth}{@{\extracolsep{\fill}}llc}
				\toprule
				Reward term & Formula & Weight \\
				\midrule
				Linear velocity tracking   & $\exp\!\bigl(-\|(v_{xy}-v^{\text{cmd}}_{xy})\|^2/0.25\bigr)$ & 1.5 \\
				Angular velocity tracking   & $\exp\!\bigl(-(w_z-w^{\text{cmd}}_z)^2/0.25\bigr)$           & 1.0 \\
				Survival bonus     & constant                                                           & 0.02 \\
				Orientation penalty     & $-\|g_{xy}\|^2 \cdot c_f$                                     & 3.0 \\
				Joint default position & $-\|\theta-\theta_{\text{default}}\|^2 \cdot c_f$             & 0.05 \\
				\bottomrule
			\end{tabular*}
		\end{subtable}
		
		\vspace{1em}
		
		\begin{subtable}{\linewidth}
			\centering
			\caption{Unitree G1 Flat (HolosomaRewardComputerV3)}
			\begin{tabular*}{\linewidth}{@{\extracolsep{\fill}}llc}
				\toprule
				Reward term & Formula & Weight \\
				\midrule
				Linear velocity tracking           & $\exp\!\bigl(-\|(v_{xy}-v^{\text{cmd}}_{xy})\|^2/0.25\bigr)$                                         & 1.5 \\
				Angular velocity tracking           & $\exp\!\bigl(-(w_z-w^{\text{cmd}}_z)^2/0.25\bigr)$                                                   & 1.0 \\
				Survival bonus             & constant                                                                                                   & 0.02 \\
				Orientation penalty             & $-\|g_{xy}\|^2 \cdot \max(c_f,0.3)$                                                                  & 3.0 \\
				Joint default position (non-knee) & $-\|\theta_{\neg k}-\theta^{\neg k}_{\text{def}}\|^2 \cdot c_f$                                       & 0.05 \\
				Joint default position (knee)   & $-\|\theta_k-\theta^k_{\text{def}}\|^2 \cdot c_f$                                                    & $0.05{\times}0.5$ \\
				Foot height             & $\exp\!\bigl(-\sum(h_f - h^{\text{swing}}_f)^2/(2\sigma_h^2)\bigr)$                                  & 1.0 \\
				Knee flexion           & $\exp\!\bigl(-(\theta_k - \theta^k_{\text{def}} - \kappa_{\text{flex}}\sin\phi_k)^2/(2\sigma_k^2)\bigr)$ & 1.5 \\
				Minimum knee angle             & $-\operatorname{clamp}(\theta_{\min} - \theta_k,\,0,\,\infty)^2 \cdot c_f$                           & 2.0 \\
				Stride length             & $\exp\!\bigl(-\sum(x^b_f - L\sin\phi_f)^2/(2\sigma_s^2)\bigr)$                                       & 1.5 \\
				Foot parallelism             & $-(\psi_{f,L}-\psi_{f,R})^2 \cdot c_f$                                                               & 1.0 \\
				Foot clearance             & $-\operatorname{clamp}(d_{\min}-|y_{f,L}-y_{f,R}|,\,0,\,\infty)^2 \cdot c_f$                        & 4.0 \\
				Anti-crossing               & $-\operatorname{clamp}(y_{f,R}-y_{f,L},\,0,\,\infty)^2 \cdot c_f$                                    & 5.0 \\
				Hip roll           & $-\|\theta_{\text{hip\_roll}}-\theta^{\text{hr}}_{\text{def}}\|^2 \cdot c_f$                          & 3.0 \\
				Hip yaw               & $-\|\theta_{\text{hip\_yaw}}-\theta^{\text{hy}}_{\text{def}}\|^2 \cdot c_f$                           & 3.0 \\
				Torso joint               & $-\|\theta_{\text{torso}}-\theta^{\text{torso}}_{\text{def}}\|^2 \cdot \max(c_f,0.3)$                & 2.0 \\
				Ankle pitch               & $-\|\theta_{\text{ap}}-\theta^{\text{ap}}_{\text{def}}\|^2 \cdot c_f$                                & 2.0 \\
				Foot slip             & $-\sum_f\|\dot{p}^{xy}_f\| \cdot \mathbf{1}[\|F_f\|>1\,\text{N}]$                                   & 0.5 \\
				\bottomrule
				\multicolumn{3}{l}{\footnotesize $\phi_f$: gait phase; $L$: stride length target; $d_{\min}$: minimum foot clearance; $\sigma$: Gaussian widths for each term.} \\
			\end{tabular*}
		\end{subtable}
		
	\end{table}
	\renewcommand{\arraystretch}{1}
	
	\section{Reachable Set Scatter  of H1 and G1}
	\label{app:sim_scatter}
	
	Figure~\ref{fig:sim_scatter} presents the empirical reachable sets of each method in the $(\Delta a^{L.\text{hip\_pitch}},\,\Delta a^{L.\text{ankle\_roll}})$ plane for H1 and G1, providing execution-level geometric verification complementary to the Gap metric in Table~\ref{tab:sim_full}(b).
	
	\begin{figure}[htbp]
		\centering
		
		\begin{subfigure}{0.99\linewidth}
			\includegraphics[width=\linewidth]{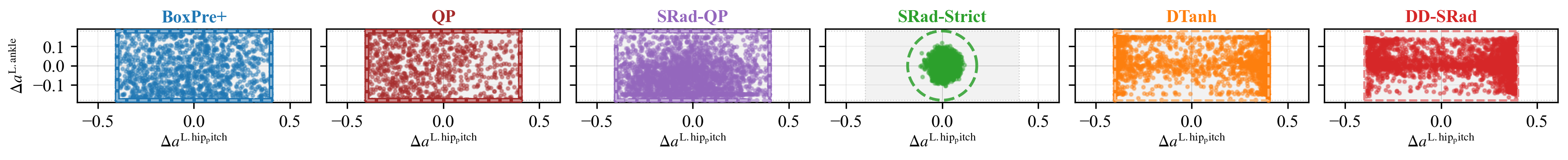}
			\caption{H1 Empirical Reachable Sets}
		\end{subfigure}
		
		\vspace{1em}
		
		\begin{subfigure}{0.99\linewidth}
			\includegraphics[width=\linewidth]{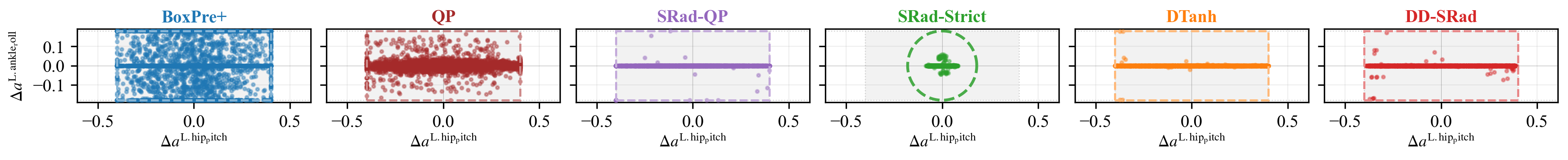}
			\caption{G1 Empirical Reachable Sets}
		\end{subfigure}
		
		\caption{Empirical Reachable Sets}
		\label{fig:sim_scatter}
	\end{figure}
	
	On H1, DD-SRad scatters uniformly within the constraint rectangle $[-\delta^\mathrm{hip},\delta^\mathrm{hip}]\times[-\delta^\mathrm{ankle},\delta^\mathrm{ankle}]$, with both dimensional rate budgets independently and fully utilized, consistent with the balanced utilization of Util-P$=0.38$, Util-D$=0.48$; D-Tanh scatter morphology closely resembles DD-SRad, reflecting that the shared $R_{\text{eff}}^i$ mechanism produces similar coverage structures under the same geometric constraints; SRad-Strict scatter concentrates within the disk of radius $R=\min_i\delta^i$, with the high-budget dimension (hip\_pitch) reachable set uniformly compressed to the tightest constraint magnitude, directly corresponding to the utilization collapse Util-P$=0.08$ and Fall$=100\%$.
	
	On G1, DD-SRad scatter distributes broadly along the $\Delta a^{L.\text{hip\_pitch}}$ axis and is highly concentrated near zero along $\Delta a^{L.\text{ankle\_roll}}$, forming a near-horizontal band structure---a direct manifestation of the position-adaptive $R_{\text{eff}}^i$ mechanism interacting with the flat terrain task structure: hip\_pitch joints bear gait propulsion (large $\delta$, fully utilized), while ankle\_roll joints have extremely low lateral stability demands under flat terrain's low perturbation, and the policy converges to near-zero ankle\_roll through independent per-dimension optimization. D-Tanh exhibits the same morphology, confirming that this solution structure is jointly determined by $\ell_\infty$ geometry and task rewards, independent of the specific squashing function; the task-metric gap (VTE: $0.138$ vs $0.207$, Return: $5473$ vs $5145$) reflects the effective gradient region advantage of spherical radial squashing over $\tanh$ (Proposition~\ref{prop:gradient}). BoxPre+ (Gap$=0.01$) forces uniform coverage---a geometric side effect of $\ell_\infty$ clip projection---preventing task-aligned allocation and directly manifesting as VTE$=0.224$ ($+62\%$ vs DD-SRad). SAC-Post(QP) covers the rectangle interior with reduced boundary density, consistent with QP's conservative margin at feasible-set vertices (Return$=4954$, $\approx10\%$ below DD-SRad). SRad-Strict concentrates within the $\ell_2$ disk of radius $R=\min_i\delta^i$, directly corresponding to Util-P$=0.08$ and Return$=1312$.
	
	The scatter morphology is fully consistent with the Gap metric in Table~\ref{tab:sim_full}(b), forming cross-platform verification from MuJoCo to IsaacLab: $\ell_\infty$-exact parameterization guarantees geometric coverage and enables task-gradient-guided dimensional allocation---balancing rate budgets under high-perturbation rough terrain (H1) and focusing on propulsion joints for precise velocity tracking on flat terrain (G1)---an adaptive capability neither BoxPre+'s clip nor SRad-Strict's global radius can achieve.
	
	\section{Per-Dimension Action Increment Probability Density of H1 and G1}
	\label{app:sim_action_distribution}
	
	To visualize the empirical distribution of executed action increments, the probability density of $\Delta a^i$ for each controlled dimension is plotted. Given $N$ observed samples $\{\Delta a_k^i\}_{k=1}^N$, the density is estimated via a normalized histogram with $B$ uniform bins over $[-\delta^i,\,\delta^i]$:
	\begin{equation}
		\hat{p}(x) \;=\; \frac{c_b}{N \cdot w}, \qquad x \in \text{bin } b,
		\label{eq:utilization}
	\end{equation}
	where $c_b$ is the count of samples in bin $b$ and $w = 2\delta^i / B$ is the bin width. This satisfies $\int_{-\delta^i}^{\delta^i} \hat{p}(x)\,dx = 1$, so $\hat{p}$ is a valid probability density function. Values of $\hat{p} \gg 1$ are physically meaningful and indicate that samples concentrate in a narrow sub-interval of $[-\delta^i,\delta^i]$.
	
	\begin{figure}[htbp]
		\centering
		
		\begin{subfigure}{0.99\linewidth}
			\includegraphics[width=\linewidth]{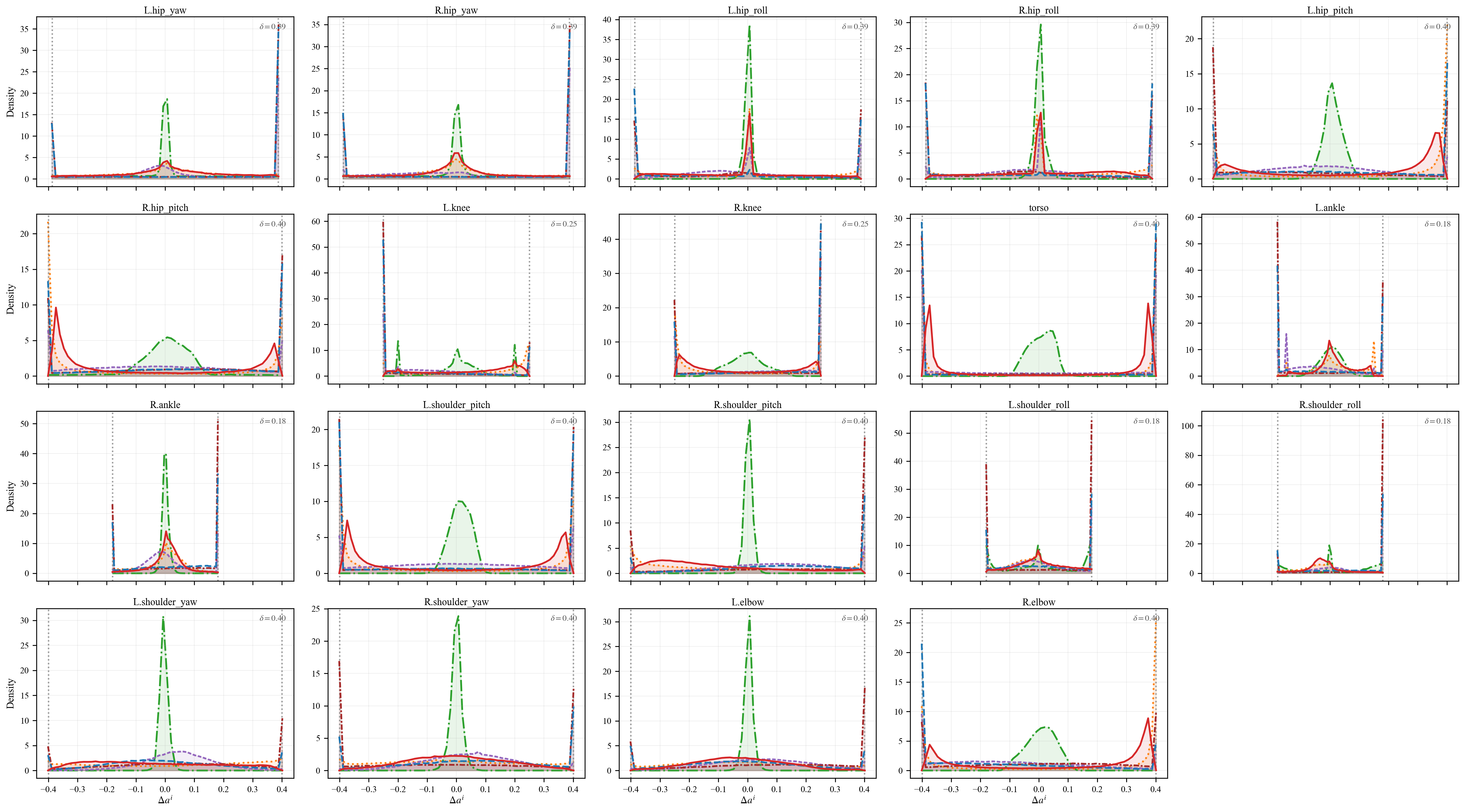}
			\caption{H1 per-dimension action distribution}
		\end{subfigure}
		
		\vspace{1em}
		
		\begin{subfigure}{0.99\linewidth}
			\includegraphics[width=\linewidth]{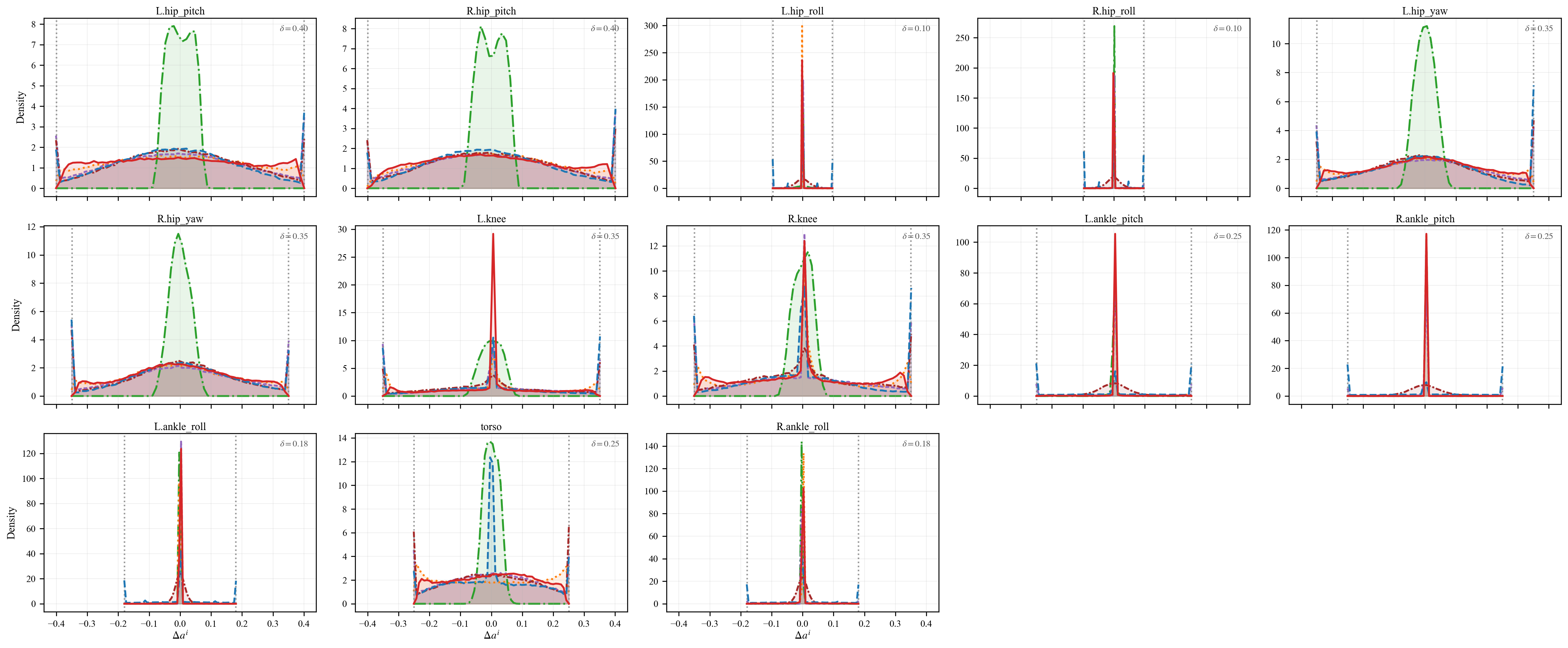}
			\caption{G1 per-dimension action distribution}
		\end{subfigure}
		
		\vspace{0.0cm}
		\centering
		\begin{tabular}{@{}l@{\hspace{1.5em}}l@{\hspace{1.5em}}l@{\hspace{1.5em}}l@{\hspace{1.5em}}l@{\hspace{1.5em}}l@{}}
			\colorbox{BoxPre}{\rule{0pt}{1pt}\rule{8pt}{0pt}} \raisebox{-2.0pt}{\scriptsize BoxPre+} &
			\colorbox{QP}{\rule{0pt}{1pt}\rule{8pt}{0pt}} \raisebox{-2.0pt}{\scriptsize QP} &
			\colorbox{SRadQP}{\rule{0pt}{1pt}\rule{8pt}{0pt}} \raisebox{-2.0pt}{\scriptsize SRad-QP} &
			\colorbox{SRadStrict}{\rule{0pt}{1pt}\rule{8pt}{0pt}} \raisebox{-2.0pt}{\scriptsize SRad-Strict} &
			\colorbox{dtanh}{\rule{0pt}{1pt}\rule{8pt}{0pt}} \raisebox{-2.0pt}{\scriptsize D-Tanh} &
			\colorbox{DDSRad}{\rule{0pt}{1pt}\rule{8pt}{0pt}} \raisebox{-2.0pt}{\scriptsize DD-SRad}
		\end{tabular}
		
		\caption{H1 and G1 per-dimension action distribution}
		\label{fig:sim_distribution}
	\end{figure}
	
	Figure~\ref{fig:sim_distribution} presents per-dimension probability density $\hat{p}(\Delta a^i)$ for each method on H1 and G1, forming a three-layer complementary verification alongside the mean utilization metric of the radar charts and the geometric morphology of the scatter plots. Density values far exceeding 1 indicate that action increments are highly concentrated in a narrow sub-interval of the constraint range, which is the probability density manifestation of the policy actively choosing a conservative solution in a specific dimension.
	
	\paragraph{H1 (Rough Terrain, 10 Dimensions).}
	DD-SRad exhibits a broad unilateral distribution on the hip\_pitch group ($\delta=0.40$) with probability mass clearly shifted toward the positive constraint boundary, reflecting the demand of rough terrain gait propulsion for sustained forward hip joint rate output---the policy fully exploits the rate budget of the large-$\delta$ dimensions to resist dynamic perturbations and maintain forward momentum. D-Tanh (orange) shows nearly identical distribution morphology in the same dimensions, consistent with the shared $R_{\text{eff}}^i$ geometry; the systematic return gap ($34.56$ vs $37.14$) arises from the difference in squashing function Jacobian decay rate, not from reachable set geometry. SRad-Strict (green) exhibits a clear bimodal structure in the hip\_pitch dimension with peaks far from the constraint boundary, directly corresponding to the mechanism by which the global $\ell_2$ radius $R=\min_i\delta^i$ uniformly compresses the hip joint large-$\delta$ step size to the ankle joint magnitude, blocking effective gait propulsion; its Fall$=100\%$ is the inevitable consequence of this geometric mismatch on rough terrain. In the ankle dimensions ($\delta=0.18$), DD-SRad shows a moderately spread distribution, indicating that the ankle joint rate budget receives substantive utilization under rough terrain's dynamic balance demands; SRad-Strict ankle dimensions also show extreme concentration, because global radius compression prevents the policy from producing effective recovery actions. BoxPre+ (blue) shows relatively flat distributions across all dimensions, reflecting the uniform projection property of $\ell_\infty$ clip within the feasible set, but this flat distribution corresponds to Return$=23.11$ significantly lower than DD-SRad, validating that uniform activation is not equivalent to task efficiency (§\ref{sec:sim_results}). In the arm and shoulder joint dimensions (shoulder/elbow, $\delta=0.18$--$0.40$), SRad-Strict similarly shows bimodal or concentrated distributions in large-$\delta$ dimensions, while DD-SRad and D-Tanh maintain broad distributions matched to task demands, with the overall pattern consistent with the leg joints.
	
	\paragraph{G1 (Flat Terrain, 13 Dimensions).}
	The G1 distribution plots reveal the task-adaptive structural differences across methods under flat terrain's low-perturbation conditions, forming a multi-perspective verification that closes layer by layer with the scatter plots and radar charts.
	
	In the hip\_pitch dimension ($\delta=0.40$, the main gait propulsion joint), DD-SRad (red) exhibits a broad distribution covering most of the constraint interval with probability mass well spread out, directly corresponding to the high utilization of Util-P$=0.40$. D-Tanh (orange) shows nearly identical distribution morphology to DD-SRad, further confirming that this broad utilization structure is jointly determined by the $R_{\text{eff}}^i$ mechanism and task rewards, not a specific product of the squashing function. SRad-Strict (green) exhibits a bimodal distribution in the hip\_pitch dimension with peaks concentrated near $\pm\min_i\delta^i=\pm0.10$, directly manifesting the mechanism by which the global $\ell_2$ radius uniformly compresses the rate of all dimensions to the tightest hip\_roll constraint magnitude, causing the actual reachable step size of the gait propulsion joint to be only 25\% of DD-SRad's.
	
	In distal joint dimensions such as ankle\_pitch ($\delta=0.25$) and ankle\_roll ($\delta=0.18$), both DD-SRad and D-Tanh exhibit highly concentrated zero-mean spikes ($\hat{p}\gg1$), i.e., probability mass highly concentrated in a narrow band around $\Delta a^i\approx0$. This is the optimal solution to which the policy actively converges through task gradient optimization under flat terrain's low-perturbation conditions: flat ground imposes extremely low demands on ankle joint lateral stabilization and pitch adjustment, so the policy compresses distal joint actions to near zero to reduce ineffective consumption, focusing on hip\_pitch and other large-$\delta$ propulsion joints for precise velocity tracking, ultimately achieving the lowest velocity tracking error of all methods at VTE$=0.138$. The $R_{\text{eff}}^i$ mechanism allows each dimension to independently respond to this task structure---the adaptive tightening of the effective radius near position boundaries for distal joints does not affect the full exploration of the hip\_pitch dimension, with the two completely decoupled at the parameterization level. BoxPre+ also shows some zero-vicinity concentration in ankle dimensions, but with broader spread than DD-SRad, reflecting that $\ell_\infty$ clip projection within the feasible set still forces retention of non-zero activation components in each dimension; the cost is quantified as VTE$=0.224$. In the hip\_roll dimension ($\delta=0.10$, G1's tightest constraint), all methods exhibit extremely high zero-mean spikes, reflecting both the strict constraint limitation and the fact that flat terrain imposes almost no demand on hip lateral roll---consistent with the task characteristics. SRad-Strict shows bimodal distributions in intermediate-$\delta$ dimensions such as torso, consistent with the expectation that the global $\ell_2$ radius produces varying degrees of mismatch across dimensions with different constraint magnitudes; the severe degradation of Return$=1312$ is the cumulative result of geometric mismatch across all dimensions.
	
	\section{Limitation}
	\label{sec:limitation}
	
	DD-SRad's diagonally positive definite Jacobian enables exact gradient propagation through off-policy backbones (SAC, TD3), but is less effective with on-policy algorithms such as PPO: the scalar advantage $A_t$ has substantially higher estimation variance under high-$\kappa$ constraints than the per-dimension $\nabla_a Q$ signal, requiring an action-conditioned advantage network with replay for effective on-policy extension. Additionally, the current design targets independent per-dimension box constraints; cross-dimensional coupled constraints require extensions such as a state-conditioned unified radius or neural $R_{\text{eff}}^i$.
	

{\small
		\begin{thebibliography}{99}
			
			\bibitem{wensing2017proprioceptive}
			P. M. Wensing, A. Wang, S. Seok, D. Otten, J. Lang, and S. Kim, ``Proprioceptive actuator design in the MIT Cheetah: Impact mitigation and high-bandwidth physical interaction for dynamic legged robots,'' \emph{IEEE Transactions on Robotics}, vol. 33, no. 3, pp. 509--522, June 2017.
			
			\bibitem{grandia2023perceptive}
			R. Grandia, F. Jenelten, S. Yang, F. Farshidian, and M. Hutter, ``Perceptive locomotion through nonlinear model-predictive control,'' \emph{IEEE Transactions on Robotics}, vol. 39, no. 5, pp. 3402--3421, Oct. 2023.
			
			\bibitem{bledt2018cheetah}
			G. Bledt, M. J. Powell, B. Katz, J. Di Carlo, P. M. Wensing, and S. Kim, ``MIT Cheetah 3: Design and control of a robust, dynamic quadruped robot,'' in \emph{2018 IEEE/RSJ International Conference on Intelligent Robots and Systems (IROS)}, Madrid, Spain, 2018, pp. 2245--2252.
			
			\bibitem{kim2019highly}
			D. Kim et al., ``Highly dynamic quadruped locomotion via whole-body impulse control and model predictive control,'' \emph{arXiv preprint arXiv:1909.06586}, 2019.
			
			\bibitem{pandala2022robust}
			A. Pandala, R. T. Fawcett, and K. A. Hamed, ``Robust predictive control for quadrupedal locomotion: Learning to close the gap between reduced and full-order models,'' \emph{IEEE Robotics and Automation Letters}, vol. 7, no. 3, pp. 6622--6629, 2022.
			
			\bibitem{williams2017information}
			G. Williams et al., ``Information theoretic MPC for model-based reinforcement learning,'' in \emph{2017 IEEE International Conference on Robotics and Automation (ICRA)}, Singapore, 2017, pp. 1714--1721.
			
			\bibitem{nagabandi2018neural}
			A. Nagabandi, G. Kahn, R. S. Fearing, and S. Levine, ``Neural network dynamics for model-based deep reinforcement learning with model-free fine-tuning,'' in \emph{2018 IEEE International Conference on Robotics and Automation (ICRA)}, Brisbane, Australia, 2018, pp. 7559--7566.
			
			\bibitem{achiam2017constrained}
			J. Achiam, D. Held, A. Tamar, and P. Abbeel, ``Constrained policy optimization,'' in \emph{Proceedings of the 34th International Conference on Machine Learning (ICML)}, 2017, pp. 22--31.
			
			\bibitem{tessler2018reward}
			C. Tessler, D. J. Mankowitz, and S. Mannor, ``Reward constrained policy optimization,'' in \emph{International Conference on Learning Representations (ICLR)}, 2019.
			
			\bibitem{zhang2020first}
			Y. Zhang, Q. Vuong, and K. W. Ross, ``First order constrained optimization in policy space,'' in \emph{Advances in Neural Information Processing Systems (NeurIPS)}, 2020.
			
			\bibitem{yang2022constrained}
			L. Yang et al., ``Constrained update projection approach to safe policy optimization,'' in \emph{Advances in Neural Information Processing Systems (NeurIPS)}, 2022.
			
			\bibitem{zhang2022P3O}
			L. Zhang et al., ``Penalized proximal policy optimization for safe reinforcement learning,'' in \emph{International Joint Conference on Artificial Intelligence (IJCAI)}, 2022.
			
			\bibitem{altman1999constrained}
			E. Altman, \emph{Constrained Markov Decision Processes}. CRC Press, 1999.
			
			\bibitem{kasaura2023benchmarking}
			K. Kasaura, S. Miura, T. Kozuno, R. Yonetani, K. Hoshino, and Y. Hosoe, ``Benchmarking actor-critic deep reinforcement learning algorithms for robotics control with action constraints,'' \emph{IEEE Robotics and Automation Letters}, vol. 8, no. 8, pp. 4449-4456, Aug. 2023.
			
			\bibitem{brahmanage2023flowpg}
			J. C. Brahmanage, J. Ling, and A. Kumar, ``FlowPG: Action-constrained policy gradient with normalizing flows,'' in \emph{Advances in Neural Information Processing Systems (NeurIPS)}, 2023.
			
			\bibitem{hung2025aram}
			W.-T. Hung et al., ``Efficient action-constrained reinforcement learning via acceptance-rejection method and augmented MDPs,'' in \emph{International Conference on Learning Representations (ICLR)}, 2025.
			
			\bibitem{stolz2025truncated}
			P. Stolz et al., ``Improving Stochastic Action-Constrained Reinforcement Learning via Truncated Distributions,'' \emph{arXiv preprint arXiv:2511.22406}, 2025.
			
			\bibitem{johannink2019residual}
			T. Johannink et al., ``Residual reinforcement learning for robot control,'' in \emph{IEEE International Conference on Robotics and Automation (ICRA)}, 2019, pp. 6023--6029.
			
			\bibitem{lee2020learning}
			J. Lee et al., ``Learning quadrupedal locomotion over challenging terrain,'' \emph{Science Robotics}, vol. 5, no. 47, 2020.
			
			\bibitem{li2020safe}
			J. Li, L. Sun, J. Chen, M. Tomizuka, and W. Zhan, ``A safe hierarchical planning framework for complex driving scenarios based on reinforcement learning,'' in \emph{2021 IEEE International Conference on Robotics and Automation (ICRA)}, Xi'an, China, 2021, pp. 2660--2666.
			
			\bibitem{rajeswaran2018learning}
			A. Rajeswaran et al., ``Learning complex dexterous manipulation with deep reinforcement learning and demonstrations,'' \emph{arXiv preprint arXiv:1709.10087}, 2017.
			
			\bibitem{gros2020data}
			S. Gros and M. Zanon, ``Data-driven economic NMPC using reinforcement learning,'' \emph{IEEE Transactions on Automatic Control}, vol. 65, no. 2, pp. 636--648, Feb. 2020.
			
			\bibitem{vecerik2017leveraging}
			M. Vecerík et al., ``Leveraging demonstrations for deep reinforcement learning on robotics problems with sparse rewards,'' \emph{arXiv preprint arXiv:1707.08817}, 2017.
			
			\bibitem{chebotar2019closing}
			Y. Chebotar et al., ``Closing the sim-to-real loop: Adapting simulation randomization with real world experience,'' in \emph{IEEE International Conference on Robotics and Automation (ICRA)}, 2019, pp. 8973--8979.
			
			\bibitem{ray2019benchmarking}
			A. Ray, J. Achiam, and D. Amodei, ``Benchmarking safe exploration in deep reinforcement learning,'' 2019. [Online]. Available: \url{https://cdn.openai.com/safexp-short.pdf}
			
			\bibitem{wachi2024survey}
			A. Wachi, X. Shen, and Y. Sui, ``A survey of constraint formulations in safe reinforcement learning,'' in \emph{International Joint Conference on Artificial Intelligence (IJCAI)}, 2024.
			
			\bibitem{gu2024review}
			S. Gu et al., ``A review of safe reinforcement learning: Methods, theories, and applications,'' \emph{IEEE Transactions on Pattern Analysis and Machine Intelligence}, vol. 46, no. 12, pp. 11216--11235, Dec. 2024.
			
			\bibitem{bertsekas1996stochastic}
			D. P. Bertsekas and S. E. Shreve, \emph{Stochastic Optimal Control: The Discrete Time Case}. Athena Scientific, 1996.
			
			\bibitem{lillicrap2016continuous}
			T. P. Lillicrap et al., ``Continuous control with deep reinforcement learning,'' in \emph{Proceedings of the Fourth International Conference on Learning Representations (ICLR)}, 2016.
			
			\bibitem{haarnoja2018soft}
			T. Haarnoja, A. Zhou, P. Abbeel, and S. Levine, ``Soft actor-critic: Off-policy maximum entropy deep reinforcement learning with a stochastic actor,'' in \emph{International Conference on Machine Learning (ICML)}, 2018, pp. 1861--1870.
			
			\bibitem{schulman2017proximal}
			J. Schulman, F. Wolski, P. Dhariwal, A. Radford, and O. Klimov, ``Proximal policy optimization algorithms,'' \emph{arXiv preprint arXiv:1707.06347}, 2017.
			
			\bibitem{fujimoto2018td3}
			S. Fujimoto, H. van Hoof, and D. Meger, ``Addressing function approximation error in actor-critic methods,'' in \emph{International Conference on Machine Learning (ICML)}, 2018, pp. 1587--1596.
			
			\bibitem{ames2017control}
			A. D. Ames, X. Xu, J. W. Grizzle, and P. Tabuada, ``Control barrier function based quadratic programs for safety critical systems,'' \emph{IEEE Transactions on Automatic Control}, vol. 62, no. 8, pp. 3861--3876, Aug. 2017.
			
			\bibitem{pham2018optlayer}
			T.-H. Pham, G. De Magistris, and R. Tachibana, ``OptLayer -- Practical constrained optimization for deep reinforcement learning in the real world,'' in \emph{2018 IEEE International Conference on Robotics and Automation (ICRA)}, Brisbane, Australia, 2018, pp. 6236--6243.
			
			\bibitem{dalal2018safe}
			G. Dalal, K. Dvijotham, M. Vecerik, T. Hester, C. Paduraru, and Y. Tassa, ``Safe exploration in continuous action spaces,'' \emph{arXiv preprint arXiv:1801.08757}, 2018.
			
			\bibitem{singletary2021safety}
			A. Singletary, S. Kolathaya, and A. D. Ames, ``Safety-critical kinematic control of robotic systems,'' \emph{IEEE Control Systems Letters}, vol. 6, pp. 139--144, 2022.
			
			\bibitem{cvflows2025}
			J. C. Brahmanage, J. Ling, and A. Kumar, ``Leveraging constraint violation signals for action-constrained reinforcement learning,'' in \emph{Proceedings of the AAAI Conference on Artificial Intelligence}, 2025.
			
			\bibitem{seo2025fastsac}
			Y. Seo, C. Sferrazza, J. Chen, G. Shi, R. Duan, and P. Abbeel, ``Learning sim-to-real humanoid locomotion in 15 minutes,'' \emph{arXiv preprint arXiv:2512.01996}, 2025.
			
		\end{thebibliography}
	}
\end{document}